\documentclass{article}
\usepackage[utf8]{inputenc}
\usepackage{algorithm}
\usepackage{algorithmic}
\usepackage{amsmath}
\usepackage{amsthm}
\usepackage{bbm}
\usepackage[british,UKenglish,USenglish]{babel}
\usepackage{graphicx}
\usepackage{caption}
\usepackage{subcaption}
\usepackage{tabu}
\usepackage{geometry}
\usepackage{amssymb}
\usepackage{eqparbox}
\usepackage{authblk}

\theoremstyle{plain}
\newtheorem{thm}{Theorem}[section]
\theoremstyle{definition}
\newtheorem{defn}[thm]{Definition}

\newtheorem{prop}{Proposition}
\newtheorem{lemma}{Lemma}

\newcommand{\x}{\mathbf x}
\newcommand{\y}{\mathbf y}
\newcommand{\R}{\mathbb R}
\newcommand{\RR}{\mathbf R}
\newcommand{\I}{\mathbf I}
\newcommand{\N}{\mathbf N}
\newcommand{\Y}{\mathbf Y}
\newcommand{\PY}{\mathcal P_{\mathcal Y}}
\renewcommand{\L}{\mathcal L}

\renewcommand{\Pr}{\mathbb P}
\newcommand{\E}{\mathbb{E}}

\newcommand{\PX}{\mathcal{P}_\mathcal{X}}

\title{Budgeted Multi-Objective Optimization with a Focus on the Central Part of the Pareto Front - Extended Version}
\author[1,2]{David Gaudrie}
\author[2]{Rodolphe le Riche}
\author[3]{Victor Picheny}
\author[1]{Benoît Enaux}
\author[1]{Vincent Herbert}
\affil[1]{Groupe PSA}
\affil[2]{CNRS LIMOS, École Nationale Supérieure des Mines de Saint-Étienne}
\affil[3]{Prowler.io}

\begin{document}
	
	\maketitle
	
	\begin{abstract}
Optimizing nonlinear systems involving expensive computer experiments with regard to conflicting objectives is a common challenge. 
When the number of experiments is severely restricted and/or when the number of objectives increases, uncovering the whole set of Pareto optimal solutions is out of reach, even for surrogate-based approaches: the proposed solutions are sub-optimal or do not cover the front well. 
As non-compromising optimal solutions have usually little point in applications, this work restricts the search to solutions that are close to the Pareto front center. The article starts by characterizing this center, which is defined for any type of front.
Next, a Bayesian multi-objective optimization method for directing the search towards it is proposed. 
Targeting a subset of the Pareto front allows an improved optimality of the solutions and a better coverage of this zone, which is our main concern. A criterion for detecting convergence to the center is described. If the criterion is triggered, a widened central part of the Pareto front is targeted such that sufficiently accurate convergence to it is forecasted within the remaining budget. Numerical experiments show how the resulting algorithm, C-EHI, better locates the central part of the Pareto front when compared to state-of-the-art Bayesian algorithms.
	\end{abstract}
	
	
	\section{Introduction}
	Over the last decades, computer codes have been widely employed for optimal design.
	Practitioners measure the worth of a design with several criteria, which corresponds to a \emph{multi-objective optimization} problem,
	\begin{equation}\underset{\x \in X}{\min}(f_1(\x) ,\dotsc, f_m(\x))\label{eq:moo}\end{equation}
	where $X\subset\R^d$ is the parameter space, and $f_j(\cdot),~ j=1,\dotsc,m$ are the $m$ objective functions. Since these goals are generally competing, there does not exist a single solution $\x^*$ minimizing every function in (\ref{eq:moo}), but several trade-off solutions that are mutually \emph{non-dominated} (ND). These solutions (or designs) form the Pareto set $\PX$, whose image corresponds to the Pareto front $\PY$. 
Elements and methods of classical multi-objective optimization can be found in \cite{sawaragi1985theory,miettinen_book,marler2004survey}.
	
	Often, the $f_j$'s are outputs of a computationally expensive computer code (several hours to days for one evaluation), so that only a small number of experiments can be carried out. Under this restriction, Bayesian optimization methods \cite{ExpectedImprovement,EGO} have proven their effectiveness in single objective problems. These techniques use a surrogate – generally a Gaussian Process (GP) \cite{GPML,stein2012interpolation} – of the true function to locate the optimum. Extensions of Bayesian optimization to multi-objective cases have also been proposed, see \cite{TheseBinois,Parego,Keane,EHI,EMI,SMS,SUR}.
	In the case of very narrow budgets (about a hundred evaluations), obtaining an accurate approximation of the Pareto front remains out of reach, even for Bayesian approaches. This issue gets worse with increasing number of criteria. The article provides illustrations of this phenomenon in Section \ref{sec:results}.
	Looking for the entire front can anyway seem useless as the Pareto set will contain many irrelevant solutions from an end-user’s point of view.
	
	In this paper, instead of trying to approximate the entire front, we search for a well-chosen part of it. Without any specific information about the preferences of the decision maker, we assume that well-balanced solutions are the most interesting ones. By specifically targeting them, we argue that convergence should be enhanced there.
	
	Restricting the search to parts of the objective space is a common practice in multi-objective optimization. Preference-based methods incorporate user-supplied information to guide the search \cite{fonseca1998multiobjective,book_mcdm,triantaphyllou2000multi,branke2008multiobjective,Pref1}. The preference can be expressed either as an aggregation of the objectives (e.g., \cite{bowman1976relationship,miettinen_book,zhang2007moea,marler2010weighted}) or as an aspiration level (also known as reference point) to be attained or improved upon (e.g., \cite{wierzbicki1980use,book_mcdm_refpoint,RNSGA}).
More recently, preferences have also been included in Bayesian multi-objective optimization. A more detailed review of related works is given in Section \ref{sec:ehi_refpoint}.

Contrarily to existing multi-objective optimization techniques which guide the search using externally supplied information, in the current article the preference region is defined through the Pareto front center and is automatically determined by processing the GPs. 
This is the first contribution of this work.
The other contributions include the definition of a criterion for targeting specific parts of the Pareto front and the management of this preference region according to the remaining computational budget.
		
	An overview of the proposed method, which we name the C-EHI algorithm (for Centered Expected Hypervolume Improvement), is sketched in Figure \ref{fig:resume_algo}. It uses the concept of Pareto front center that is defined in  Section \ref{sec:center}. 
	C-EHI iterations are made of three steps. First, an estimation of the Pareto front center is carried out, as described in Section \ref{sec:center} and sketched in Figure \ref{fig:resume_algo}a. Second, the estimated center allows to target well-balanced parts of the Pareto front by a modification of the EHI criterion (cf. Section \ref{sec:target}). Figure \ref{fig:resume_algo}b illustrates the idea. 
	Third, to avoid wasting computations once the center is sufficiently well located, the part of the Pareto front that is searched for is broadened in accordance with the remaining budget. To this aim, a criterion to test convergence to the center is introduced in Section \ref{sec:conv}. 
	When triggered (see Figure \ref{fig:resume_algo}c), a new type of iteration starts until the budget is exhausted (see Figure \ref{fig:resume_algo}d). 
	Section \ref{sec:seconde_phase} explains how the new goals are determined.
	
	\begin{figure}[h!]
		\centering
		\includegraphics[width=0.5\textwidth]{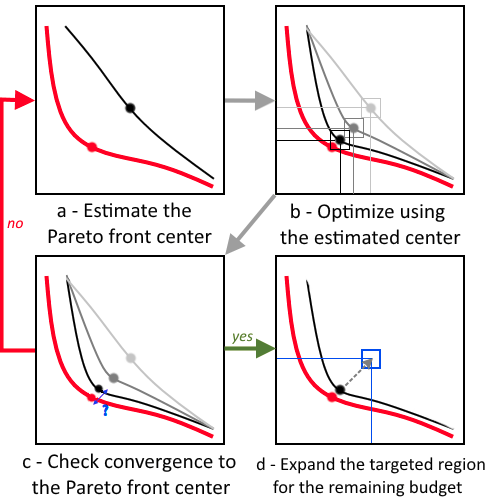}
		\caption{Sketch of the proposed C-EHI algorithm for targeting equilibrated solutions. The Pareto front center properties (a) are discussed in Section~\ref{sec:center}; How to guide the optimization with the center (b) is the topic of Section~\ref{sec:target}; Section~\ref{sec:conv} details how convergence to the Pareto front center is tested (c); How to widen the search within the remaining budget (d), is presented in Section~\ref{sec:seconde_phase}.
}
		\label{fig:resume_algo}
	\end{figure} 
	
The methodology is first tested the popular ZDT1 \cite{zdt2000a} and P1 \cite{TheseParr} functions and then on a benchmark built from real-world airfoil aerodynamic data. The airfoil benchmark has variables in dimension $d=3$, 8 and 22 that represent CAD parameters, and 2 to 4 aerodynamic objectives (lift and drag at different airfoil angles). The results are presented in Section \ref{sec:results}. 
The default test case that illustrates the algorithm concepts before numerical testing (Figures~\ref{fig:center_estimation} to \ref{fig:apres_second_phase}) is the airfoil problem with 2 objectives and 8 variables.
	
	\section{A brief review of Bayesian multi-objective optimization}
	\label{sec:review_bayesian_optim}
	
	\subsection{Bayesian optimization}
	\label{sec:bayesian_optimization}	
	Bayesian optimization techniques \cite{ExpectedImprovement} have become popular to tackle single-objective optimization problems within a limited number of iterations. 
	These methods make use of Bayes rule: a prior distribution, usually a GP, is placed over $f$ and is enhanced by observations to derive the posterior distribution. 
	Denoting $Z(\cdot)$ the GP and  $\mathcal A_n=\{(\x^1,y_1),\dotsc,(\x^n,y_n)\}=\{\mathbb X,\mathbb Y\}$ the observational event, the posterior GP conditioned on $\mathcal A_n$ has a known Gaussian distribution: \[\forall\x\in X~,~ Y(\x):=[Z(\x)\vert\mathcal A_n]\sim\mathcal N(\widehat{y}(\x),s^2(\x))\]
	where 
	\[\widehat{y}(\x)=\widehat{\mu}+k(\x,\mathbb X)K^{-1}(\mathbb Y-\mathbf1\widehat{\mu})\] is the conditional mean function (a.k.a., the kriging mean predictor) \cite{sacks1989design,GPML,stein2012interpolation,santner2013design,cressie2015statistics} 
	and \[s^2(\x)=c(\x,\x)\] is the conditional variance, obtained from the conditional covariance function 
	\[c(\x,\x')=\widehat{\sigma^2}\left(1-k(\x,\mathbb X)K^{-1}k(\mathbb X,\x')+\frac{(1-\mathbf1^\top K^{-1}k(\x,\mathbb X))(1-\mathbf1^\top K^{-1}k(\x',\mathbb X))}{\mathbf1^\top K^{-1}\mathbf1}\right)~.\]
	$K_{ij}=k(\x_i,\x_j)$ is the covariance matrix with $k(\cdot,\cdot)$ the covariance function (or kernel). $\widehat{\mu}$, $\widehat{\sigma^2}$ are the estimated mean and variance of the GP. 
	$k(\cdot,\cdot)$ is typically chosen from a parametric family and its parameters are estimated along with $\widehat{\mu}$ and $\widehat{\sigma^2}$ by likelihood maximization.
	Further discussion about these parameters and their estimation can be found e.g. in \cite{GPML,roustant2012dicekriging}.
	
	Given a set of inputs $\x^{n+1},\dotsc,\x^{n+s}\in X$, the posterior distribution of $Z(\cdot)$ at these points is a Gaussian vector
	\[\begin{pmatrix}Y(\x^{n+1})\\\cdots\\Y(\x^{n+s})\end{pmatrix} ~\sim~ \mathcal N\left(\begin{pmatrix}\widehat{y}(\x^{n+1})\\\cdots\\\widehat{y}(\x^{n+s})\end{pmatrix},\Gamma\right),\] 
	with $\Gamma_{a,b}=c(\x^{n+a},\x^{n+b})$. 
	It is possible to simulate plausible values of $f(\cdot)$ by sampling $n_{sim}$ GPs $\widetilde{Y}^{(k)}(\cdot), k=1,\dotsc,n_{sim}$ at $\x^{n+1},\dotsc,\x^{n+s}\in X$.
	GP simulations require the Cholesky decomposition of the $s\times s$ matrix $\Gamma$ and are therefore only tractable for moderate sample sizes.
	
	For optimization purposes, new data points $(\x^{n+1},f(\x^{n+1}))$ are sequentially determined through the maximization of an \emph{acquisition function} (or infill criterion)
	until a limiting number of function evaluations, the $budget$, is attained. 
	Acquisition functions use the posterior distribution $Y(\cdot)$. 
	A commonly used infill criterion is the Expected Improvement (EI) \cite{ExpectedImprovement,jones2001}, 
	which balances minimization of the GP mean (``exploitation'' of past information) and maximization of the GP variance (``exploration'' of new regions of the design space) in order to both search for the minimum of $f(\cdot)$ and improve the GP accuracy.
	The Expected Improvement below a threshold $T$ is defined as 
	\begin{equation}
	\text{EI}(\x;T):=\E[(T-Z(\x))_+\vert\mathcal A_n]
	\label{eq:EI0}
	\end{equation}
	which is computable	in closed-form: 
	\begin{equation}
	\text{EI}(\x;T)=(T-\widehat{y}(\x))\phi\left(\frac{T-\widehat{y}(\x)}{s(\x)}\right)+s(\x)\varphi\left(\frac{T-\widehat{y}(\x)}{s(\x)}\right)
	\label{eq:EI}
	\end{equation}
	$\varphi$ and $\phi$ correspond to the probability density function and to the cumulative distribution function of a standard normal random variable, respectively. 
	$T$ is generally set as the best value observed so far, $f_{min}:=\min(y_1,\dotsc,y_n)$.
	EGO (Efficient Global Optimization, \cite{EGO}) iteratively evaluates the function to optimize at the EI maximizer (Figure \ref{fig:bayesian_optimization}) before updating the GP. During the update step, 
	the covariance parameters are re-estimated and the additional evaluation taken into account, which modifies the conditional mean and covariance.
	At the end of the procedure, the best observed design and its performance, $\x^*:=\underset{i=1,\dotsc,budget}{\arg\min}f(\x^i)$ $,y^*:=f(\x^*)$, are returned. 
	
	\begin{figure}[h!]
		\centering
		\includegraphics[width=0.5\textwidth]{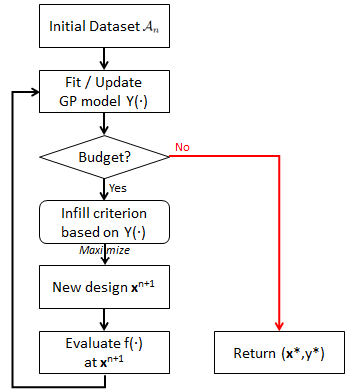}
		\caption{Outline of a Bayesian optimization algorithm}
		\label{fig:bayesian_optimization}
	\end{figure}
	
	\subsection{Extension to the multi-objective case}
	\label{sec:multiobj}
	In multi-objective optimization there is a (possibly infinite) set of solutions to (\ref{eq:moo}) called the Pareto set $\PX$. Designs in $\PX$ correspond to an optimal compromise in the sense that it is not possible to find a competitor being better in all objectives simultaneously.
	
	Mathematically, $\PX=\{\x\in X:\nexists \x'\in X, \mathbf f(\x')\preceq\mathbf f(\x)\}$ where $\preceq$ stands for weak or strong Pareto domination in $\R^m$ as $\mathbf f(\x):=(f_1(\x),\dotsc,f_m(\x))^\top$ is no longer a scalar but an $m$-dimensional objective vector. The Pareto front $\PY$ is the image of the Pareto set and contains only non-dominated solutions: $\PY=\mathbf f(\PX)=\{\y\in Y:\nexists \y'\in Y, \y'\preceq\y\}$, with $Y=\mathbf f(X)\subset\R^m$ the image of the design space through the objectives.
	Multi-objective optimizers aim at finding an approximation front built upon past observations $\widehat\PY=\{\y\in\mathbb Y:\nexists\y'\in\mathbb Y,\y'\preceq\y\}$ to $\PY$, with some properties such as convergence or diversity. Evolutionary Multi-Objective Optimization Algorithms (EMOA) have proven their benefits for solving Multi-Objective Problems \cite{LivreDeb}. They are however, in the absence of a model to the objective functions, not adapted to expensive objectives (this will be observed in Section \ref{sec-test_results}).
	
	Multi-objective extensions to EGO do exist. These Bayesian approaches generally model the objective functions $f_j(\cdot)$ as independent GP's $Y_j(\cdot)$. Svensson \cite{TheseSvensson} has considered the GP's to be (negatively) correlated in a bi-objective case, without noticing significant benefits. The GP framework enables both the prediction of the objective functions, $\widehat{y_j}(\x)$, and the quantification of the uncertainties, $s_j^2(\x), \forall\x\in X$. 
As in the single-objective case, an acquisition function is used for determining $\x^{t+1}\in X$, the most promising next iterate to be evaluated. In some approaches, the $m$ surrogates are aggregated or use an aggregated form of EI \cite{Parego,EIEMO,WSEI,TAEI}. Other methods use a multi-objective infill criterion for taking into account all the metamodels simultaneously \cite{ICMOO}. 
The Expected Hypervolume Improvement (EHI) \cite{EHI,EHI2,EHI3}, the EMI \cite{TheseSvensson,EMI}, and Keane's Euclidean-based Improvement \cite{Keane} are three multi-objective infill criteria that reduce to EI when facing a single objective. SMS \cite{SMS} is based on a lower confidence bound strategy, and SUR \cite{SUR} considers the stepwise uncertainty reduction on the Pareto front. These infill criteria aim at providing new non-dominated points and eventually approximating the Pareto front in its entirety.
All these Bayesian multi-objective methods conform to the outline of Figure \ref{fig:bayesian_optimization}, excepted that $m$ surrogates $Y_1(\cdot),\dotsc,Y_m(\cdot)$ and $m$ objective functions are now considered, and that an empirical Pareto set $\widehat{\PX}$ and Pareto front $\widehat{\PY}$ are returned \cite{liang2017pareto,zhang2019multi}.
	
	\subsubsection*{EHI: a multi-objective optimization infill criterion}
	The EHI (Expected Hypervolume Improvement) \cite{EHI,EHI2,EHI3} is one of the most competitive \cite{yang2015expected} multi-objective infill criterion. It rewards the expected growth of the hypervolume indicator \cite{TheseZitzler}, corresponding to the hypervolume dominated by the approximation front up to a reference point $\RR$ (see Fig.~\ref{fig:hypervolume}), when adding a new observation $\x$. More precisely, the hypervolume indicator of a set $\mathcal A$ is
	
	\[H(\mathcal A;\RR)=\bigcup_{\y\in\mathcal A}\int_{\y\preceq\mathbf z\preceq\RR}d\mathbf z=\Lambda\left(\bigcup_{\mathbf y\in\mathcal A}\{\mathbf z : \mathbf y\preceq\mathbf z\preceq\mathbf R\}\right)\]
	where $\Lambda$ is the Lebesgue measure on $\R^m$. The hypervolume improvement induced by $\y\in\R^m$ is $I_H(\y;\RR)=H(\mathcal A\cup\{\y\};\RR)-H(\mathcal A;\RR)$. In particular, if $\mathcal A\preceq \y$ (in the sense that $\exists\mathbf a\in\mathcal A:\mathbf a\preceq\y$), or if $\y\npreceq\RR$, $I_H(\y;\RR)=0$. For a design $\x$, EHI$(\x;\RR)$ is
	\begin{equation}
	\text{EHI}(\x;\RR) := \E[I_H(\Y(\x);\RR)]
	\label{eq:EHI}
	\end{equation}
	\begin{figure}[!ht]
		\centering
		\includegraphics[width=0.4\textwidth]{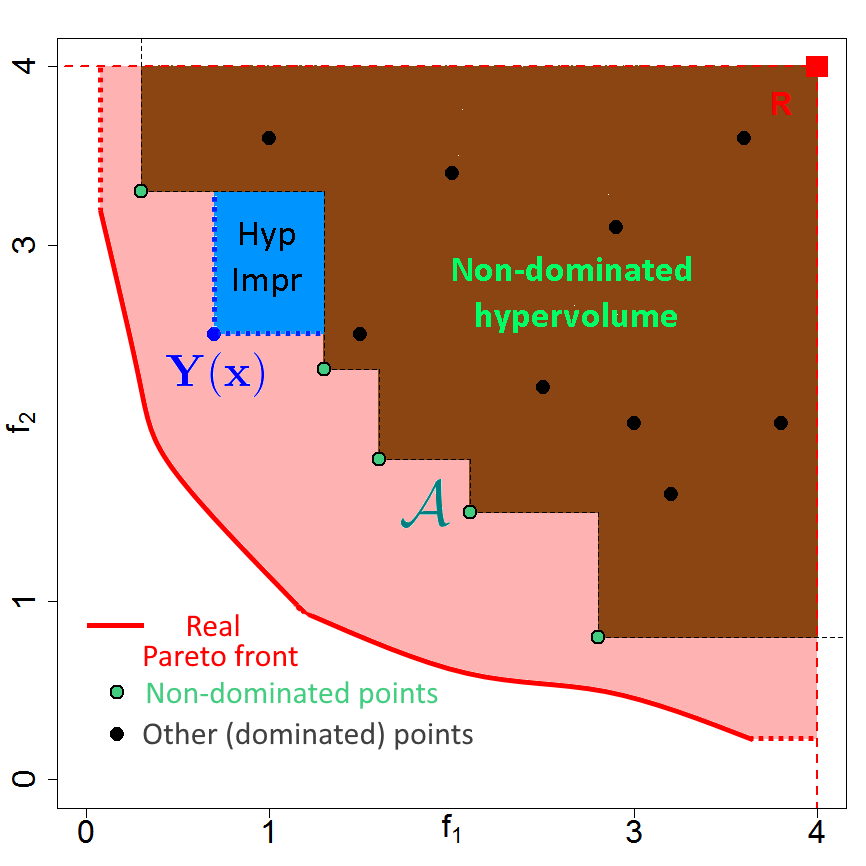}
		\caption{The hypervolume indicator of the non-dominated set (green points) corresponds to the area dominated by it, up to $\RR$ (in brown). The blue rectangle is the hypervolume improvement brought by $\mathbf{Y}(\mathbf{x})$, $I_H(\Y(\x);\RR)$.
		}
		\label{fig:hypervolume}
	\end{figure}
	
	The hypervolume indicator being a refinement of the Pareto dominance \cite{EHI,ICMOO} ($\mathcal{A}\preceq\mathcal{B}\Rightarrow H(\mathcal{A};\RR)>H(\mathcal{B};\RR)$ for two non-dominated sets $\mathcal{A}$ and $\mathcal{B}$, and any reference point $\RR$), and as the hypervolume Improvement induced by a dominated solution equals zero, EHI maximization intrinsically leads to Pareto optimality. It also favors well-spread solutions, as the hypervolume increase is small when adding a new value close to an already observed one in the objective space \cite{auger2009theory,auger2012hypervolume}.
	
	Several drawbacks should be mentioned. 
	First, EHI requires the computation of $m$-dimensional hypervolumes. Even though the development of efficient algorithms for computing the criterion to temper its computational burden is an active field of research \cite{beume2009complexity,while2012fast,chan2013klee,couckuyt2014fast,russo2014quick,lacour2017box,jaszkiewicz2018improved} with two \cite{EHI,emmerich2016multicriteria} and three objectives \cite{yang2017computing}, the complexity grows exponentially with the number of objectives and non-dominated points. When $m>3$, expensive Monte-Carlo estimations are required to compute the EHI. An analytic expression of its gradient has been discovered recently in the bi-objective case \cite{yang2019multi}.
Second, the hypervolume indicator is less relevant for many-objective optimization, as the amount of non-dominated solutions rises with $m$, and more and more solutions contribute to the growth of the non-dominated hypervolume; in a many-objective settings, this metric is less able to distinguish truly relevant from non-informative solutions. 
Last, the choice of the reference point $\RR$ is unclear and influences the optimization results, as will be discussed in Section \ref{sec:target}.
	
	\subsection{Past work on targeted multi-objective optimization}
	\label{sec:ehi_refpoint}
Targeting special parts of the objective space has been largely discussed within the multi-objective optimization literature, see for example \cite{Pref1} or \cite{Pref2} for a review. 
The benefits of targeting a part of the Pareto front instead of trying to unveil it entirely go beyond reflecting the user's preferences: as will be shown by the experiments of Section \ref{sec:results}, it allows an enhanced distribution of the proposed solutions within this area.
Preference-based optimization makes use of user-supplied information to guide the search towards specific parts of the Pareto front. The preference is typically expressed as desired objective values (i.e., reference or aspiration points, cf. \cite{book_mcdm_refpoint,wierzbicki1980use}) the distance to which is measured by a specific metric (e.g., $L_1$, $L_2$ or $L_\infty$ norms). Preference can appear as a ranking of solutions or of objectives via an aggregation function \cite{miettinen_book,bowman1976relationship}

Bayesian multi-objective optimization (see Section~\ref{sec:multiobj}) most often relies on the EHI infill criterion where the hypervolume is computed up to a reference point $\RR$.
$\RR$ has been originally seen as a second order hyperparameter with default values chosen so that all Pareto optimal points are valued in the EHI. 
For example, several studies (e.g. \cite{SMS}) suggest taking $\N+\mathbf 1$ ($\N$ being the Nadir point of the empirical Pareto front).
	
	Later, the effect of $\RR$ has received some attention. 
	Auger et al. \cite{auger2009theory,auger2012hypervolume} have theoretically and experimentally investigated the $\mu$-optimal distribution on the Pareto front induced by the choice of $\RR$. 
	Ishibuchi et al. \cite{REMOA} have noticed a variability in the solutions given by an EMO algorithm when $\RR$ changes. 
	Feliot \cite{TheseFeliot} has also observed that $\RR$ impacts the approximation front and recommends $\RR$ to be neither too far away nor too close to $\PY$. 
	By calculating EHI restricted to areas dominated by ``goal points'', Parr \cite{TheseParr} implicitly acted on $\RR$ and noticed fast convergence when the goal points were taken on $\widehat{\PY}$. In \cite{li2018modified}, a modification of the hypervolume improvement is proposed. It is a sum of EHI's over different non-dominated $\RR$'s which eases the computations when compared to EHI in a fashion similar to the Section \ref{sec:mEI}.
	
	Previous works in Bayesian Multi-Objective Optimization have also targeted particular areas of the objective space thanks to ad-hoc infill criteria. The Weighted Expected Hypervolume Improvement (WEHI) \cite{WHI,WHI2,WHI3,WHI4} is a variant of EHI that emphasizes given parts of the objective space through a user-defined weighting function. 
	In \cite{TEHI,TEHI2}, the Truncated EHI  criterion is studied where the Gaussian distribution is restricted to a user-supplied hyperbox in which new solutions are sought.
	
In the absence of explicitly provided user preferences, the algorithm proposed here targets a specific part of the Pareto front, its center, through a choice of $\RR$ that is no longer arbitrarily chosen. 
The center of the Pareto front is defined in the following section. 
Since it balances the objectives, the center is considered as a default preference.
	
\section{Center of the Pareto front: definition and estimation}
\label{sec:center}
There has been attempts to characterize parts of the Pareto front where objectives are ``visually'' equilibrated. In \cite{book_mcdm_refpoint}, the neutral solution is defined as the closest point in the objective space to the Ideal point in a (possibly weighted) $\mathcal L^p$ norm and is located ``somewhere in the middle'' of the Pareto front. The point of the Pareto front which minimizes the distance to the Ideal point is indeed a commonly preferred solution \cite{zeleny1976theory}. In \cite{buchanan03}, not only the closest to the Ideal point, but also the farthest solution to the Nadir point (see definitions hereafter) are brought out, in terms of a weighted Tchebycheff norm. Note that the weights depend on user-supplied aspiration points.
Other appealing points of the Pareto front are knee points as defined in \cite{branke2004finding}. They correspond to parts of the Pareto front where a small improvement in one objective goes with a large deterioration in at least one other objective, which makes such points stand out as kinks in the Pareto front. When the user's preferences are not known, the authors claim that knee points should be emphasized and propose methods for guiding the search towards them.

Continuing the same effort, we propose a definition of the Pareto front center that depends only on the geometry of the Pareto front.

\subsection{Definitions}

Before defining the center of a Pareto front, other concepts of multi-objective optimization have to be outlined.
	
	\begin{defn}
		The Ideal point $\I$ of a Pareto front $\PY$ is its component-wise minimum, $\I=(\underset{\y\in\PY}{\min}y_1,\dotsc,\underset{\y\in\PY}{\min}y_m)$.
	\end{defn}
		
	The Ideal point also corresponds to the vector composed of each objective function minimum.
	Obviously, there exists no $\y$ better in all objectives than the minimizer of objective $j$. As a consequence, the latter belongs to $\PY$ and
	$\underset{\y\in\PY}{\min}y_j=\underset{\y\in Y}{\min}~y_j$, $j=1,\dotsc,m$. $\I$ can therefore be alternatively defined as $(\underset{\x\in X}{\min}~f_1(\x),\dotsc,\underset{\x\in X}{\min}~f_m(\x))$. 
	The decomposition on each objective
	does not hold for the Nadir point, which depends on the Pareto front structure:	
	
	\begin{defn}\label{defnadir} The Nadir point $\N$ of a Pareto front $\PY$ is the component-wise maximum of the Pareto front, $\N=(\underset{\y\in\PY}{\max}y_1,\dotsc,\underset{\y\in\PY}{\max}y_m)$.
	\end{defn}
	
	$\I$ and $\N$ are virtual points, that is to say that there generally does not exist an $\x\in X$ such that $\mathbf f(\x)=\I$ or $\N$. They are bounding points for the Pareto front, as every $\y\in\PY$ will be contained in the hyperbox defined by these points.	
	
	\begin{defn}\label{defextremepoints}An extreme point for the $j$-th objective, $\pmb\nu^j$, is an $m$-dimensional vector that belongs to the Pareto front, $\pmb\nu^j\in \PY$, and such that $\nu^{j}_j=N_j$.
	The Nadir point can thus be rewritten as $\N=(\nu^1_1,\dotsc,\nu^m_m)$. A $j$-th extreme design point is $\pmb\xi^{j}\in X$ such that $\mathbf f(\pmb\xi^{j})=\pmb\nu^j$.
	\end{defn}
		
	In the following, extreme points of the approximation front $\widehat\PY$ are denoted by $\widehat{\pmb\nu}^{j}$, hence the Nadir of that empirical front is  $\pmb{\widehat{\nu}}=(\widehat{\nu}_1^1,\dotsc,\widehat{\nu}_m^m)$.
	Note that we will also introduce in Section~\ref{sec:estimation} the notation $\widehat{\N}$ to denote an estimator of the Nadir point. We can now define the center of a Pareto front:
	
	\begin{defn}The center of a Pareto front $\mathbf C$ is the closest point in Euclidean distance to $\PY$ on the Ideal-Nadir line $\L$.
	\end{defn}
	
	In the field of Game Theory, our definition of the center of a Pareto front corresponds to a particular case of the Kalai-Smorodinsky equilibrium \cite{KS,binois2019kalai}, taking the Nadir as disagreement point $\mathbf d \equiv \N$. 
	This equilibrium aims at equalizing the ratios of maximal gains of the players, which is the appealing property for the center of a Pareto front as an implicitly preferred point. Recently, it has been used for solving many-objective problems in a Bayesian setting \cite{binois2019kalai}.
In general, $\mathbf C$ is different from the neutral solution \cite{book_mcdm_refpoint} and from knee points \cite{branke2004finding}. They coincide in particular cases, e.g. a symmetric and convex front with scaled objectives and a non-weighted norm.
	
	In the case where the Pareto front is an $m$-dimensional continuous hypersurface, $\mathbf C$ corresponds to the intersection between $\PY$ and $\L$. 
	In a more general case, e.g. if the Pareto front is not continuous, or contains some lower dimensional hypersurfaces, $\mathbf C$ is the projection of the closest point belonging to $\PY$ on $\L$. 
	The computation of this point remains cheap even for a large $m$ in comparison with alternative definitions involving e.g. the computation of a barycenter in high-dimensional spaces. Some examples for two-dimensional fronts are shown in Figure~\ref{fig:exemples_centre}. The center of the Pareto front has also some nice properties that are detailed in following section. The center exists even if $\PY$ is discontinuous (top right front) or convoluted.
	
	\begin{figure}[!ht]
		\centering
		\begin{minipage}[c]{.46\linewidth}
			\includegraphics[width=\textwidth]{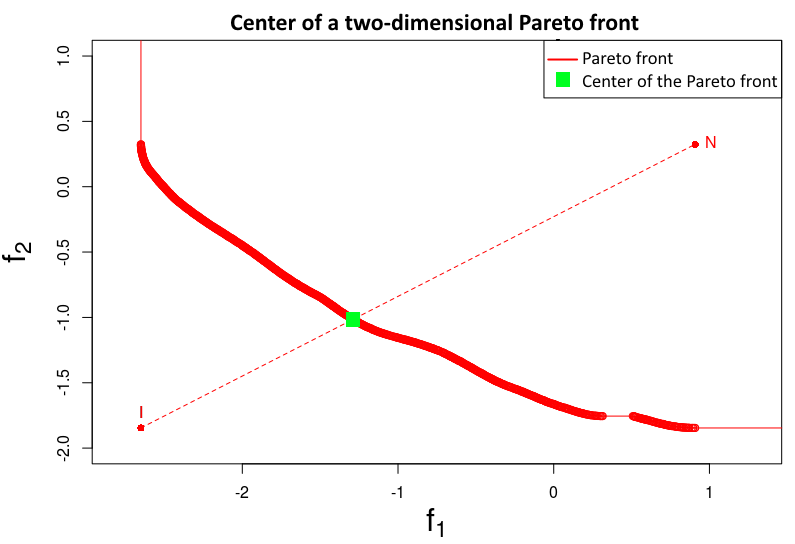}
		\end{minipage}
		\begin{minipage}[c]{.46\linewidth}
			\includegraphics[width=\textwidth]{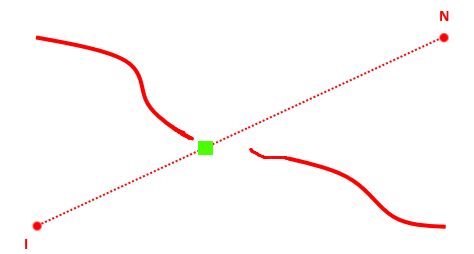}
		\end{minipage}
		\begin{minipage}[c]{.46\linewidth}
			\includegraphics[width=\textwidth]{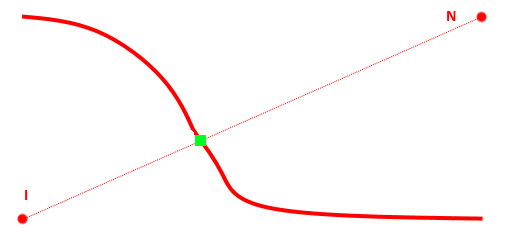}
		\end{minipage}
		\begin{minipage}[c]{.46\linewidth}
			\centering
			\includegraphics[width=0.6\textwidth]{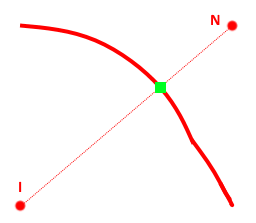}
		\end{minipage}
		\caption{Examples of two-dimensional Pareto fronts and their center. Notice that on the bottom, the left and the right fronts are the same, except that the left is substantially extended in the $x$ direction. However, the center has been only slightly modified.}
		\label{fig:exemples_centre}
	\end{figure}
	
	\subsection{Properties}
	
	\subsubsection*{Invariance to a linear scaling of the objectives}
	The Kalai-Smordinsky solution has been proved to verify a couple of properties, such as invariance to linear scaling\footnote{
		in Game Theory, given a feasible agreement set $F\subset\R^m$ ($Y$ in our context) and a disagreement point $\mathbf d\in\R^m$ ($\mathbf N$ here), a KS solution $f\in F$ (the center $\mathbf C$) satisfies the four following requirements: 
		Pareto optimality, symmetry with respect to the objectives, invariance to affine transformations (proven in Proposition~\ref{prop:intersection}) and, contrarily to a Nash solution, 
		monotonicity with respect to the number of possible agreements in $F$.} \cite{KS},
	which hold in our case.
	We extend here the linear invariance to the case where there is no intersection between $\PY$ and $\L$.
	
	\begin{prop}[Center invariance to linear scaling, intersection case]
		When $\PY$ intersects $\L$, the intersection is unique and is the center of the Pareto front. Furthermore, in that case, the center is invariant after a linear scaling $S:\R^m\rightarrow\R^m$ of the objectives: $S(\mathbf C(\PY))=\mathbf C(S(\PY))$.
		\label{prop:intersection}
	\end{prop}
	
	\begin{proof}
		First, it is clear that if $\PY$ intersects $\L$, the intersection is unique. Indeed, as in non degenerated cases $\I\prec\N$, $t\I+(1-t)\N\prec t'\I+(1-t')\N\Leftrightarrow t>t'$. Two points on $\mathcal L$ are different as long as $t\ne t'$. $\PY$ being only composed of non-dominated points it is impossible to find two different points $t\I+(1-t)\N$ and $t'\I+(1-t')\N$ that belong simultaneously to $\PY$. Obviously, as it lies on $\mathcal L$, $\nexists\y\in\PY$ that is closer to it.
		
		Let $\mathbf C$ be this intersection. $S$ being a linear scaling, it can be expressed in the form $S(\y)=\mathbf A\y+\mathbf b$ with $\mathbf A$ an $m\times m$ diagonal matrix with entries $a_i>0$, and $\mathbf b\in\R^m$. Applying this scaling to the objective space modifies $\mathbf C$ to $\mathbf C'=\mathbf A\mathbf C+\mathbf b$, $\I$ to $\I'=\mathbf A\I+\mathbf b$ and $\N$ to $\N'=\mathbf A\N+\mathbf b$. Because the scaling preserves orderings of the objectives, $\mathbf C'$ remains non-dominated, and $\I'$ and $\N'$ remain the Ideal point and the Nadir point of $\PY$ in the scaled objective space. As $\mathbf C$ belongs to $\mathcal L$ it writes $t\I+(1-t)\N$ for one $t\in[0,1]$, and therefore	
		\begin{align*}\mathbf C'&=\mathbf A(t\I+(1-t)\N)+\mathbf b\\&=t\mathbf{AI}+(1-t)\mathbf{AN}+\mathbf b\\&=t(\mathbf{AI}+\mathbf b)+(1-t)(\mathbf{AN}+\mathbf b)\\&=t\I'+(1-t)\N'\end{align*}
		
		$\mathbf C'$ is thus the unique point belonging to both the Pareto front and to the Ideal-Nadir line in the transformed objective space: it is the center in the scaled objective space.
	\end{proof}
	
	In the bi-objective case ($m=2$), we also show that a linear scaling applied to the objective space does not change the order of Euclidean distances to $\L$. When $\PY\cap\L=\emptyset$, the closest $\y\in\PY$ to $\L$, whose projection on $\L$ produces ${\mathbf C}$, remains the closest after any linear scaling of the objective space.
	
	\begin{prop}[Center invariance to linear scaling, 2D case]
		Let $\y,\y'\in Y\subset\R^2$, and $\L$ be a line in $\R^2$ passing through the two points $\I$ and $\N$
		. Let $\Pi_\L$ be the projection on $\L$. If $\Vert\y-\Pi_\L(\y)\Vert\le\Vert\y'-\Pi_\L(\y')\Vert$, then $\y$ remains closer to $\L$ than $\y'$ after having applied a linear scaling $S:\R^2\rightarrow\R^2$ to $Y$.
		\label{lemma}
	\end{prop}
	
	\begin{proof}
		Let $A$ be the area of the $\I\y\N$ triangle and $A'$ be the area of $\I\y'\N$. Applying a linear scaling $S(\y)=\mathbf A\y+\mathbf b$ with $\mathbf A=\begin{pmatrix}\alpha & 0\\0 & \beta\end{pmatrix}$, $\alpha, \beta > 0$ to $Y$ will modify the areas $A$ and $A'$ by the same factor $\alpha\beta$. Thus, $\Vert S(\y)-\Pi_{S(\L)}(S(\y))\Vert\le\Vert S(\y')-\Pi_{S(\L)}(S(\y'))\Vert$ still holds: in the transformed subspace, $\y$ remains closer to $\L$ than $\y'$.
	\end{proof}
	
	This property is of interest as the solutions in the approximation front $\widehat{\PY}$ will generally not belong to $\L$. Applying a linear scaling to $Y$ in a bi-objective case does not change the solution in $\widehat{\PY}$ that generates $\widehat{\mathbf C}$. However, exceptions may occur for $m\ge3$ as the closest $\y\in\PY$ to $\L$ may not remain the same after a particular affine transformation of the objectives, as seen in the following example:
	
	Let us consider the case of an approximation front composed of the five following non-dominated points (in rows), in a three-dimensional space: $\mathbf P=\begin{bmatrix}1 & 0 & 0\\0 & 1 & 0\\0 & 0 & 1\\0.5 & 0.5 & 0.6\\0.5 & 0.55 & 0.5\end{bmatrix}$. The Ideal point is $\I=(0,0,0)^\top $ and the Nadir point $\N=(1,1,1)^\top $. The squared Euclidean distance to $\L$ of these 5 points equals respectively 2/3, 2/3, 2/3, 0.02/3 and 0.005/3, hence $\mathbf P^5=(0.5,0.55,0.5)^\top $ is the closest point to $\L$. Let us now apply a linear scaling $S(\y)=\mathbf A\y$ with $\mathbf A=\begin{pmatrix}3 & 0 & 0\\0 & 3 & 0\\0 & 0 & 1\end{pmatrix}$. In the modified objective space, we now have $\mathbf{\widetilde{P}}=\begin{bmatrix}3 & 0 & 0\\0 & 3 & 0\\0 & 0 & 1\\1.5 & 1.5 & 0.6\\1.5 & 1.65 & 0.5\end{bmatrix}$, $\widetilde\I=(0,0,0)^\top $ and $\widetilde\N=(3,3,1)^\top $. The squared distances to $\widetilde\L$ after scaling are now respectively 1710/361, 1710/361, 342/361, 3.42/361, 4.275/361. After scaling, the fourth point becomes the closest to the line. As the projection of the latter on $\L$ is different from the projection of the fifth point, the center of the Pareto front will change after this scaling.
	
	\subsubsection*{Low sensitivity to Ideal and Nadir variations}	
	Another positive property is the low sensitivity of $\mathbf C$ with regard to extreme points. This property is appealing because the Ideal and the Nadir will be estimated with errors at the beginning of the search (cf. Section \ref{sec:estimation}) and having a stable target $\mathbf C$ prevents dispersing search efforts.
	
	Under mild assumptions, the following Proposition expresses the low sensitivity in terms of the norm of the gradient of $\mathbf C$ with respect to $\mathbf N$. Before, Lemma \ref{lemme:vecteur_normal} gives a condition on the normal vector to the Pareto front that will be needed to prove the Proposition.
	
	\begin{lemma}
		\label{lemme:vecteur_normal}
		Let $\y^*\in\R^m$ be a Pareto optimal solution, and the Pareto front be continuous and differentiable at $\y^*$ with $\mathbf d\in\R^m$ the normal vector to the Pareto front at $\y^*$. Then all components of $\mathbf d$ have the same sign.
	\end{lemma}
	
	\begin{proof}
		Because of the differentiability assumption at $\y^*$ and the definition of Pareto dominance, $\mathbf d$ cannot have null components. Suppose that some components in $\mathbf d$ have opposite signs, $\mathbf d^+$ corresponding to positive ones and $\mathbf d^-$ to negatives ones, $\mathbf d=[\mathbf d^+, \mathbf d^-]^\top$. Let $\varepsilon^+$ and $\varepsilon^-$ be two small positive scalars such that $\frac{\varepsilon^+}{\varepsilon^-}=\frac{\sum_{i:d_i<0} {d_i}^2}{\sum_{i:d_i>0} {d_i}^2}$. Then, $\mathbf f = \y^* + \begin{pmatrix} - \varepsilon^+ \mathbf d^+ \\ \varepsilon^- \mathbf d^- \end{pmatrix}$ dominates $\y^*$ and belongs to the local first order approximation to $\PY$ since  $\mathbf d^\top (\mathbf f - \mathbf C) = 0$, which is a contradiction as $\y^*$ is Pareto optimal.
	\end{proof}
	
	\begin{prop}[Stability of the Center to perturbations in Ideal and Nadir]
		\label{prop:insensibilite}
		Let $\PY$ be locally continuous and $m-1$ dimensional around its center $\mathbf C$. Then, $\vert\frac{\partial C_i}{\partial N_j}\vert<1,\ i,j=1,\dotsc,m$ where $\mathbf N$ is the Nadir point, and the variation $\Delta\mathbf C$ of $\mathbf C$ induced by a small variation $\Delta\mathbf N$ in $\N$ verifies $\Vert\Delta\mathbf C\Vert_2<\Vert\Delta\mathbf N\Vert_2$. A similar relation stands for small Ideal points variations, $\Vert\Delta\mathbf C\Vert_2<\Vert\Delta\mathbf I\Vert_2$.
	\end{prop}
	
	\begin{proof}
		If $\PY$ is locally continuous and $m-1$ dimensional, $\mathbf C$ is the intersection between $\mathcal L$ and $\PY$. For simplicity, the Pareto front is scaled between 0 and 1, that is, $\I=\mathbf0_m$ and $\N=\mathbf1_m$. Proposition \ref{prop:intersection} ensures that the center is not modified by such a scaling. The tangent hyperplane to $\PY$ at $\mathbf C$ writes $\mathbf d^\top\mathbf f+e=0$ where $\mathbf d\in\R^m$, the normal vector to the tangent hyperplane, and $e\in\R$ depend on $\PY$ and are supposed to be known. Lemma \ref{lemme:vecteur_normal} ensures that $d_i,\ i=1,\dotsc,m$ have the same sign, that we choose positive.
		$\mathbf C$ satisfies both $\mathbf d^\top \mathbf C=-e$ and $\mathbf C=(1-\alpha_C)\I+\alpha_C\N=\alpha_C\mathbf1_m$ for some $\alpha_C\in]0,1[$. Hence, \[\mathbf C=\frac{-e}{\mathbf d^\top \N}\N,\ C_i=\frac{-e}{\mathbf d^\top \N}N_i\]
		$\forall j=1,\dotsc,m,j\ne i$,
		\[\frac{\partial C_i}{\partial N_j}=\frac{eN_id_j}{(\mathbf d^\top \N)^2}=\frac{-d_j}{\sum_kd_kN_k}C_i=\frac{-d_j}{\sum_kd_k}C_i\]
		For $i=j$,
		\[\frac{\partial C_i}{\partial N_i}=\frac{-e\mathbf d^\top \N+eN_id_i}{(\mathbf d^\top \N)^2}=\frac{C_i}{N_i}-\frac{C_i}{\sum_kd_kN_k}=C_i\left(1-\frac{d_i}{\sum_kd_k}\right)\]
		$C_i=\alpha_C\in]0,1[\ \forall i=1,\dotsc,m$ and as the $d_i$'s share the same sign, $\vert d_i\vert\le\vert\sum_kd_k\vert$. Therefore, $\vert \frac{\partial C_i}{\partial N_i}\vert<1$ and $\vert \frac{\partial C_i}{\partial N_j}\vert<1$~.
		Consider now that $\N$ is modified into $\N+\Delta\N$, which changes the center to $\mathbf C+\Delta\mathbf C$. One has $\Delta\mathbf C=\nabla\mathbf C\cdot\Delta\N$ where $\nabla\mathbf C$ is the $m\times m$ matrix with entries $\frac{\partial C_i}{\partial N_j}$. Rearranging the terms of the derivatives into matrix form yields
		\[\nabla\mathbf C=\alpha_C\left[I_m-\underbrace{\frac{1}{\sum_kd_k}\begin{pmatrix}
			d_1 & d_2 & \cdots & d_m\\
			\vdots & \vdots & \vdots & \vdots\\
			d_1 & d_2 & \cdots & d_m
			\end{pmatrix}}_D\right]\]
		where $I_m$ stands for the identity matrix here.
		$D$ is a rank 1 matrix with positive entries whose rows sum to 1, and has eigenvalues 0 and 1 with respective multiplicity $m-1$ and 1. Consequently, $\nabla\mathbf C$'s largest eigenvalue is $\alpha_C\in]0,1[$. Finally, $\Vert\Delta\mathbf C\Vert_2\le\Vert\nabla\mathbf C\Vert_2\Vert\Delta\mathbf N\Vert_2\le\Vert\Delta\mathbf N\Vert_2$. By symmetry, the Proposition extends to the sensitivity of the center to the Ideal point, $\vert\frac{\partial C_i}{\partial I_j}\vert<1,\ i,j=1,\dotsc,m$ and $\Vert\Delta\mathbf C\Vert_2<\Vert\Delta\mathbf I\Vert_2$.
	\end{proof}
	
	Proposition \ref{prop:insensibilite} is a local stability result. Without formal proof, it is observed that the center will be little affected by larger errors in Ideal and Nadir positions when compared to alternative definitions of the center. A typical illustration is as follows: the Nadir point is moved by a large amount in one objective (see Figure \ref{fig:streched_center}).
	The center is shifted by a relatively small amount and will continue to correspond to an area of equilibrium between \emph{all} objectives. Other definitions of the center, typically those based on the barycenter of $\PY$ would lead to a major displacement of $\mathbf C$. In Figure~\ref{fig:streched_center}, the barycenter on $\PY$ signaled by $\mathbf B$ and $\mathbf B'$ has $B'_2\approx I_2$, which does not correspond to an equilibrated solution as the second objective would almost be at its minimum.
	
	\begin{figure}[!ht]
		\centering
		\includegraphics[width=0.8\textwidth]{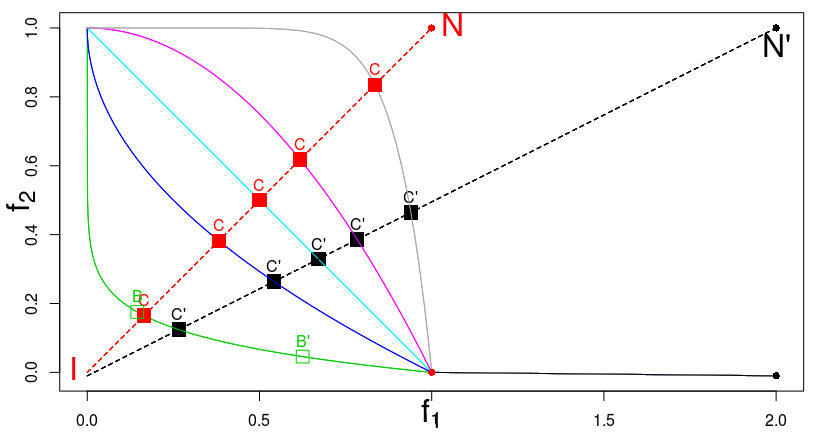}
		\caption{Illustration of the global stability of the center in 2D: adding the black part to the colored Pareto fronts will highly modify them and $\mathbf N'$ becomes the new Nadir point. The new center $\mathbf C'$ is relatively close to $\mathbf C$ despite this major $\N$ modification. $\mathbf B$, a barycenter-based center would be much more affected, and would no longer correspond to an equilibrium.}
		\label{fig:streched_center}
	\end{figure}
	
	\subsection{Estimation of the Pareto front center using Gaussian processes}
	\label{sec:estimation}
	Now that we have given a definition of $\mathbf C$ relying on $\PY$ through $\I$ and $\N$, let us discuss the estimation of $\mathbf C$. The \emph{real} front $\PY$ is obviously unknown and at any stage of the algorithm, we solely have access to an approximation front $\widehat{\PY}$. The empirical Ideal and Nadir points (computed using $\widehat{\PY}$) could be weak estimates in the case of a biased approximation front.
	Thus, we propose an approach using the GPs $Y_j(\cdot)$ to better estimate $\I$ and $\N$ through conditional simulations.
	
	\subsubsection*{Estimating I and N with GP simulations}
	Estimating the Ideal and the Nadir point accurately is a difficult task.
	Indeed, obtaining $\I$ is equivalent to finding the minimum of each $f_j(\cdot), j=1,\dotsc,m$, which corresponds to $m$ classical mono-objective optimization problems. 
	Prior to computing $\N$, the whole Pareto front has to be unveiled but this is precisely our primary concern!
	Estimating $\N$ before running the multi-objective optimization has been proposed in \cite{DebNadir,BechikhNadir} using modified EMOAs to emphasize extreme points. We aim at obtaining sufficiently accurate estimators $\widehat{\I}$ and $\widehat{\N}$ of $\I$ and $\N$ rather than solving these problems exactly. The low sensitivity of $\mathbf C$ with regard to $\I$ and $\N$ discussed previously suggests that the estimation error should not be a too serious issue for computing $\mathbf C$.
	As seen in Section \ref{sec:bayesian_optimization}, given $s$ simulation points $\x^{n+1},\dotsc,\x^{n+s}$, possible responses at those locations can be obtained through the conditional GPs $Y_j(\cdot), j=1,\dotsc,m$. The simulated responses can be filtered by Pareto dominance to get $n_{sim}$ simulated fronts $\widetilde{\PY}^{(k)}$.
	The Ideal and Nadir points are then estimated by
	$\widehat{I_j}=\underset{k=1,\dotsc,n_{sim}}{\text{median}}\left(\underset{\y\in\widetilde{\PY}^{(k)}}{\min}~y_j\right)$; $\widehat{N_j}=\underset{k=1,\dotsc,n_{sim}}{\text{median}}\left(\underset{\y\in\widetilde{\PY}^{(k)}}{\max}~y_j\right)$, $j=1,\dotsc,m$.
	
	Notice that the definition of $\I$ is not based on the Pareto front. Hence the estimation of $I_j$ does not require $m$-dimensional simulated fronts, but only single independently simulated responses $\widetilde{Y_j}^{(k)}$. By contrast, as the Nadir point needs the front to be defined, simulated fronts $\widetilde{\PY}^{(k)}$ are mandatory for estimating $\N$.
	
	GP simulations are attractive for estimating extrema because they not only provide possible responses of the objective functions but also take into account the surrogate's uncertainty. It would not be the case by applying a (multi-objective) optimizer to a deterministic surrogate such as the conditional mean functions. 
	Even so, they rely on the choice of simulation points $\x^{n+i}, i=1,\dotsc,s$ (in a $d$-dimensional space).	For technical reasons (Cholesky or spectral decomposition of $\Gamma_j$ required for sampling from the posterior), the number of points is restricted to $s \lessapprox5000$. $\x^{n+i}$ have thus to be chosen in a smart way to make the estimation as accurate as possible.
	In order to estimate $\I$ or $\N$, GP simulations are performed at $\x$'s that have a large probability of contributing to one component of those points: first, the kriging mean and variance of a very large sample $\mathbb S\subset X$ is computed. The calculation of $\widehat{y_j}(\mathbb S)$ and $s_j(\mathbb S)$ is indeed tractable for large samples contrarily to GP simulations. Next, $s$ designs are picked up from $\mathbb S$ using these computations. In order to avoid losing diversity, the selection is performed using an importance sampling procedure \cite{bect2017bayesian}, based on the probability of contributing to the components $I_j$ or $N_j$.
	
	\vskip\baselineskip
	As $I_j=\underset{\x\in X}{\min}~f_j(\x)$ good candidates are $\x$'s such that $\Pr(Y_j(\x)<a_j)$ is large. To account for new evaluations of $f_j$, a typical value for $a_j$ is the minimum observed value in the $j$-th objective, $\underset{i=1,\dotsc,n}{\min}f_j(\x^{i})$.
	According to the surrogate, such points have the greatest probability of improving over the currently best value if they were evaluated.
	
	\vskip\baselineskip
	Selecting candidates for estimating $\N$ is more demanding. 
	Indeed, as seen in Definition~\ref{defnadir}, $N_j$ is not the maximum value over the whole objective space $Y$ but over the unknown $\PY$, i.e., each $N_j$ arises from a ND point.
	Thus the knowledge of an $m$-dimensional front is mandatory for estimating $\N$. The best candidates for $\N$'s estimation are, by Definition~\ref{defextremepoints}, extreme design points.
	Quantifying which points are the most likely to contribute to the Nadir components, in other terms produce extreme points, is a more difficult task than its pendant for the Ideal. Good candidates are $\mathbf x$'s such that the sum of probabilities $\Pr(Y_j(\x)>\widehat{\nu}^{j}_j, \Y(\x) \text{ ND})+ \Pr(\Y(\x)\preceq\widehat{\pmb\nu}^j)$ is large.
	For reasons of brevity, the procedure is detailed in Appendix~\ref{annexes:estimation_nadir}.
	
	Since the optimization is directed towards the center of the Pareto front, the metamodel may lack precision at extreme points. 
It might be tempting to episodically target these parts of the Pareto front to improve $\I$ and $\N$'s estimation. But this goes against the limited budget of calls to $\mathbf f(\cdot)$
and it is not critical since the center is quite stable with respect to $\I$ and $\N$'s inaccuracies (Proposition~\ref{prop:insensibilite}).
Since the optimality of solutions is favored over the attainment of the exact center of the Pareto front, this option has not been further investigated.
	
	\subsubsection*{Ideal-Nadir line and estimated center}
	To estimate $\I$ and $\N$, we first select $s=5000$ candidates from a large space-filling Design of Experiments (DoE) \cite{halton,sobol}, $\mathbb S\subset X$, with a density proportional to their probability of generating either a $I_j$ or a $N_j$ as discussed before. 
	$s/2m$ points are selected for the estimation of each component of $\mathbf I$ and $\mathbf N$. $n_{sim}$ conditional GP simulations are then performed at those $\x^{n+i},i=1,\dotsc,s$ in order to generate simulated fronts, whose Ideal and Nadir points are aggregated through the medians to produce the estimated $\widehat{\I}$ and $\widehat{\N}$. The resulting simulated fronts are biased towards particular parts of the Pareto front (extreme points, individual minima).
	
	Experiments have shown significant benefits over methodologies that choose $\x^{n+i}$'s according to their probability of being not dominated by the whole approximation front, or that use $s$ points from a space-filling DoE \cite{doe} in $X$.
	Figure~\ref{fig:component_estimation_comparison} compares the component estimation of $\I$ and $\N$ for different techniques during one optimization run with $m=3$ objectives. X.IN (blue curve) corresponds to our methodology. The other curves stand for competing methodologies: X.LHS (green) selects the $\x^{n+i}$ from a space-filling design, and X.ND (red) chooses them according to their probability of being non-dominated with respect to the entire front. NSGA-II (gold) does not select design points $\x^{n+i}$ to perform GP simulations but rather uses the Ideal and Nadir point found by one run of the NSGA-II \cite{NSGAII} multi-objective optimizer applied to the kriging predictors $\widehat{y_i}(\cdot),i=1,\dotsc,m$. The black dashed line corresponds to the component of the current empirical front, a computationally much cheaper estimator. The bold dashed line shows $\I$ and $\N$'s true components.
	
	\begin{figure}[h!]
		\centering
		\includegraphics[width=\textwidth]{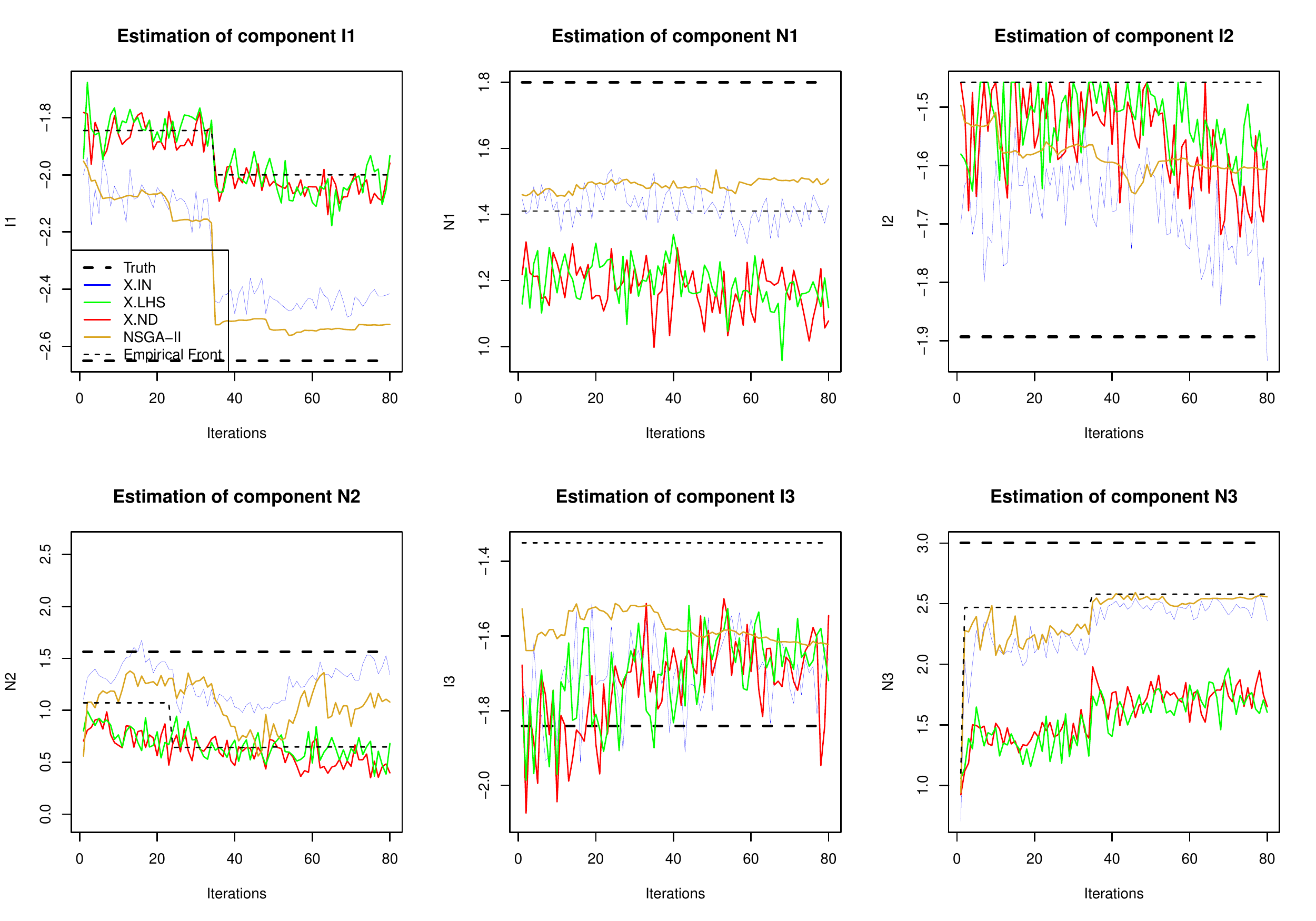}
		\caption{Example of estimation of $\I$ and $\N$ using different techniques. The proposed methodology (blue) is able to consistently produce close estimates to $\I$'s and $\N$'s components (bold black dashed line).}
		\label{fig:component_estimation_comparison}
	\end{figure}

Our methodology outperforms the two other simulation techniques, because they do not perform the simulations specifically at locations that are likely to correspond to an extreme design point or to a single-objective minimizer. Benefits are also observed compared with the empirical Ideal and Nadir points, that are sometimes poor estimators (for example for $I_1$, $I_2$ and $N_2$). Using the output of a multi-objective optimizer (here NSGA-II) applied to the kriging mean functions is also a promising approach but has the drawback of not considering any uncertainty in the surrogates (that may be large at the extreme parts of the Pareto front). It also suffers from classical EMOA's disadvantages, e.g. several runs would be required for more reliable results and convergence can not be guaranteed. Note that as these methods rely on the surrogates they are biased by the earlier observations: the change of the empirical Ideal or Nadir point has an impact on the estimation. However, the X.IN, X.LHS and X.ND estimators compensate by considering the GPs uncertainty to reduce this bias.

As we are in fine not interested in the Ideal and the Nadir point but in the Pareto front center, we want to know if these estimations lead to a good $\widehat{\mathbf C}$. Proposition~\ref{prop:insensibilite} suggests that the small Ideal and Nadir estimation error should not be a too serious concern. This is confirmed by Figure~\ref{fig:estimated_pareto_front_center}, where the center estimation error is low with respect to the range of the Pareto front.

\begin{figure}[h!]
	\centering
	\includegraphics[width=0.7\textwidth]{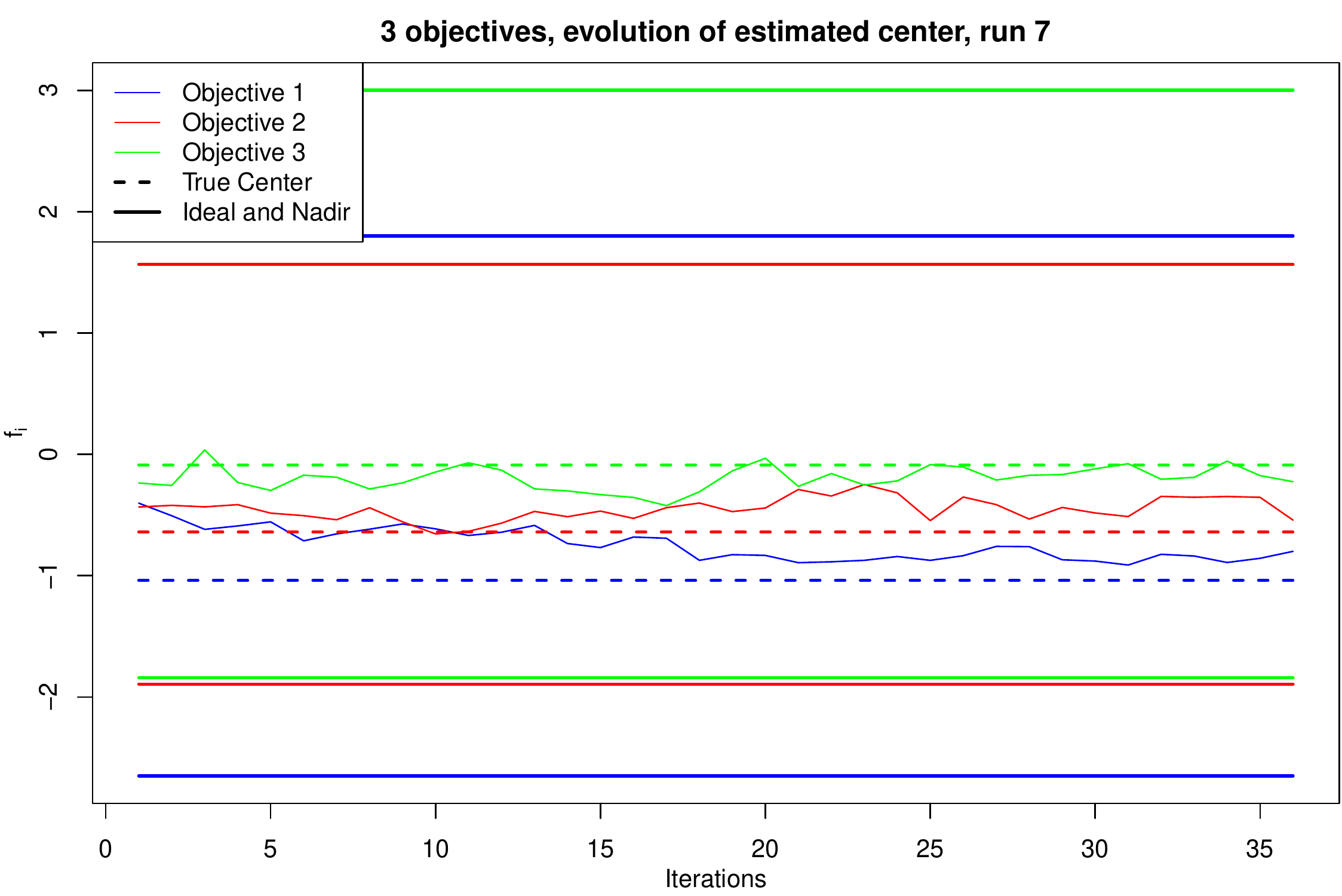}
	\caption{Evolution of the estimated center during one run (using $\widehat{\I}$ and $\widehat{\N}$ from Figure~\ref{fig:component_estimation_comparison}): $\widehat{\mathbf C}$'s components are close to the true ones.}
	\label{fig:estimated_pareto_front_center}
\end{figure}

Figure \ref{fig:center_estimation} shows an example of one GP simulation targeting the extremes of the Pareto front.
Notice the difference between the current empirical Pareto front (in blue) and the simulated front for $\mathbf N$ and $\mathbf I$ (in black): the extreme points which are simulated go well beyond those already observed.

Linearly extending the Pareto front approximation \cite{Paint} and taking the intersection with $\widehat\L$ was originally considered for defining $\widehat{\mathbf C}$. But as an $m$-dimensional interpolated Pareto front is not necessarily composed of only $m-1$ dimensional hyperplanes (but is a collection of polytopes of dimension at most $m-1$), the intersection with an $m$-dimensional line does not necessarily exist.

	\begin{figure}[!ht]
		\centering
		\includegraphics[width=0.8\textwidth]{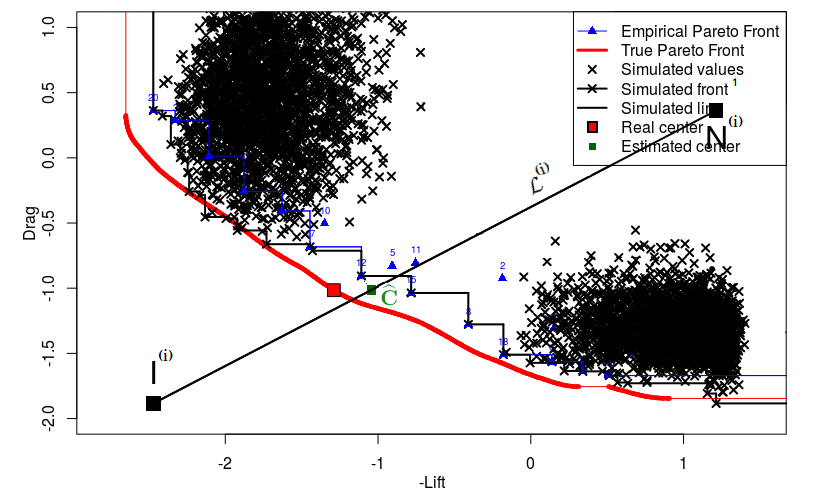}
		\caption{One GP simulation targeting the extremes of the Pareto front to enhance the estimation of $\I$ and $\N$. The projection of the closest non-dominated point to $\L$ on it is the estimated center (in green). The real center (in red) lies close to the estimated center and to the estimated Ideal-Nadir line.}
		\label{fig:center_estimation}
	\end{figure}

	\section{An infill criterion to target the center of the Pareto front}
	\label{sec:target}
	
	\subsection{Targeting the Pareto front center with the reference point}
	\label{sec:ehi_refpoint_targeting}
	
	Our approach starts from the observation that any region of the objective space can be aimed targeted with EHI solely by controlling the reference point $\mathbf R$. Indeed, as $\y\npreceq\RR\Rightarrow I_H(\y;\RR)=0$, the choice of $\mathbf R$ is instrumental in deciding the combination of objectives for which improvement occurs, the \emph{improvement region}:
	\begin{equation*}
	\mathcal I_{\mathbf R} :=\{\y\in Y : \y\preceq\mathbf R\}~.
	\end{equation*}
	As shown in Fig.~\ref{fig:4R}, the choice of $\RR$ defines the region in objective space where $I_H>0$ and where the maximum values of EHI are expected to be found.
	The choice of $\RR$ is crucial as it defines the region in objective space that is highlighted. To our knowledge, $\RR$ has always been chosen to be dominated by the whole approximation Front (that is, $\RR$ is at least the empirical Nadir point, which corresponds to the case of $\RR\mathbf1$ in Fig.~\ref{fig:4R}).
	The targeting ability of $\RR$ can and should however be taken into account: for example, solutions belonging to the left part of the Pareto front in Fig.~\ref{fig:4R} can be aimed at using EHI$(\cdot;\RR\mathbf2)$ instead of the more general EHI$(\cdot;\RR\mathbf1)$.
	
	\begin{figure}[!ht]
		\centering
		\includegraphics[width=0.4\textwidth]{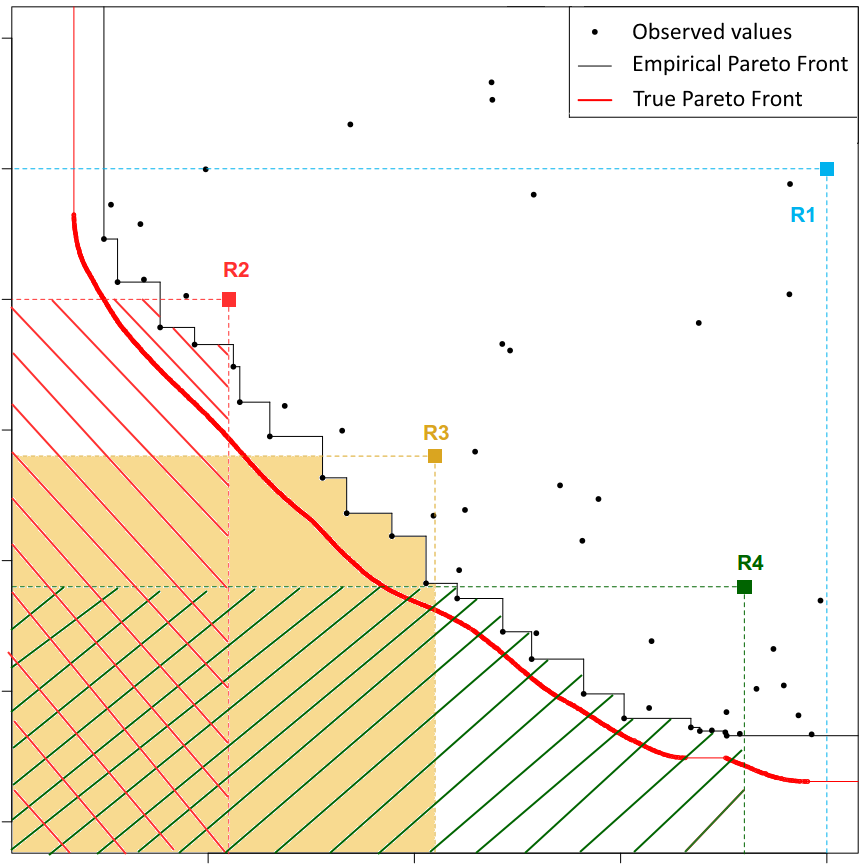}
		\caption{Different reference points and the areas $\mathcal I_{\RR}$ that are targeted}
		\label{fig:4R}
	\end{figure}
		
	Because of the extremely limited number of possible calls to the objective functions, we would like to prioritize the search by first looking for the Pareto front center: we implicitly prefer the center of the Pareto front over other solutions.
	This is implemented simply by setting the reference point as the estimated center, 
	$\RR \equiv \widehat{\mathbf C}$. Then, the algorithm maximizes $\text{EHI}(\mathbf x;\widehat{\mathbf C})$ on $\mathbf x$.
	In contrast to other works that set $\RR$ at levels dominated by all Pareto optimal points, $\RR$ at $\widehat{\mathbf C}$ will typically be non-dominated\footnote{When using the projection of the closest non-dominated point on the line, exceptions may occur and $\widehat{\mathbf C}$ may be dominated. In that case, it as to be slightly moved towards $\widehat{\I}$.} since it is close to the empirical Pareto front.
	
	$\widehat{\mathbf C}$ corresponds to the center of the \emph{current} approximation front, at a given moment $t$. 
	Since the goal is to find optimal, better, solutions, it makes sense to look for points dominating it: at each iteration of the algorithm, after having estimated $\widehat{\mathbf C}$, improvements over it are sought by maximizing EHI$(\cdot,\widehat{\mathbf C})$.
	
	\subsection{mEI, a computationally efficient proxy to EHI}
	\label{sec:mEI}
	Choosing the Pareto front center, a non-dominated point, as reference point in EHI has an additional advantage: it allows to define a criterion that can replace EHI for targeted optimization at a much lower computational cost. 
	We name this criterion mEI for multiplicative Expected Improvement.
	\begin{defn}[mEI criterion]
		\label{def:mEI}
		The multiplicative Expected Improvement is the product of Expected Improvements in each objective defined in Equation~(\ref{eq:EI0})
		\begin{equation}
		\text{mEI}(\cdot;\RR) := \prod_{j=1}^{m}\text{EI}_j(\cdot;R_j)~.
		\label{eq:mEI}
		\end{equation}
	\end{defn}
	mEI is a natural extension of the mono-objective Expected Improvement, as $(f_{min}-\widehat{y}(\x))_+$ is replaced by $\prod(\RR-\widehat{\Y}(\x))_+$.
	
	mEI is an attractive infill criterion in several ways. 
	First, it is able to target a part of the objective space via $\RR$ as the Improvement function it is built on differs from zero only in $\mathcal I_{\RR}$. 
	Conversely of course, as it does not take the shape of the current approximation front into account, mEI cannot help in finding well-spread Pareto optimal solutions. 
	
	Second, when $\widehat{\PY}\npreceq\RR$, mEI is equivalent to EHI but it is much easier to compute. Contrarily to EHI, mEI does not imply the computation of an $m$-dimensional hypervolume which potentially requires Monte-Carlo simulations (cf. Section~\ref{sec:multiobj}). Its formula is analytical (substitute Equation~(\ref{eq:EI}) into (\ref{eq:mEI})) and can easily be parallelized. 
	\begin{prop}[EHI-mEI equivalence]
		Let $Y_1(\cdot),\dotsc ,Y_m(\cdot)$ be independent GPs fitted to the observations $(\mathbb X,\mathbb Y)$, with the associated empirical Pareto front $\widehat{\PY}$.
		If $\widehat{\PY}\npreceq\RR$, $\text{EHI}(\cdot;\RR)=\text{mEI}(\cdot;\RR)$.
		\label{prop:ehi_mei}
	\end{prop}
	\begin{proof}
		Let $\widehat{\PY}\npreceq\RR$. For such a reference point, the hypervolume improvement is \[I_H(\y;\RR)=H(\widehat{\PY}\cup\{\y\};\RR)-H(\widehat{\PY};\RR)=H(\{\y\};\RR)=\left\{
		\begin{array}{ll}
		\prod_{j=1}^{m}(R_j-y_j) & \text{ if }\y\preceq\RR\\
		0 & \text{ else }
		\end{array}
		\right.\]
		With the $(.)_+$ notation, $I_H(\y;\RR)=\prod_{j=1}^{m}(R_j-y_j)_+$ and EHI($\x;\RR)$ reduces to $\E[\prod_{j=1}^{m}(R_j-Y_j(\mathbf x))_+]=\prod_{j=1}^{m}\E[(R_j-Y_j(\mathbf x))_+]$ as the $Y_j(\cdot)$ are independent. This is the product of $m$ Expected Improvements
		with objectives at the thresholds $R_j$.
	\end{proof}
	
	\begin{figure}[!ht]
		\centering
		\includegraphics[width=0.6\textwidth]{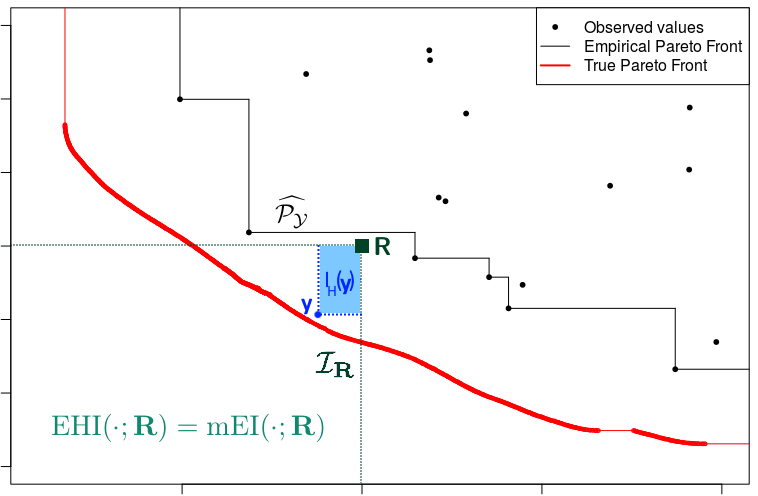}
		\caption{When using a non-dominated reference point w.r.t. $\widehat{\PY}$, EHI and mEI are equivalent. The area in blue corresponds to a sample of both the product of improvements w.r.t. $R_j$ and the hypervolume improvement.}
		\label{fig:ehi_ei}
	\end{figure}
	
	Third, being a product of Expected Improvements, $\nabla\text{mEI}(\x;\RR)$ is computable as \[\nabla\text{mEI}(\x;\RR)=\sum_{i=1}^{m}\left[\nabla\text{EI}_i(\x;R_i)\prod_{\substack{j=1 \\ j\neq i}}^m\text{EI}_j(\x;R_j)\right]\]
	where $\nabla EI(\x;\RR)$ has closed form, see \cite{roustant2012dicekriging} for instance. This offers the additional possibility of combining global optimization with gradient based methods when maximizing mEI$(\cdot;\RR)$. In comparison, EHI's gradient has no closed-form.
	
	As we shall soon observe with the numerical experiments in Section~\ref{sec:results}, mEI is an efficient infill criterion for attaining the Pareto front provided that $\RR$ is taken in the non-dominated neighborhood of the Pareto front.
	It is important that $\RR$ is not dominated, not only for the equivalence with EHI to hold. 
	Indeed, mEI with a dominated $\RR$ may lead to clustering: let $\y^{i_0}=\mathbf f(\x^{i_0})\in\PY$ such that $\y^{i_0}\preceq\RR$. 
	Then, because improvement over $\RR$ is certain at $\x^{i_0}$, mEI$(\x^{i_0};\RR)$ will be large and often maximal in the vicinity of $\x^{i_0}$. Clustering in both the objective and the design space will be a consequence, leading to	ill-conditioned covariance matrices. 
	Taking a non-dominated reference point instead will diminish this risk as $\prod_{j=1}^{m}(R_j-y_j)_+=0~~\forall\y\in\widehat{\PY}$, and no already observed solution will attract the search. 
	If the reference point is too optimistic, the mEI criterion makes the search exploratory as	the only points $\x$ where progress is achieved during GP sampling are those with a large associated uncertainty $s^2(\x)$.
	A clear example of a too optimistic reference point comes from the straightforward generalization of the default single objective EI$(\cdot;f_{min})$ to multiple objectives: it is the criterion $\prod_{j=1}^m \text{EI}(\cdot;f_{j,min}) ~\equiv~ \text{mEI}(\cdot;\mathbf I)$, that is, the mEI criterion with the empirical Ideal as a reference. In non-degenerated problems where the Ideal is unattainable, sequentially maximizing $\text{mEI}(\cdot;\mathbf I)$ will be close to sequentially maximizing $s^2(\x)$.
	
	Thus, $\RR$ has to be set up adequately. This is achieved in the proposed C-EHI algorithm by selecting the estimated Pareto front center as reference point and maximizing mEI$(\mathbf x ; \widehat {\mathbf C})$.
	
	\section{Detecting local convergence to the Pareto front}
	\label{sec:conv}
	The Pareto front center may be reached before depletion of the computational resources. If the algorithm continues targeting the same region, it can no longer improve the center,	and the infill criterion will favor the most uncertain parts of the design space. 
	It is necessary to detect convergence to the center so that a broader part of the Pareto front can be searched in the remaining iterations, as will be explained in Section \ref{sec:seconde_phase}. 
	In this section, we propose a novel method for checking convergence to the center. It does not utilize the mEI value which was found too unstable to yield a reliable stopping criterion.
	Instead, the devised test relies on a measure of local uncertainty.
	
	To test the convergence to a local part of the Pareto front, we define the \emph{probability of domination in the $Y$ space}\footnote{The probability of domination is also called ``attainment function'' in \cite{TheseBinois}.}, $p(\y)$, as the probability that there exists $\y'\in Y : \y'\preceq\y$. 
	$\y$'s for which $p(\y)$ is close to 0 or to 1 have a small or large probability, respectively, that there exist objective vectors dominating them. 
	On the contrary, $p(\y)$ close to 0.5 indicates no clear knowledge about the chances to find better vectors than $\y$. $p(\y)$ measures how certain domination or non-domination of $\y$ is.
	Formally, the domination $d(\y)$ is a binary variable that equals 1 if $\exists\x\in X:\mathbf f(\x)\preceq\y$ and 0 otherwise. The Pareto front being a boundary for domination, $d$ can also be expressed in the following way
	\[d(\y)=\begin{cases}
	1 \text{ if } \PY\preceq\y\\
	0 \text{ otherwise }
	\end{cases}\]
	$d(\y)$ can be seen as a binary classifier between dominated and non-dominated vectors whose frontier is the Pareto front and which is only known for previous observations $\y\in\mathbb Y$. 
	We now consider an estimator $D(\y)$ of $d(\y)$ that has value 1 when the random Pareto front of the GPs, $\mathcal P_{\Y(\cdot)}$, dominates $\y$, and has value 0 otherwise,
	\begin{equation*}
	D(\y) = \mathbbm 1(\mathcal P_{\Y(\cdot)} \preceq \y)
	\end{equation*}
	The reader interested in theoretical background about the random set $\mathcal P_{\Y(\cdot)}$ is referred to \cite{molchanov2005theory,TheseBinois}.
	$D(\y)$ is a Bernoulli variable closely related to the domination probability through $p(\y) = \Pr(D(\y)=1) = \E[D(\y)]$.
	If $p(\y)$ goes quickly from 0 to 1 as $\y$ crosses the Pareto front, the front	is precisely known around this $\y$.
	
	As the $Y_j(\cdot)$ are independent, it is easy to calculate the probability of domination for a specific $\x$, $\Pr(\Y(\x)\preceq\y)=\prod_{j=1}^{m}\Phi\left(\frac{y_j-\widehat{y}_j(\x)}{s_j(\x)}\right)$. 
	In contrast, the probability of dominating $\y$ at any $\x$ by $\Y(\x)$, $\Pr(\exists\x\in X:\Y(\x)\preceq\y)$, has no closed-form as many overlapping cases occur.
	Even for a discrete  set $\mathbb S=\{\x^{n+1},\dotsc,\x^{n+s}\}$, $\Pr(\exists\x\in \mathbb S:\Y(\x)\preceq\y)$ has to be estimated by numerical simulation because of the correlations in the Gaussian vector $\Y(\mathbb S)$.
	
	To estimate the probability $p(\y)$ that an objective vector $\y$ can be dominated, we 
	exploit the probabilistic nature of the GPs conditioned by previous observations: 
	we simulate $n_{sim}$ GPs, from which we extract the corresponding Pareto fronts $\widetilde{\PY}^{(k)}, k=1,\dotsc,n_{sim}$. 
	$D^{(k)}$ is a realization of the estimator and random variable $D(\y)$,
	\[D^{(k)}(\y)=\mathbbm 1(\widetilde{\PY}^{(k)}\preceq\y)=\begin{cases}
	1 \text{ if } \exists \mathbf z\in\widetilde{\PY}^{(k)} : \mathbf z\preceq\y\\
	0 \text{ otherwise }
	\end{cases}\] 
	Therefore, $p(\y)$ which is the mean of $D(\y)$ can be estimated by averaging the realizations,
	\[
	p(\y) = \lim_{n_{sim} \to \infty} \widehat p(\y) \qquad \text{ where }\qquad
	\widehat p(\y) =
	\frac{1}{n_{sim}}\sum_{k=1}^{n_{sim}}D^{(k)}(\y) ~.
	\] 
	One can easily check that $\widehat p(\y)$ is monotonic with 
	domination: if $\y'\preceq\y$, then every $\widetilde{\PY}^{(k)}$ dominating 
	$\y'$ will also dominate $\y$ and $\widehat p(\y')\le \widehat p(\y)$. 
	
	As discussed in Section \ref{sec:estimation}, the choice of points $\x^{n+i}\in X,i=1,\dotsc,s$ where the GP simulations are performed is crucial. Here, as the simulated Pareto fronts aim 
	at being possible versions of the true front, the $\x$'s are chosen according to 
	their probability of being non-dominated with regard to the current approximation 
	$\widehat{\PY}$ in a roulette wheel selection procedure \cite{LivreDeb} 
	to maintain both diversity and a selection pressure. 
	Using a space-filling DoE \cite{sobol,halton,doe} for the simulations would lead to less dominating simulated fronts, and to an under-estimated probability of dominating $\y$. 
	Another advantage of this technique is that the computational burden resides in the $\x$ selection procedure and the simulation of the GPs. 
	Once the simulated fronts have been generated, $p(\cdot)$ can be estimated for many $\y$'s $\in Y$ without significant additional effort.
	
	The variance of the Bernoulli variable $D(\y)$ is $p(\y)(1-p(\y))$ and can be interpreted as a measure of uncertainty about dominating $\y$. 
	When $p(\y)=1$ or 0, no doubt subsists regarding the fact that $\y$ is dominated or non-dominated, respectively. 
	When half of the simulated fronts dominate $\y$, $p(\y)=0.5$ and $p(\y)(1-p(\y))$ is maximal: uncertainty about the domination of $\y$ is at its highest. 
	
	Here, we want to check convergence to the Pareto front center which, by definition, is located on the estimated Ideal-Nadir line $\widehat{\L}$. 
	We therefore consider the uncertainty measure ($p(\y)(1-p(\y))$) for $\y$ varying along $\widehat{\L}$, convergence at the center being equivalent to a sufficiently small uncertainty of $D(\y)$ along $\widehat{\L}$.
	This leads to saying that convergence to the center has occurred if the \emph{line uncertainty} is below a threshold, $U(\widehat{\L})<\varepsilon$, where the line uncertainty is defined as
	\begin{equation}
	U(\widehat{\L}) := \frac{1}{\vert\widehat{\L}\vert}\int_{\widehat{\L}}p(\y)(1-p(\y))d\y~.
	\label{eq-convcenter}
	\end{equation}
	$\vert\widehat{\L}\vert$
	is the (Euclidean) distance between the estimated Ideal and Nadir points and $\varepsilon$ is a small positive threshold.
	Figure \ref{fig:p1p_milieu} illustrates a case of detection of convergence to the Pareto front center.
	On the left plot, when moving along $\widehat\L$ from $\widehat \I$ to $\widehat \N$, $p(\cdot)$ goes quickly from 0 to 1 when crossing the estimated and real Pareto fronts. 
	The variability between the simulated Pareto fronts is low in the central part, as seen on the right plot: $p(\y)(1-p(\y))$ equals 0 (up to estimation precision) all along $\widehat{\L}$ and in particular near the center of the approximation front where sufficiently many points $\mathbf f(\x)$ have been observed and no further improvement can be achieved.
		
	If $p(\y)$ equals either 0 or 1 along $\widehat{\L}$, all $n_{sim}$ simulated fronts are intersected at the same location by $\widehat{\L}$, thus convergence is assumed in this area.
	To set the threshold $\varepsilon$, we consider that convergence has occurred in the following limit scenarios: 
	as there are 100 integration points on $\widehat{\L}$ for the computation of the criterion (\ref{eq-convcenter}), $p(\y)$ jumps successively from 0 to 0.01 and 1 (or from 0 to 0.99 and 1);
	or $p(\y)$ jumps successively from 0 to 0.005, 0.995 and 1.
	This rule leads to a threshold $\varepsilon=10^{-4}$.
		
	\begin{figure}
		\centering
		\begin{minipage}{.48\textwidth}
			\centering
			\includegraphics[width=\linewidth]{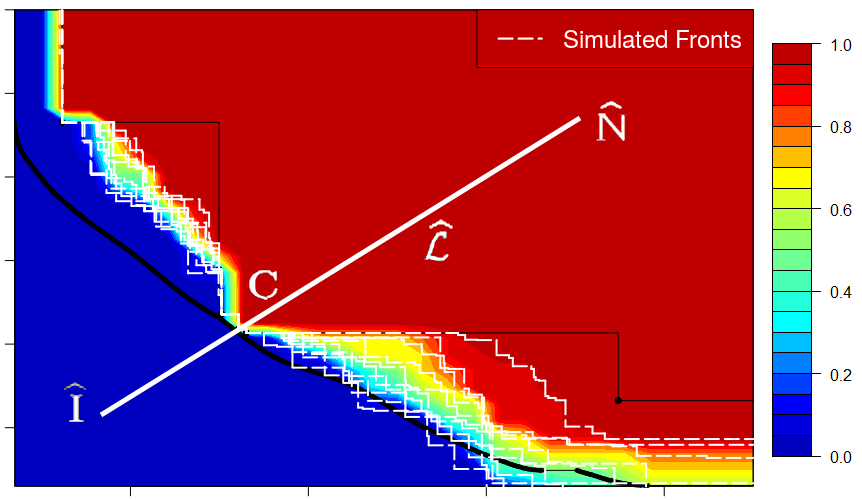}
			(a) $\widehat p$ in $Y$ space
			\label{fig:test1}
		\end{minipage}
		\begin{minipage}{.48\textwidth}
			\centering
			\includegraphics[width=\linewidth]{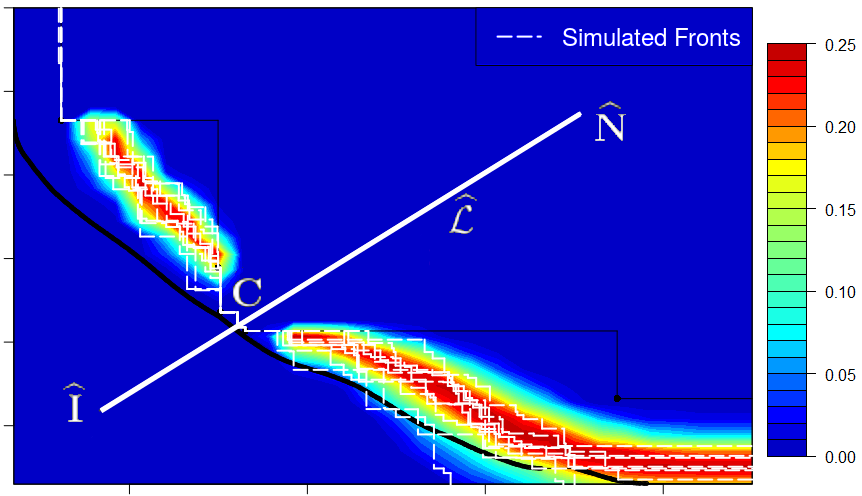}
			(b) $\widehat p(1-\widehat p)$ in $Y$ space
			\label{fig:test2}
		\end{minipage}
		\caption{Detection of convergence to the Pareto front center using simulated fronts. Five of the $n_{sim}=200$ simulated fronts are shown. The approximation $\widehat{\PY}$ (thin black line) has converged towards $\PY$ (thick black curve) at the center of the front (intersection with $\widehat\L$). Consequently, $p(\y)$ grows very fast from 0 to 1 along $\widehat\L$ and the domination uncertainty on the right plot $p(\y)(1-p(\y))$ is null.
			\label{fig:p1p_milieu}
		}
	\end{figure}
	
	\section{Expansion of the approximation front within the remaining budget}
	\label{sec:seconde_phase}
	If convergence to the center of the Pareto front is detected and the objective functions budget is not exhausted, the goal is no longer to search at the center where no direct progress is possible, but to investigate a \emph{wider central part} of the Pareto front.
	A second phase of the algorithm is started during which	a new, fixed, reference point $\RR$ is set for the EHI infill criterion. 
	To continue targeting the central part of the Pareto front, the new $\RR$ has to be located on $\widehat{\L}$. The more distant $\RR$ is from $\PY$, the broader the targeted area in the objective space will be, as $\mathcal I_{\RR}\subset \mathcal I_{\RR'}$ if $\RR\preceq\RR'$. As shown in Figure \ref{fig:zone_refpoints}, $\RR$ is instrumental in deciding in which area solutions are sought.
	After having spent the $b$ remaining calls to the objective functions, we would like to have (i) an approximation front $\widehat{\PY}$ as broad as possible, (ii) which has converged to $\PY$ in the entire targeted area $I_{\RR}$. These goals are conflicting: at a fixed budget $b$, the larger the targeted area, the least $\PY$ will be well described. The reference point leading to the best trade-off between convergence to the Pareto front and width of the final approximation front is sought.
	
	\begin{figure}[!ht]
		\centering
		\includegraphics[width=0.7\textwidth]{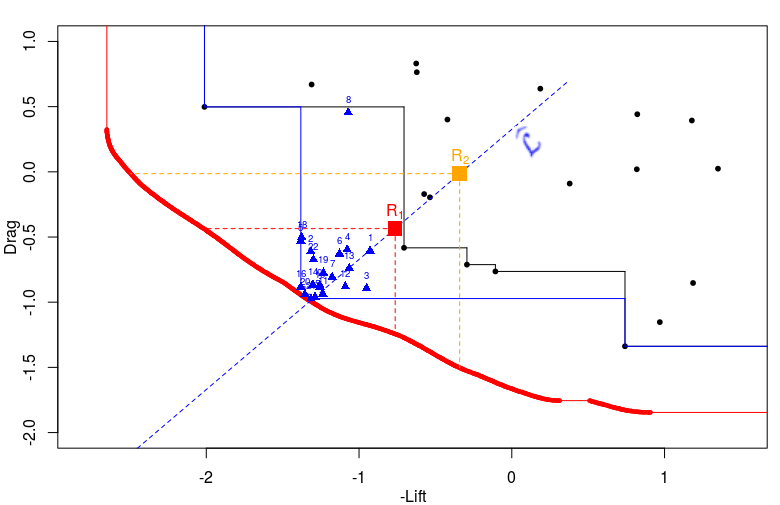}
		\caption{Two possible reference points $\RR_1$ and $\RR_2$ located on $\widehat\L$, and the part of the Pareto front they allow to target when used within EHI}
		\label{fig:zone_refpoints}
	\end{figure}
	
	\vskip\baselineskip
	To choose the \emph{best} reference point for the remaining $b$ iterations, we anticipate the behavior of the algorithm and the final approximation front obtained with a given $\RR$. 
	Candidate reference points $\RR^c, c=1,\dotsc,C$, are uniformly distributed along $\widehat\L$ with $\RR^0=\widehat{\mathbf C}$ and $\RR^C=\widehat{\N}$. 
	Each $\RR^c$ is related to an area in the objective space it targets, $\mathcal I_{\RR^c}$. 
	Starting from the current GPs $\Y(\cdot)$, $C$ virtual optimization scenarios are anticipated by sequentially maximizing EHI $b$ times for each candidate reference point $\RR^c$. 
	We use a Kriging Believer \cite{KrigingBeliever} strategy in which the metamodel is augmented at each virtual iteration using the kriging mean functions $\widehat\y(\x^{*i})$, $\x^{*i}$ being the maximizer of EHI$(\cdot;\RR^c)$ at one of the virtual step $i\in\{1,\dotsc,b\}$. 
	Such a procedure does not modify the posterior mean $\widehat\y(\cdot)$, but it changes the posterior variance $\mathbf s^2(\cdot)$. 
	The conditional GPs $\Y(\cdot)$ augmented by these $b$ Kriging Believer steps are denoted as $\Y^{KB}(\cdot)$.
	
	The optimizations for the $\RR^c$'s are independent and parallel computing can be exploited (in our implementation, it has been done through the \texttt{foreach} \texttt{R} package).
	At the end, $C$ different final Kriging Believer GPs ${\Y}^{KB}(\cdot)$ are obtained that characterize the associated $\RR^c$.
	$\RR$'s close to the center produce narrow and densely sampled final fronts whereas  distant $\RR$'s lead to more extended and sparsely populated fronts, as can be seen in Figure~\ref{fig:virtual_infills_different_R}. 
	
	\begin{figure}[!ht]
		\centering
		\includegraphics[width=.5\textwidth]{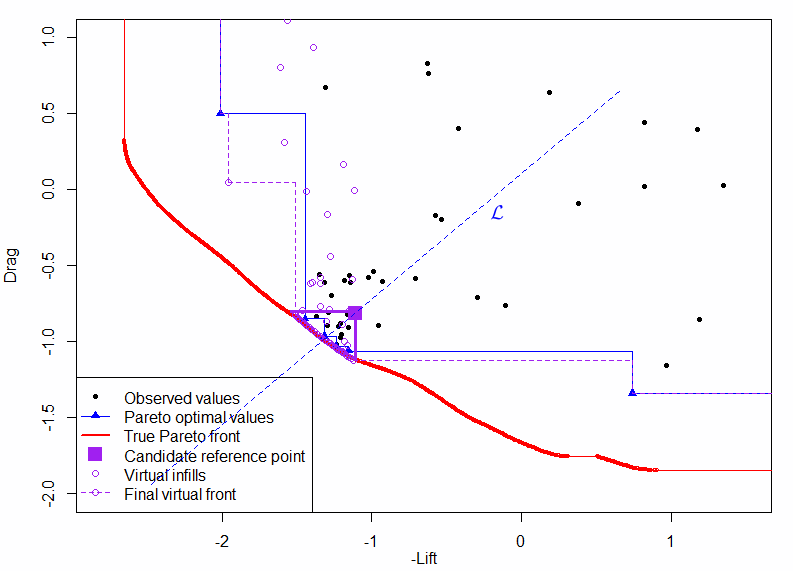}
		\includegraphics[width=.47\textwidth]{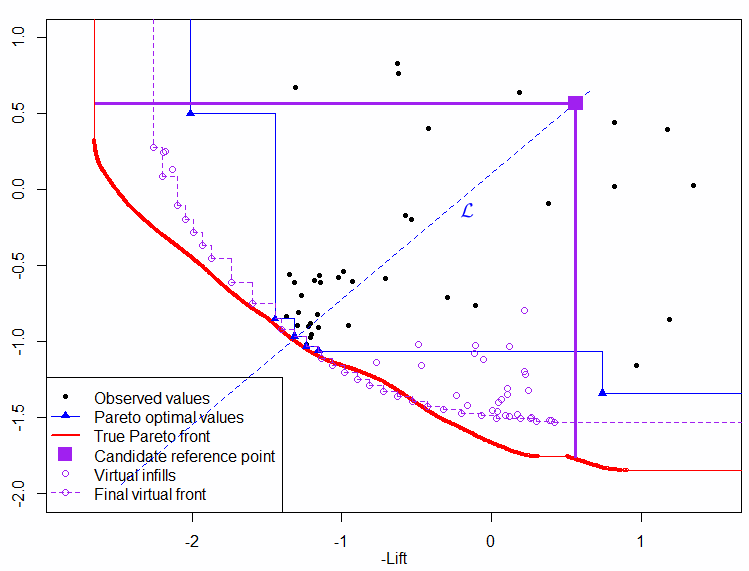}
		\caption{Virtual infills obtained by sequentially maximizing EHI$(\cdot;\RR)$ $b$ times, for two different reference points (purple squares). The shape and sampling density of the final virtual front depends on $\RR$.
		}
		\label{fig:virtual_infills_different_R}
	\end{figure}
	
	\vskip\baselineskip
	To measure how much is known about the Pareto front, we generalize the line uncertainty of Equation~(\ref{eq-convcenter}) to the 
	volume $\mathcal I_\RR$ and define the \emph{volume uncertainty}, $U(\RR;\Y)$ of the GPs $\Y(\cdot)$. 
	The volume uncertainty is the average domination uncertainty $p(\y)(1-p(\y))$ in the volume that dominates $\RR$ bounded by the Ideal point where $p(\y)$ is calculated for $\Y(\cdot)$,
	\begin{equation}
	U(\RR;\Y) := \frac{1}{Vol(\I,\RR)}\int_{\I\preceq\y\preceq\RR}p(\y)(1-p(\y))d\y ~. 
	\label{eq:Uvol}
	\end{equation}
	
	In practice, the estimated Ideal $\widehat \I$ is substituted for the Ideal.
	$U(\RR;\Y)$ quantifies the \emph{convergence} to the estimated Pareto front in the progress region delimited by $\RR$. It is a more rigorous uncertainty measure than others based on the density of points in the $\mathcal Y$ space as it accounts for the possibility of having many inverse images $\x$ to $\y$.
	
	The optimal reference point is the one that creates the largest and sufficiently well populated Pareto front.
	The concepts of augmented GPs and volume uncertainty to measure convergence allow to define the \emph{optimal reference point},
	\begin{equation}
	\begin{split}
	\RR^{*} := \RR^{c^*} \quad & \text{ where }\quad c^*=\max_{c=1,\dotsc,C} c \\
	& \text{ such that }U(\RR^c;\Y^{KB})<\varepsilon
	\label{eq:Rstar}
	\end{split}
	\end{equation}
	Note that the uncertainty is calculated with the augmented GPs $\Y^{KB}(\cdot)$, i.e., the domination probabilities $p(\y)$ in Equation~(\ref{eq:Uvol}) are obtained with $\Y^{KB}(\cdot)$. Associated to $\RR^*$ is the \emph{optimal improvement region}, $\mathcal I_{\RR^*}$, that will be the focus of the search in the second phase.
	For $\RR^*$ to be able to depart from the center, a threshold $\varepsilon$ 10 times larger as the one of Equation~(\ref{eq-convcenter}) is applied.
	The procedure for selecting $\RR$ after local convergence is illustrated in Figures~\ref{fig:choice_R_second_phase} and \ref{fig:apres_second_phase}. The initial DoE is made of 20 points and $\varepsilon=10^{-3}$. 
	Convergence to the center is detected after 26 added points, leaving $b=54$ points in the second phase of the algorithm for a total budget of 100 $\mathbf f(\cdot)$ evaluations.
	Figure~\ref{fig:choice_R_second_phase} shows the final virtual Pareto fronts obtained for two different reference points, as well as simulated fronts sampled from the final virtual posterior (those fronts are used for measuring the uncertainty).
	On the left, the area targeted by $\RR$ is small, and so is the remaining uncertainty ($U(\RR;\Y^{KB})=3\times10^{-6}<10^{-3}$). On the right, a farther $\RR$ leads to a broader approximation front, but to higher uncertainty ($U(\RR;\Y^{KB})=0.0015>10^{-3}$). 
	Figure~\ref{fig:apres_second_phase} represents the approximation front obtained when using the optimal $\RR^*$ ($U(\RR^*;\Y^{KB})=9.4\times10^{-4}$) of Equation~(\ref{eq:Rstar}). 
	A complete covering of $\PY$ in the targeted area is observed. As the remaining budget after local convergence was important in this example (54 iterations), the Pareto front has been almost entirely unveiled.
	
	\begin{figure}[!ht]
		\centering
		\includegraphics[width=\textwidth]{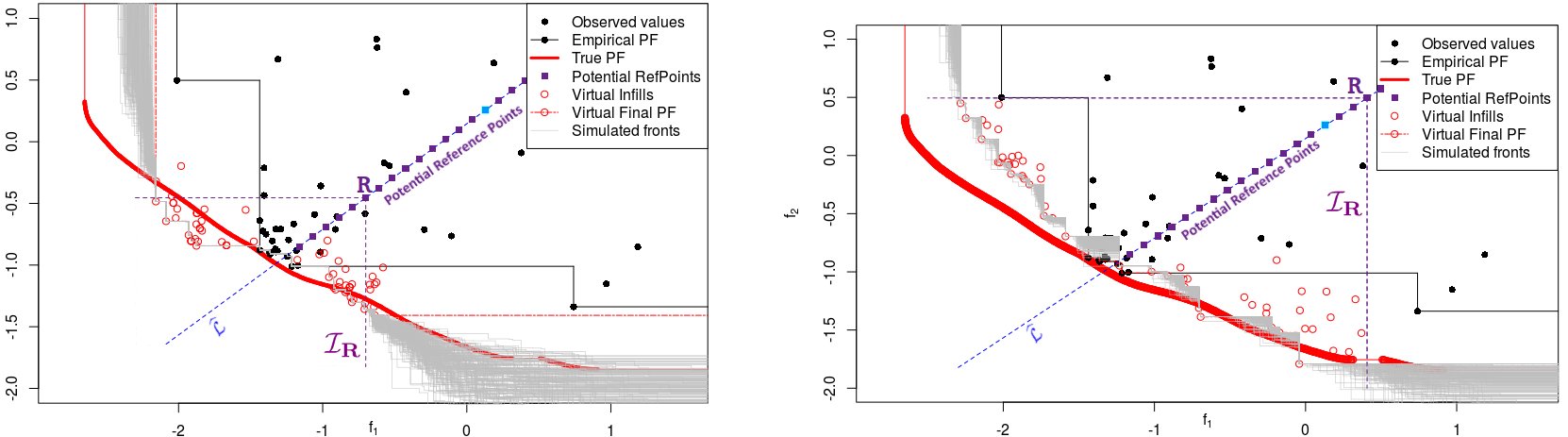}
		\caption{Uncertainty quantification through final virtual fronts.
			The anticipated remaining uncertainty can be visualized as the grey area within $\mathcal I_R$ roamed by the sampled fronts. It is small enough for the $\RR$ used on the left and too important for the $\RR$ on the right. The blue reference point on $\widehat\L$ is $\RR^{*}$, the farthest point that leads to a virtual front with low enough uncertainty.}
		\label{fig:choice_R_second_phase}
	\end{figure}
	
	\begin{figure}[!ht]
		\centering
		\includegraphics[width=0.6\textwidth]{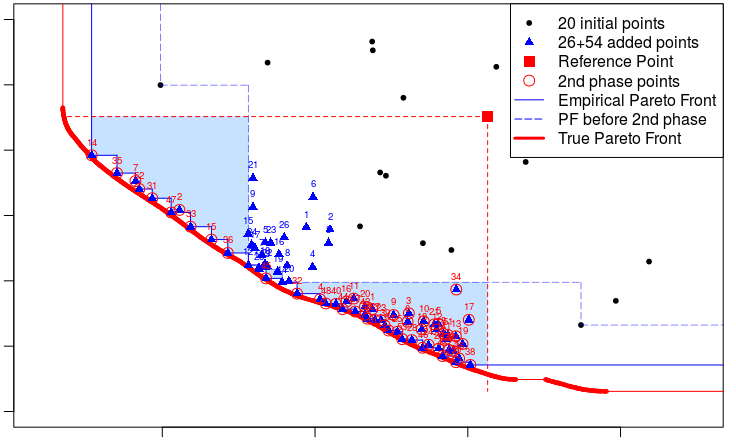}
		\caption{Final approximation of the Pareto front with, as a red square, the reference point of the second phase chosen as a solution to Problem~(\ref{eq:Rstar}), $\RR =\RR^*$. The objectives values added during the second phase of the algorithm are circled in red. Compared to the initial front obtained when searching for the center, the last approximation front is expanded as highlighted by the blue hypervolume.
		}
		\label{fig:apres_second_phase}
	\end{figure}

\paragraph{Possible improvements: }
The computational cost of this second phase of the C-EHI algorithm can be further improved.
When $m\le3$ the EHI has a closed-form expression and its update can be accelerated using the kriging variance update formulae \cite{chevalier2014corrected}. This is computationally appealing if the maximization is carried out on a fixed discrete set of designs. Another possibility for accelerating the virtual iterations is to replace the costly EHI by a cheaper and similar acquisition function such as SMS \cite{SMS}, or the Matrix-Based Expected Improvement \cite{zhan2017expected}. A last alternative is to pre-compute the Pareto set of the kriging mean functions, $\PX(\widehat{\y}(\cdot))$, using an EMOA, and to iteratively choose $\x^{*i}=\underset{\x\in\PX(\widehat{\y}(\cdot))}{\arg\max}\text{EHI}(\cdot,\RR^c)$.

	\section{Algorithm implementation and testing}
	\label{sec:results}
	\subsection{Implementation of the C-EHI algorithm}
	\label{sec:CEHI}
	The concepts and methods defined in Sections~\ref{sec:center} to \ref{sec:seconde_phase} are put together to make the C-EHI algorithm which stands for Centered Expected Hypervolume Improvement. The \texttt{R} package \texttt{DiceKriging} has been used for building the Gaussian processes and additional implementations were written in the \texttt{R} language. The C-EHI algorithm which was sketched in Figure~\ref{fig:resume_algo} is further detailed in Algorithm~\ref{algo}. The integral for $U(\widehat\L)$ is estimated numerically using $N=100$ points regularly distributed along $\widehat\L$. $U(\RR)$ is computed by means of Monte-Carlo techniques with $N=10^5$ samples.

The C-EHI algorithm can easily be extended to target non-central, user-defined, parts of the Pareto front. This extension is described in Appendix \ref{sec:noncentraltarget}.
	
	\begin{algorithm}[!ht]
		\caption{C-EHI (Centered Expected Hypervolume Improvement)}
		\label{algo}
		\begin{algorithmic}
			\STATE \textbf{Inputs:} uncertainty limit $\varepsilon$, $budget$\medskip
			\STATE create an initial DoE of $n$ points;
			\STATE initialize $m$ GPs for each objective $f_i, i=1,\dotsc,m$; \# see Section~\ref{sec:bayesian_optimization}
			\STATE $t=n$; $U(\widehat{\L})=+\infty$; \quad \# $U(\widehat{\L})$ line uncertainty, Eq.~(\ref{eq-convcenter})
			\STATE \# First phase: optimization towards the center
			\WHILE{($U(\widehat{\L})>\varepsilon$) \AND ($t\le budget$) }
			\STATE estimate $\widehat{\mathbf I}$, $\widehat{\mathbf N}$ and $\widehat{\mathbf C}$; \quad \# see Section~\ref{sec:center}
			\STATE $\x^{t+1}=\underset{\x\in X}{\arg\max}\text{ mEI}(\x; \widehat{\mathbf C})$; \quad \# see Section~\ref{sec:target}
			\STATE evaluate $\mathbf f(\x^{t+1})$ and update the GPs; \# see Section~\ref{sec:bayesian_optimization}
			\STATE compute $U(\widehat{\L})$; \quad \# see Section~\ref{sec:conv}
			\STATE $t=t+1$;
			\ENDWHILE
			\STATE \# If remaining budget after convergence: second phase
			\STATE \# Determine widest accurately attainable area and target it, see Section~\ref{sec:seconde_phase}
			\IF{$t\le budget$}
			\STATE choose $\RR^*$ solution of Eq.~(\ref{eq:Rstar}); \# see Section~\ref{sec:seconde_phase}
			\STATE $\RR^*=\underset{U(\RR;\Y^{KB})<\varepsilon}{\underset{\text{s.t. }\RR\in\widehat\L}{\arg\min}}\Vert\RR-\widehat\N\Vert$;
			\ENDIF
			\WHILE{$t\le budget$}
			\STATE $\x^{t+1}=\underset{\x\in X}{\arg\max}\text{ EHI}(\x; \RR^*)$; \# target larger improvement region $\mathcal I_{\RR^*}$
			\STATE evaluate $f_i(\x^{t+1})$ and update the GPs;
			\STATE $t=t+1$;
			\ENDWHILE
			\RETURN final DoE, final GPs, and approximation front $\widehat{\PY}$
		\end{algorithmic}
	\end{algorithm}
	
	\subsection{MetaNACA: a practical performance test bed}
	\label{sec:metaNACA}
	Comparing the efficiency of multi-objective optimizers is difficult because the performance of the algorithms depends on the test functions and a proper metric needs to be chosen to compare the Pareto fronts. The COCO platform \cite{coco-biobj-TR2016} allows the comparison of bi-objective optimizers on a general set of functions with the hypervolume improvement (calculated with respect to the Nadir point) as a performance measure.
	In the spirit of MOPTA \cite{jones2008large}, the choice was made here to test the optimizers on a set of functions that were designed to represent the real-world problems of interest.  The test set is called MetaNACA. 
For the purpose of comparison with other approaches, this set will be completed by two classical problems in Section \ref{sec:expe_analytic}.
	
The MetaNACA test bed has been built by combining surrogate modeling techniques and aerodynamic data coming from 2D simulations of the flow around a NACA airfoil (RANS with $k$-$\varepsilon$ turbulence model). 
More precisely, for each aerodynamic objective, a GP with a Matérn 5/2 kernel is first fit to an initial large space-filling DoE of 1000 designs. The evaluation of the aerodynamic performance of one design has a cost of approximately 15 minutes (wall clock time, on a standard personal computer). 
	Exploiting parallel computation, the evaluation of such a large DoE remains affordable.	Next, a sequential Bayesian multi-objective optimization infill criterion (as described in Section~\ref{sec:review_bayesian_optim}) is employed to enrich the DoE. 
	The goal of this step is to enhance the GPs in promising areas that are likely to be visited by a multi-objective optimizer. Last, 100 additional designs, drawn randomly in the design space are evaluated. 
	While these last points will help in improving the accuracy, 
	they are mainly useful in removing any artificial periodicity in the design space due to space-filling properties which might hinder the estimation of correlation parameters. 
	The evaluation, that is to say the computation of the kriging mean of the final GPs is very rapid (less than 0.1s on a personal computer), and has turned out to be an accurate substitute to the aerodynamic simulations after validation (Q2 between 0.96 and 0.99).
	The whole process of approximation building by a GP was repeated for the variable dimensions (CAD parameters) $d=3,8,22$ and $m=$ 2 to 4 objectives (lift and drag at 2 different angles of attack: 0$^\circ$ and 8$^\circ$). 
We have then computed the ``true'' Pareto front by applying the NSGA-II multi-objective optimization algorithm \cite{NSGAII} to the kriging mean functions. 
	In the following, experiments are only reported for $d=8$ variables, which compromises the dimension of the problem and the time of one optimization run, but the same conclusions have been obtained for the cases $d=3$ and $d=22$. One typical run of the C-EHI algorithm for $d=22$, $m=2$ objectives is shown in Figure \ref{fig:exemple_22d}.
	
Figure~\ref{fig:comparaison_approches} shows a typical run of the C-EHI algorithm when facing too restricted a budget to uncover the entire Pareto front. 
During the first iterations, the center of the Pareto front is targeted. Once local convergence has been detected, the part of the Pareto front in which convergence can be accurately obtained within the remaining budget is forecasted, and then targeted. The approximation of $\PY$ is enhanced in its central part. 
The same results are observed with three or four objectives and a typical run with $m=3$ is given in Figure~\ref{fig:comparaison_3obj}.
The targeting methodology gains in importance as the number of objectives increases because the relative number of Pareto optimal solutions grows and it becomes harder to approximate all of them.
\clearpage\thispagestyle{empty}
\begin{figure}[!ht]
	\centering
	\includegraphics[width=\textwidth]{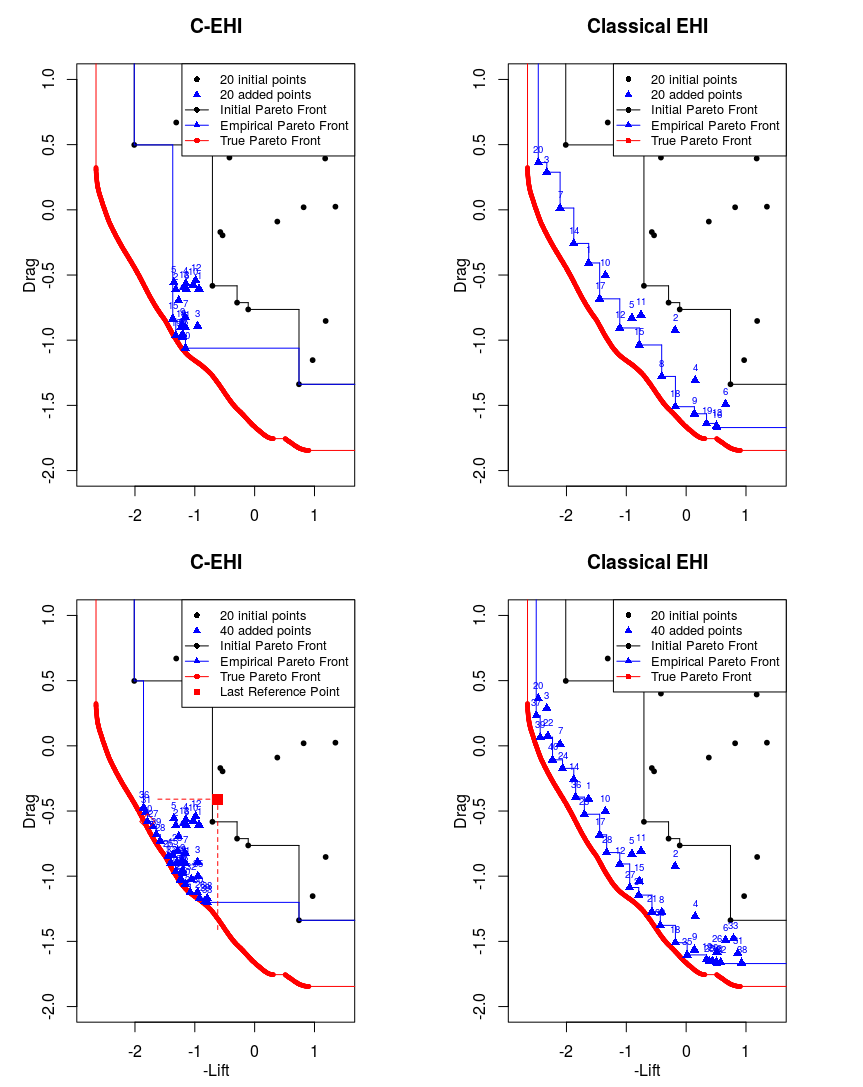}
	\caption{Comparison of C-EHI (left) with the standard EHI (right). Top: approximation front after 20 iterations: C-EHI better converges to the center of the Pareto front to the detriment of the front ends. Bottom: approximation front after 40 iterations: after local convergence (at the 22nd iteration here), a wider optimal improvement region (under the red square) is targeted for the 18 remaining iterations, is targeted by the algorithm. Compared to the standard EHI, the Pareto front is sought in a smaller balanced part of the objective space, at the advantage of a better convergence.}
	\label{fig:comparaison_approches}
\end{figure}\clearpage

\begin{figure}[h!]
	\centering
	\includegraphics[width=0.5\textwidth]{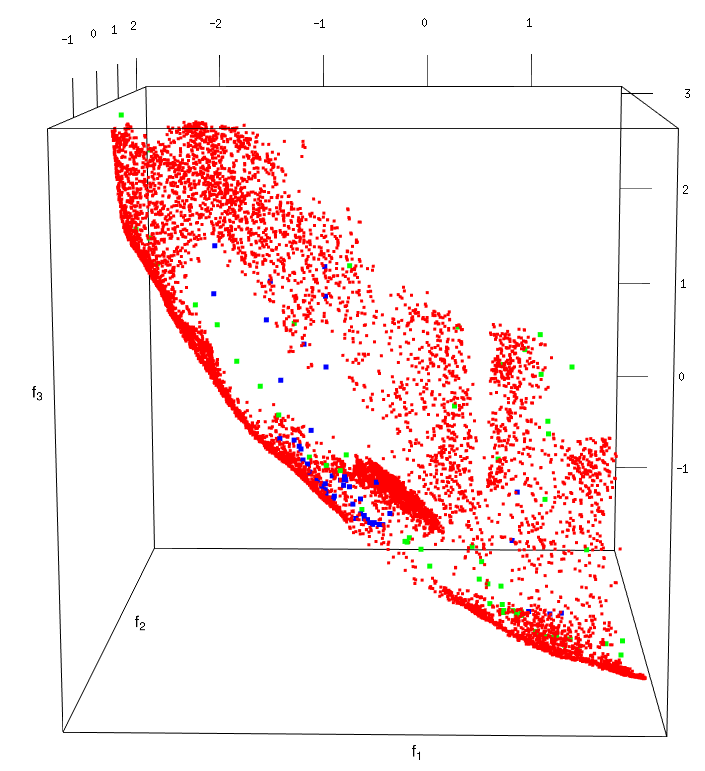}
	\caption{Typical C-EHI (blue points) and EHI (green points) runs on the MetaNACA problem with $m=3$ objectives. The true Pareto front (red) is attained at its center by C-EHI while it is approximated globally yet less accurately by EHI.}
	\label{fig:comparaison_3obj}
\end{figure}

\begin{figure}[h!]
	\centering
	\includegraphics[width=0.48\textwidth]{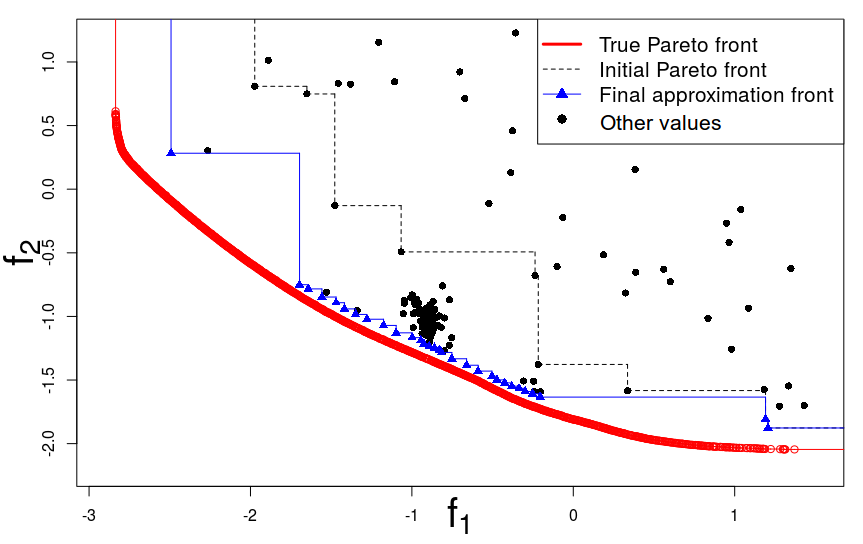}
	\includegraphics[width=0.48\textwidth]{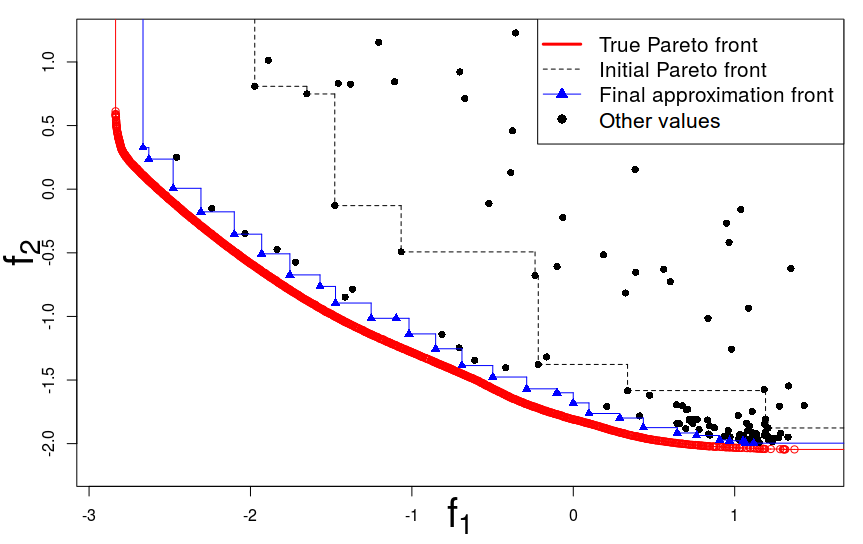}
	\caption{Comparison between C-EHI (left) and EHI (right) for one run of the MetaNACA problem in $d=22$ dimensions. 150 calls to $\mathbf f(\cdot)$ were allowed and 50 of them were devoted to the initial DoE. Again, C-EHI improves the Pareto front at its center, 
		EHI tries to uncover the whole front at the cost of a lower accuracy.
	}
	\label{fig:exemple_22d}
\end{figure}
	
	\subsection{Performance metrics}
	
For comparing approximation fronts produced by multi-objective algorithms, considering several indicators is recommended \cite{knowles2002metrics,zitzler2003performance}.
	In the following, we use three common metrics: the non-dominated 
hypervolume \cite{TheseZitzler} that we normalize with respect to the hypervolume of the true Pareto front, and the Inverse Generational Distance (IGD) 
\cite{IGD}, which corresponds to the mean distance between points of a reference set (in our case the true Pareto front) and the approximation front.
	A modified version of the $\varepsilon$-Indicator \cite{zitzler2003performance} is also used for measuring the minimal distance to the Pareto front of an approximation front: $\varepsilon(\widehat{\PY};\PY):=\underset{\y\in\widehat\PY}{\min}{\min}~\{\varepsilon:\nexists\mathbf z\in\PY,\mathbf z\preceq\y-\varepsilon\cdot\mathbf1_m\}$. 
It corresponds to the smallest value that has to be subtracted to $\widehat\PY$ such that one of its solutions becomes non-dominated with regard to $\PY$. To simplify, we will still refer to the $\varepsilon$-Indicator when considering this indicator.
	These metrics deal with approximations of the \emph{whole} Pareto front, and empirical Pareto fronts having a similar shape to the one shown in blue in Fig.~\ref{fig:zone_refpoints} will be measured as performing poorly as they do not cover the entire front. 
	
	In order to focus on the central part of the Pareto front, the indicators are restricted to the regions of interest
	\begin{equation*}
	{\mathcal I}_w := \{\y\in Y:\y\preceq \RR^w\} \quad\text{ where }\quad \RR^w := (1-w)\mathbf C+w\N~.
	\end{equation*}
	To focus on the central part, $w$'s ranging between 0.05 and 0.3 will be used.

Another performance metric, the attainment time, will allow to measure the convergence speed. 
The attainment time of $\RR^w$ which is the number of functions evaluations (including the initial DoE)
required by an algorithm to dominate $\RR^w$\footnote{If one run does not attain $\RR^w$, we compute a rough estimator of the Expected Runtime \cite{auger2005performance}, $\overline{T_s}/p_s$, where $\overline{T_s}$ and $p_s$ correspond to the runtime of successful runs and the proportion of successful runs, respectively.}.
	
	\subsection{Test results}
\label{sec-test_results}
	
	\subsubsection{Experiments with analytical test functions}
\label{sec:expe_analytic}
In this section, we investigate how C-EHI converges to the center of the Pareto front and compare it with two state-of-the-art algorithms: a Bayesian optimizer with the EHI infill criterion \cite{EHI} and the Evolutionary Algorithm NSGA-II \cite{NSGAII}. As discussed in Section \ref{sec:multiobj}, EHI is defined up to a reference point which is instrumental in selecting the part of the objective space $\mathcal I_\RR$ where $\PY$ is sought. 
To target the entire $\PY$ with EHI, $\RR$ should be placed at the Nadir point of the true Pareto front. 
Since $\PY$ is unknown, it is suggested \cite{ishibuchi2018specify,TheseFeliot} to take a conservative empirical Nadir point, $r\widehat\N+(1-r)\widehat\I$ with $r=1.1$
, where $\widehat\I$ and $\widehat\N$ stand here for the empirical Ideal and Nadir points.
	
This EHI implementation depends on $\widehat\PY$ through $\widehat\I$ and $\widehat\N$. 
We therefore consider three additional EHI variants.
In the idealized EHI$_{\PY}$, the reference point is $\RR:=\N$, the true Nadir point. 
In this variant, $\mathcal I_\RR=\mathcal I_{\PY}$: the considered improvement area is the right one. EHI$_{\PY}$ corresponds to an utopian setting where it would be known in advance where to look for the Pareto front in the objective space. 
Its interest is that it provides an upper bound on the expected performance of EHI.
	
The third variant, EHI$_\text{N}$, has $\RR$ defined as the estimated Nadir point of the Pareto front, $\widehat\N$ using the techniques of Section \ref{sec:estimation}. 
EHI$_\text{N}$ is a new version of the EHI algorithm: instead of defining $\RR$ relying on observed data such as the empirical front or extreme observations, $\RR$ is set up according to the metamodels.
	
Last, we consider the EHI$_\text{M}$ variant in which the reference point is $\RR:=\mathbf M$ where $\mathbf M$ stands for the maximal value observed, $M_j=\underset{i=1,\dotsc,t}{\max}f_j(\x^i)$, $j=1,\dotsc,m$. 
Contrarily to EHI$_\text{N}$, the maximum is taken over all the points instead of over those in $\widehat{\PY}$.
Such a reference point will often have large components. If it covers all of the objective space, it may over-emphasize the extreme parts of the Pareto front.

The algorithms are benchmarked with two popular analytical test functions for multi-objective optimization. 
The first one is the P1 problem of \cite{TheseParr}, which has $d=2$ dimensions and $m=2$ objectives. It is initialized with a design of experiments of size $n=8$ and run for 12 iterations. The second test problem is ZDT1 \cite{zdt2000a} in $d=4$ dimensions and $m=2$ objectives, initialized with a design of experiments of size $n=20$ and run for 40 additional iterations.
	
Two comparison metrics are considered. 
The first one is the hypervolume indicator restricted to $\mathcal I_w$ for $w=0.05,0.15,0.25$ to evaluate convergence and diversity in the central parts of the Pareto front. 
Figure \ref{fig:central_fronts} shows these improvement regions for both benchmark problems. 
The second performance metric is the attainment time which assesses the time it takes to a method for entering the improvement region irrespectively of the final hypervolume covered.
	
	\begin{figure}[h!]
		\centering
		\begin{subfigure}[b]{0.6\textwidth}\centering
		\includegraphics[width=\textwidth]{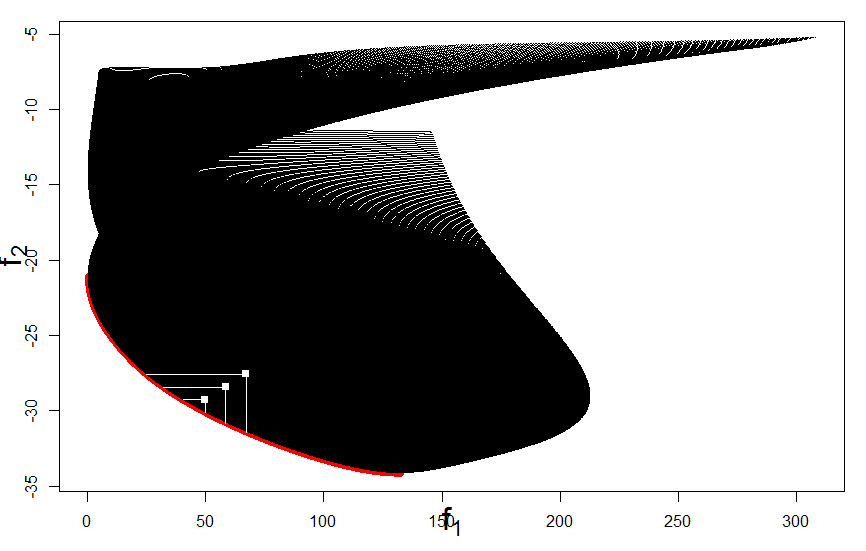}\caption{P1 objective space}
	\end{subfigure}\\
	\begin{subfigure}[b]{0.46\textwidth}
		\includegraphics[width=\textwidth]{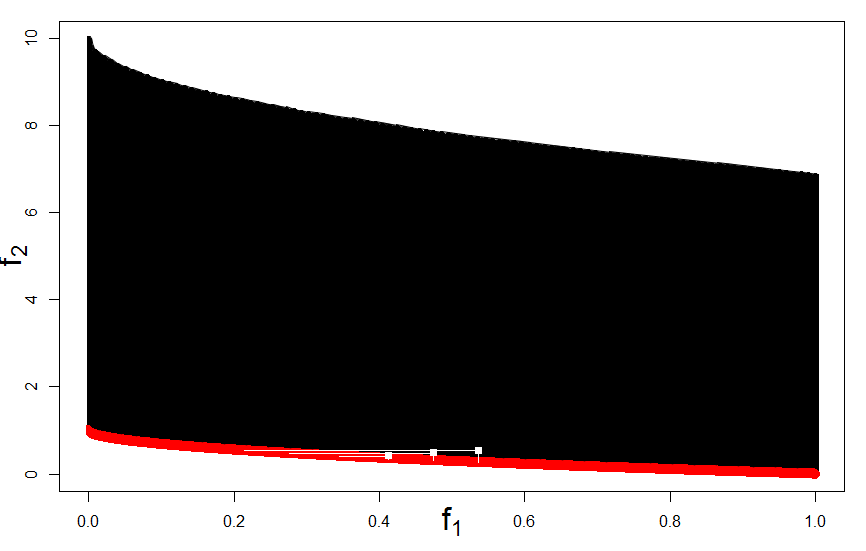}\caption{ZDT1 $(d=4)$ objective space}
	\end{subfigure}
		\begin{subfigure}[b]{0.46\textwidth}
		\includegraphics[width=\textwidth]{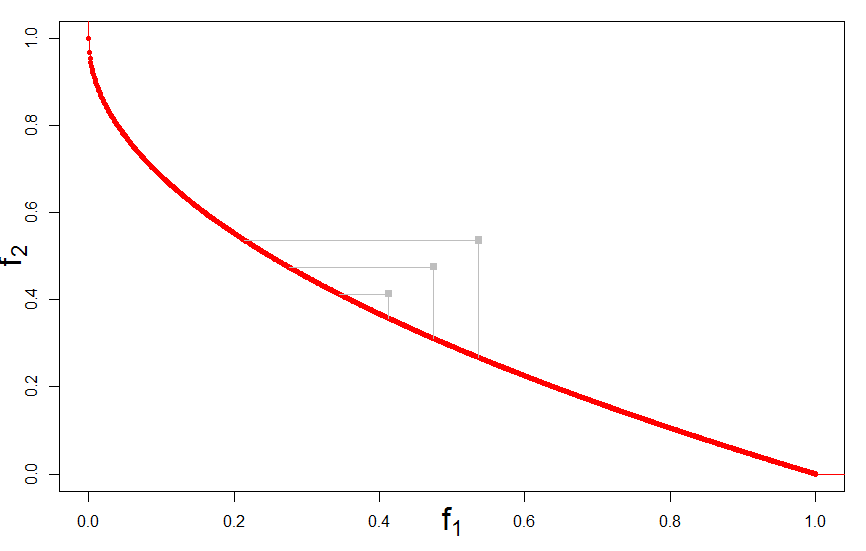}\caption{Zoom on the ZDT1 $(d=4)$ Pareto front}
	\end{subfigure}\\
		\caption{Pareto fronts (red) and objective spaces (black) of the P1 problem (top) and of the ZDT1, $d=4$, problem (bottom, zoom on $\PY$ on the right) with the $\mathcal I_w$ areas to which the performance metrics are restricted. These correspond to a central part of the Pareto front.
}
		\label{fig:central_fronts}
	\end{figure}
	
Runs are repeated 10 times starting from different initial space-filling designs. 
The metrics means and standard deviations are reported in Tables \ref{tab:results_p1} and \ref{tab:results_zdt1}. 
They are computed for C-EHI, the four EHI variants, and NSGA-II. 
The population size of NSGA-II is set to 12 and 20 for P1 and ZDT1, respectively.
The performance of NSGA-II is recorded at the smallest number of generations such that the number of functions evaluations is larger or equal to that of the Bayesian algorithms. This number of generations is 2 and 3 for P1 and ZDT1 and the metrics are on the NSGA-II$_b$ row in Tables \ref{tab:results_p1} and \ref{tab:results_zdt1}.
For comparison purposes, NSGA-II runs are continued until 120 and 800 functions evaluations are reached for the P1 and ZDT1 functions. The final metrics are given in both Tables on the NSGA-II$_+$ row.

\begin{table}[!ht]
		\centering
		\makebox[\textwidth][c]{
			\begin{tabu}{|c|c|c|c|c|c|c|}
				\hline
				 & \multicolumn{3}{c|}{Hypervolume} & \multicolumn{3}{c|}{Attainment time}\\\hline
				$w$ & 0.05 & 0.15 & 0.25 & 0.05 & 0.15 & 0.25\\\hline
				C-EHI & 0.185 \scriptsize(0.233) & 0.549 \scriptsize(0.263) & 0.668 \scriptsize(0.185) & {\color{red}21.6 \scriptsize[7]} & 13.1 \scriptsize(2.7) & 9.5 \scriptsize(1)\\
				EHI & 0.155 \scriptsize(0.218) & 0.465 \scriptsize(0.179) & 0.611 \scriptsize(0.114) & {\color{red}39.4 \scriptsize[4]} & 13.2 \scriptsize(2.6) & 11.4 \scriptsize(2.6)\\
				EHI$_{\PY}$ & 0.269 \scriptsize(0.260) & 0.446 \scriptsize(0.175) & 0.636 \scriptsize(0.136) & {\color{red}30.0 \scriptsize[6]} & 14 \scriptsize(3.2) & 11 \scriptsize(2.6)\\
				EHI$_\text{N}$ & 0.130 \scriptsize(0.158) & 0.312 \scriptsize(0.223) & 0.460 \scriptsize(0.192) & {\color{red}32.4 \scriptsize[5]} & {\color{red}16.7 \scriptsize[9]} & 11.5 \scriptsize(3.5)\\
				EHI$_\text{M}$ & 0.012 \scriptsize(0.039) & 0.202 \scriptsize(0.181) & 0.389 \scriptsize(0.136) & {\color{red}180 \scriptsize[1]} & {\color{red}22.7 \scriptsize[7]} & 12.6 \scriptsize(4.1)\\
				NSGA-II$_b$ & 0 & 0.052 \scriptsize(0.110) & 0.107 \scriptsize(0.183) & {\color{red}- \scriptsize[0]} & {\color{red}80 \scriptsize[2]} & {\color{red}51.1 \scriptsize[3]}\\
				NSGA-II$_+$ & 0.188 \scriptsize(0.219) & 0.576 \scriptsize(0.109) & 0.705 \scriptsize(0.069) & {\color{red}169.6 \scriptsize[5]} & 50.4 \scriptsize(31.1) & 41.3 \scriptsize(31.9)\\\hline
			\end{tabu}
		}
		\caption{Hypervolume and attainment time averaged over 10 runs (standard deviation in brackets), for different central parts of the Pareto front on the P1 problem. When at least one run did not attain $\RR^w$, red figures correspond to empirical runtimes with the number of successful runs in brackets. '-' indicates that no run was able to attain $\RR^w$ in the given budget.}
		\label{tab:results_p1}
	\end{table}

	\begin{table}[!ht]
	\centering
	\makebox[\textwidth][c]{
		\begin{tabu}{|c|c|c|c|c|c|c|}
			\hline
			& \multicolumn{3}{c|}{Hypervolume} & \multicolumn{3}{c|}{Attainment time}\\\hline
			$w$ & 0.05 & 0.15 & 0.25 & 0.05 & 0.15 & 0.25\\\hline
			C-EHI & 0.703 \scriptsize(0.049) & 0.895 \scriptsize(0.010) & 0.936 \scriptsize(0.006) & 26.8 \scriptsize(6.6) & 23.4 \scriptsize(2.2) & 23.4 \scriptsize(2.2)\\
			EHI & 0.065 \scriptsize(0.154) & 0.097 \scriptsize(0.204) & 0.101 \scriptsize(0.213) & {\color{red}145 \scriptsize[2]} & {\color{red}145 \scriptsize[2]} & {\color{red}145 \scriptsize[2]}\\
			EHI$_{\PY}$ & 0.611 \scriptsize(0.066) & 0.848 \scriptsize(0.029) & 0.901 \scriptsize(0.023) & 28.7 \scriptsize(2.8) & 22.8 \scriptsize(2.3) & 21.4 \scriptsize(0.5)\\
			EHI$_\text{N}$ & 0.362 \scriptsize(0.349) & 0.650 \scriptsize(0.246) & 0.740 \scriptsize(0.206) & {\color{red}48.1 \scriptsize[6]} & 22.2 \scriptsize(0.4) & 22.2 \scriptsize(0.4)\\
			EHI$_\text{M}$ & 0.575 \scriptsize(0.107) & 0.845 \scriptsize(0.038) & 0.906 \scriptsize(0.022) & 24.4 \scriptsize(5.6) & 22.2 \scriptsize(0.6) & 22.1 \scriptsize(0.3)\\
			NSGA-II$_b$ & 0 & 0 & 0 & {\color{red}- \scriptsize[0]} & {\color{red}- \scriptsize[0]} & {\color{red}- \scriptsize[0]}\\
			NSGA-II$_+$ & 0.375 \scriptsize(0.161) & 0.749 \scriptsize(0.075) & 0.842 \scriptsize(0.052) & 532.9 \scriptsize(143.4) & 331.9 \scriptsize(121) & 219.2 \scriptsize(101.5)\\\hline
		\end{tabu}
	}
	\caption{Hypervolume and attainment time averaged over 10 runs (standard deviation in brackets), for different central parts of the Pareto front on the ZDT1 problem. When at least one run did not attain $\RR^w$, red figures correspond to empirical runtimes with the number of successful runs in brackets. '-' indicates that no run was able to attain $\RR^w$ in the given budget.
}
	\label{tab:results_zdt1}
\end{table}

Before analyzing the results in more details, let us state the main conclusions of Tables \ref{tab:results_p1} and \ref{tab:results_zdt1}. 
On both test problems, C-EHI consistently outperforms all other EHI variants in terms of hypervolume and time to reach the central parts of $\PY$. 
The performances of the different optimizers depend on the test function and further explanations are given in the following. 
At the considered limited budget, the evolutionary algorithm NSGA-II gives a weaker approximation of the Pareto front central regions than the Bayesian methods, as measured by both the hypervolumes and the attainment times.

\subsubsection*{P1 problem}
The statistics of the hypervolumes reported in Table \ref{tab:results_p1} indicate that C-EHI better converges to the central part of the Pareto front than the other EHI algorithms.
The helped EHI$_{\PY}$ outperforms C-EHI only when $w=0.05$. 
This is due to the fact that this benchmark contains a local Pareto front (which can be seen on Figure \ref{fig:central_fronts} for small $f_1$ values and $f_2\approx-17$), which lightly deteriorates the Ideal and the Nadir point estimation, hence the estimation of the Center. 
The error in $\widehat{\mathbf C}$ leads to a slightly off-centered convergence which is highlighted by the fact that 3 C-EHI runs out of 10 did not attain this narrow part of $\PY$. 
Some difficulties in estimating $\N$ through GPs simulations are visible in the moderate performance of EHI$_\text{N}$ relatively to the standard EHI approach (where $\RR$ is defined according to the empirical front). 
Yet, as stated in Proposition \ref{prop:insensibilite}, the error in Nadir estimation barely affects C-EHI, but impacts EHI$_\text{N}$ more significantly.
Regarding EHI variants, EHI$_\text{M}$ performs poorly when compared to the standard EHI and EHI$_{\PY}$ because of the distant reference point which targets an unnecessarily large
part of the objective space. 
At the same number of function evaluations (20), C-EHI clearly outperforms NSGA-II which needs approximately 6 times more function evaluations to achieve the same performance. 
		
The attainment times recorded in Table \ref{tab:results_p1} for the P1 problem confirm that the center-targeting C-EHI reaches the central regions faster than the other methods. 
The thinnest area of interest ($w=0.05$) is attained more consistently (reached 7 times out of 10 against 6 times by EHI$_{\PY}$, 5 times by EHI$_\text{N}$, 4 times by EHI and 1 time by EHI$_\text{M}$). 
Because of its distant $\RR$, EHI$_\text{M}$ is the Bayesian method which needs the most function evaluations to find $\mathcal I_w$. 
The evolutionary NSGA-II is not able to attain $\mathcal I_{0.05}$ within 24 function evaluations, only 2 runs out of 10 attain $\mathcal I_{0.15}$ and 3 out of 10 attain $\mathcal I_{0.25}$. 
	
\subsubsection*{ZDT1 problem}
As shown at the bottom of Figure \ref{fig:central_fronts}, 
the ZDT1 problem has a wide $f_2$ range. In dimension $d=4$, it is difficult to find $f_2$ values in $\PY$'s range: only 0.8\% of $X$ leads to $f_2\le1$. On the contrary, all $f_1$ values are in $\PY$'s range.
Therefore, the definition of the part of the objective space where to seek $\PY$ through $\RR$ is critical.

C-EHI correctly identifies the center of $\PY$ and drives the optimization towards it, as evidenced by the larger hypervolumes of C-EHI in Table \ref{tab:results_zdt1} for all $w$'s. 
C-EHI has the best but one attainment time of $\mathcal I_{0.05}$ with 26.8 evaluations on the average. EHI$_\text{M}$ solely attains $\mathcal I_{0.05}$ in fewer function evaluations. 
It is worth mentioning that only $5\times10^{-6}$\% of the design space has an image in $\mathcal I_{0.05}$, 
highlighting the performance of C-EHI (and EHI$_\text{M}$ for the occasion).
The number of function evaluations to reach $\mathcal I_{0.15}$ and $\mathcal I_{0.25}$ is slightly larger for C-EHI than for the other EHI's. 
This is due to the fact that the first mEI iterations of the C-EHI algorithm sometimes target parts of $\PY$ that are not exactly at the center, because of ZDT1's objective space shape. Nonetheless, C-EHI corrects this initial inaccuracy and, at the end of the second phase, a better convergence is achieved as confirmed by the hypervolume.
Even though it is equipped with the correct $\RR$, EHI$_{\PY}$ does not exhibit results as good as C-EHI, except the attainment time of the wider central parts ($\mathcal I_{0.15}$ and $\mathcal I_{0.25}$). 

The EHI in which $\RR$ is computed through the empirical Ideal and Nadir points performs poorly.
Only two runs touch the central parts of $\PY$. Because the Pareto front of ZDT1 has a small $f_2$ range and a large $f_1$ range, the initial errors in $\widehat\RR$ cut large $f_1$ values out of the improvement region. Graphically, the search seems directed towards the left-hand-side of the Pareto front.
EHI$_\text{N}$ is outperformed by C-EHI and EHI$_{\PY}$, but achieves a much better convergence than EHI. 
This shows the benefits of estimating the location of the Nadir point through GP simulations instead of picking the empirical Nadir for $\RR$ in problems such as ZDT1, if the whole Pareto front is sought. 
Even though EHI$_\text{M}$ does not work well on general functions because of a too large targeted part in the objective space $\mathcal I_{\RR}$, it yields good results here both in terms of hypervolume and attainment time. 
Indeed, EHI$_\text{M}$ avoids the pitfalls of ZDT1 that were just mentioned, i.e., it does not remove large $f_1$ values from the improvement region.
At the same number of function evaluations (60, row NSGA-II$_b$), NSGA-II is never able to find any $\mathcal I_w$. 
Even when 800 designs (row NSGA-II$_+$) are evaluated, the hypervolume in these central areas is much smaller than that of C-EHI.

\subsubsection{Experiments on the MetaNACA test bed}
The Tables \ref{tab:hypervolume} to \ref{tab:epsilon} below contain the hypervolume indicator, the IGD, and the modified $\varepsilon$-Indicator for the 2, 3 and 4 objective MetaNACA test cases. 
They are computed in $\mathcal I_{0.1}$, $\mathcal I_{0.2}$ and $\mathcal I_{0.3}$, and averaged over 10 runs. Standard deviations are indicated in parentheses.
	The last column averages the indicator values restricted to $\mathcal I_{\RR^*}$ (the optimal reference point of Equation (\ref{eq:Rstar})) over the runs that reached the second phase. A - indicates that no run has reached the second phase for the considered budget. 
Similarly to the attainment times in the previous Section, red figures correspond to extrapolated indicators: when for at least one run, no solution was found in $\mathcal I_{w}$, the indicator is averaged over the runs which entered $\mathcal I_{w}$ and divided by the proportion of successful runs. Brackets indicate the number of successful runs.
	The indicator values of the C-EHI algorithm are compared to those obtained with the standard EHI \cite{EHI3} implementation of the \texttt{R} package \texttt{GPareto} \cite{GPareto} (right column). 
	In \texttt{GPareto}, the default reference point is taken at $\N +\mathbf{1}$.
Dealing with parsimonious calls to the objective functions, four tight optimization budgets are considered: 40, 60, 80 and 100 calls to $\mathbf f$. The 20 first calls are devoted to the initialization of the GPs using an LHS space-filling design \cite{stein1987large}, and the experiments are repeated 10 times starting from different initial designs.
	
	Figure \ref{fig:evolution_hypervolume} shows how the hypervolume indicator evolves with optimization iterations. The indicators are of course increasing with the iterations, and the C-EHI consistently outperforms the general EHI in finding points in the central part of the Pareto front for 2 and 3 objectives. For 4 objectives an important number of points obtained by both algorithms belongs to $\mathcal{I}_{0.2}$ and $\mathcal{I}_{0.3}$. 
While significantly more values (and Pareto-optimal values) are obtained by C-EHI in $\mathcal{I}_{0.2}$ and $\mathcal{I}_{0.3}$, EHI may episodically and non-significantly yield a larger hypervolume.
	
	A few words of caution are needed to read the Tables \ref{tab:hypervolume} to \ref{tab:epsilon}. 
As the width of the Pareto front that is targeted in the second phase depends on the remaining budget, runs of the C-EHI algorithm with different total budgets are not directly comparable. For instance, if convergence is detected after 35 iterations, the reference point that defines the targeted area for the last calculations $\RR^*$ will be different if 5 or 45 iterations remain. 
	The first case will concentrate on a very central part of the Pareto front, whereas the second will target a broader area. As a consequence, some numbers may express better performance in thinner portions of the Pareto front in spite of a smaller total budget, which is only due to the fact that they have explicitly targeted a smaller part of the solutions. 
	
	\begin{table}[!ht]
		\centering
		\setlength\tabcolsep{3pt}
		\makebox[\textwidth][c]{
			\begin{tabu}{|c|c|c|c|c|c|c|c|c|c|c|c|}
				\hline
				$m$ & $budget$ & \multicolumn{2}{c|}{$\RR^{0.1}$} & \multicolumn{2}{c|}{$\RR^{0.2}$} & \multicolumn{2}{c|}{$\RR^{0.3}$} & \multicolumn{2}{c|}{$\RR^*$}\\
				\rowfont{\scriptsize}
				& & C-EHI & EHI & C-EHI & EHI & C-EHI & EHI & C-EHI & EHI\\\hline
				\rowfont{\normalsize}
				& 40 & 0.275 \scriptsize(0.18) & 0.025 \scriptsize(0.04) & 0.498 \scriptsize(0.17) & 0.227 \scriptsize(0.15) & 0.581 \scriptsize(0.10) & 0.386 \scriptsize(0.19) & 0.664 & 0.253\\
				2 & 60 & 0.377 \scriptsize(0.19) & 0.096 \scriptsize(0.12) & 0.651 \scriptsize(0.11) & 0.342 \scriptsize(0.14) & 0.719 \scriptsize(0.09) & 0.525 \scriptsize(0.12) & 0.768 \scriptsize(0.13) & 0.418 \scriptsize(0.24)\\
				& 80 & 0.548 \scriptsize(0.10) & 0.118 \scriptsize(0.11) & 0.759 \scriptsize(0.05) & 0.398 \scriptsize(0.12) & 0.821 \scriptsize(0.03) & 0.572 \scriptsize(0.11) & 0.881 \scriptsize(0.04) & 0.606 \scriptsize(0.22)\\
				& 100 & 0.524 \scriptsize(0.14) & 0.153 \scriptsize(0.16) & 0.744 \scriptsize(0.08) & 0.503 \scriptsize(0.13) & 0.831 \scriptsize(0.05) & 0.658 \scriptsize(0.08) & 0.919 \scriptsize(0.02) & 0.805 \scriptsize(0.08)\\\hline
				& 40 & 0.013 \scriptsize(0.02) & 0 \scriptsize(0) & 0.181 \scriptsize(0.09) & 0.086 \scriptsize(0.05) & 0.319 \scriptsize(0.05) & 0.237 \scriptsize(0.07) &  - & -\\			
				3 & 60 & 0.058 \scriptsize(0.06) & 0.010 \scriptsize(0.02) & 0.267 \scriptsize(0.08) & 0.136 \scriptsize(0.06) & 0.394 \scriptsize(0.05) & 0.305 \scriptsize(0.04) & 0.286 \scriptsize(0.03) & 0.021 \scriptsize(0.03)\\		
				& 80 & 0.109 \scriptsize(0.08) & 0.012 \scriptsize(0.02) & 0.327 \scriptsize(0.14) & 0.170 \scriptsize(0.10) & 0.447 \scriptsize(0.17) & 0.321 \scriptsize(0.13) & 0.476 \scriptsize(0.08) & 0.161 \scriptsize(0.11)\\		
				& 100 & 0.160 \scriptsize(0.09) & 0.016 \scriptsize(0.02) & 0.412 \scriptsize(0.07) & 0.218 \scriptsize(0.06) & 0.546 \scriptsize(0.04) & 0.391 \scriptsize(0.06) & 0.584 \scriptsize(0.05) & 0.224 \scriptsize(0.09)\\\hline
				& 40 & 0.113 \scriptsize(0.11) & 0.075 \scriptsize(0.10) & 0.291 \scriptsize(0.09) & 0.240 \scriptsize(0.10) & 0.374 \scriptsize(0.06) & 0.378 \scriptsize(0.09) & - & -\\			
				4 & 60 & 0.187 \scriptsize(0.15) & 0.138 \scriptsize(0.09) & 0.356 \scriptsize(0.08) & 0.340 \scriptsize(0.09) & 0.418 \scriptsize(0.05) & 0.473 \scriptsize(0.07) & 0.533 & 0.238\\		
				& 80 & 0.312 \scriptsize(0.16) & 0.198 \scriptsize(0.08) & 0.470 \scriptsize(0.09) & 0.413 \scriptsize(0.07) & 0.516 \scriptsize(0.09) & 0.533 \scriptsize(0.06) & 0.617 \scriptsize(0.08) & 0.338 \scriptsize(0.07)\\		
				& 100 & 0.519 \scriptsize(0.08) & 0.219 \scriptsize(0.07) & 0.612 \scriptsize(0.11) & 0.464 \scriptsize(0.07) & 0.642 \scriptsize(0.12) & 0.580 \scriptsize(0.06) & 0.729 \scriptsize(0.05) & 0.453 \scriptsize(0.04)\\\hline
			\end{tabu}
		}
		\caption{Hypervolume indicator averaged over 10 runs for different central parts of the Pareto front, budgets and number of objectives. The true Pareto front has an hypervolume indicator of 1.
\label{tab:hypervolume}
}
	\end{table}
	
	\begin{table}[!ht]
		\centering
		\setlength\tabcolsep{3pt}
		\makebox[\textwidth][c]{
			\begin{tabu}{|c|c|c|c|c|c|c|c|c|c|c|c|}
				\hline
				$m$ & $budget$ & \multicolumn{2}{c|}{$\RR^{0.1}$} & \multicolumn{2}{c|}{$\RR^{0.2}$} & \multicolumn{2}{c|}{$\RR^{0.3}$} & \multicolumn{2}{c|}{$\RR^*$}\\
				\rowfont{\scriptsize}
				& & C-EHI & EHI & C-EHI & EHI & C-EHI & EHI & C-EHI & EHI\\\hline
				\rowfont{\normalsize}		
				& 40 & {\color{red}0.130 \scriptsize[9]} & {\color{red}0.391 \scriptsize[5]} & 0.176 \scriptsize(0.09) & {\color{red}0.246 \scriptsize[9]} & 0.228 \scriptsize(0.05) & 0.293 \scriptsize(0.20) & 0.069 & 0.175\\		
				2 & 60 & 0.095 \scriptsize(0.05) & {\color{red}0.242 \scriptsize[7]} & 0.109 \scriptsize(0.05) & 0.204 \scriptsize(0.08) & 0.133 \scriptsize(0.06) & 0.184 \scriptsize(0.06) & 0.066 \scriptsize(0.02) & {\color{red}0.101 \scriptsize[9]}\\
				& 80 & 0.059 \scriptsize(0.02) & {\color{red}0.203 \scriptsize[8]} & 0.058 \scriptsize(0.01) & 0.171 \scriptsize(0.05) & 0.067 \scriptsize(0.02) & 0.161 \scriptsize(0.07) & 0.050 \scriptsize(0.01) & 0.149 \scriptsize(0.05)\\
				& 100 & 0.067 \scriptsize(0.02) & {\color{red}0.177 \scriptsize[8]} & 0.059 \scriptsize(0.02) & 0.138 \scriptsize(0.05) & 0.055 \scriptsize(0.02) & 0.118 \scriptsize(0.03) & 0.048 \scriptsize(0.02) & 0.109 \scriptsize(0.03)\\\hline
				& 40 & {\color{red}0.736 \scriptsize[5]} & {\color{red}4.267 \scriptsize[1]} & 0.455 \scriptsize(0.13) & 0.518 \scriptsize(0.13) & 0.531 \scriptsize(0.12) & 0.500 \scriptsize(0.10) & - & -\\			
				3 & 60 & {\color{red}0.390 \scriptsize[8]} & {\color{red}0.961 \scriptsize[4]} & 0.388 \scriptsize(0.11) & 0.460 \scriptsize(0.11) & 0.471 \scriptsize(0.13) & 0.439 \scriptsize(0.06) & 0.196 \scriptsize(0.03) & {\color{red}0.287 \scriptsize[8]}\\
				& 80 & 0.238 \scriptsize(0.10) & {\color{red}0.550 \scriptsize[5]} & 0.256 \scriptsize(0.12) & 0.361 \scriptsize(0.17) & 0.339 \scriptsize(0.14) & 0.356 \scriptsize(0.14) & 0.181 \scriptsize(0.05) & {\color{red}0.241 \scriptsize[9]}\\
				& 100 & 0.226 \scriptsize(0.05) & {\color{red}0.510 \scriptsize[6]} & 0.250 \scriptsize(0.05) & 0.349 \scriptsize(0.06) & 0.335 \scriptsize(0.08) & 0.351 \scriptsize(0.07) & 0.183 \scriptsize(0.05) & 0.349 \scriptsize(0.08)\\\hline
				& 40 & {\color{red}0.345 \scriptsize[9]} & {\color{red}0.624 \scriptsize[6]} & 0.381 \scriptsize(0.05) & 0.447 \scriptsize(0.12) & 0.626 \scriptsize(0.07) & 0.571 \scriptsize(0.07) & - & -\\			
				4 & 60 & 0.280 \scriptsize(0.13) & {\color{red}0.374 \scriptsize[8]} & 0.334 \scriptsize(0.04) & 0.359 \scriptsize(0.06) & 0.587 \scriptsize(0.07) & 0.512 \scriptsize(0.07) & 0.197 & 0.233\\
				& 80 & 0.210 \scriptsize(0.06) & 0.282 \scriptsize(0.06) & 0.285 \scriptsize(0.05) & 0.298 \scriptsize(0.04) & 0.523 \scriptsize(0.08) & 0.460 \scriptsize(0.06) & 0.212 \scriptsize(0.04) & 0.262 \scriptsize(0.08)\\
				& 100 & 0.158 \scriptsize(0.02) & 0.266 \scriptsize(0.06) & 0.236 \scriptsize(0.05) & 0.277 \scriptsize(0.03) & 0.468 \scriptsize(0.08) & 0.430 \scriptsize(0.05) & 0.257 \scriptsize(0.04) & 0.291 \scriptsize(0.08)\\\hline
			\end{tabu}
		}
		\caption{Inverted Generational Distance averaged over 10 runs for different central parts of the Pareto front, budgets and number of objectives. Lower values are better.}
		\label{tab:igd}
	\end{table}
	
	\begin{table}[!ht]
		\centering
		\setlength\tabcolsep{3pt}
		\makebox[\textwidth][c]{
			\begin{tabu}{|c|c|c|c|c|c|c|c|c|c|c|c|c|c|}
				\hline
				$m$ & $budget$ & \multicolumn{2}{c|}{$\RR^{0.1}$} & \multicolumn{2}{c|}{$\RR^{0.2}$} & \multicolumn{2}{c|}{$\RR^{0.3}$} & \multicolumn{2}{c|}{Whole front} & \multicolumn{2}{c|}{$\RR^*$}\\
				\rowfont{\scriptsize}
				& & C-EHI & EHI & C-EHI & EHI & C-EHI & EHI & C-EHI & EHI & C-EHI & EHI\\\hline
				\rowfont{\normalsize}
				& 40 & {\color{red}0.048 \scriptsize[9]} & {\color{red}0.189 \scriptsize[5]} & 0.033 \scriptsize(0.02) & {\color{red}0.112 \scriptsize[8]} & 0.033 \scriptsize(0.02) & 0.121 \scriptsize(0.12) & 0.033 \scriptsize(0.02) & 0.076 \scriptsize(0.08) & 0.014 & 0.078\\
				2 & 60 & 0.024 \scriptsize(0.02) & {\color{red}0.121 \scriptsize[7]} & 0.014 \scriptsize(0.01) & 0.081 \scriptsize(0.04) & 0.014 \scriptsize(0.01) & 0.061 \scriptsize(0.03) & 0.012 \scriptsize(0.01) & 0.042 \scriptsize(0.03) & 0.009 \scriptsize(0.01) & {\color{red}0.070 \scriptsize[9]}\\
				& 80 & 0.010 \scriptsize(0.01) & {\color{red}0.099 \scriptsize[8]} & 0.008 \scriptsize(0) & 0.062 \scriptsize(0.02) & 0.007 \scriptsize(0) & 0.052 \scriptsize(0.03) & 0.006 \scriptsize(0) & 0.032 \scriptsize(0.02) & 0.003 \scriptsize(0) & 0.044 \scriptsize(0.02)\\
				& 100 & 0.017 \scriptsize(0.02) & {\color{red}0.083 \scriptsize[8]} & 0.010 \scriptsize(0.01) & 0.041 \scriptsize(0.02) & 0.008 \scriptsize(0.01) & 0.034 \scriptsize(0.02) & 0.003 \scriptsize(0.01) & 0.022 \scriptsize(0.02) & 0.003 \scriptsize(0) & 0.027 \scriptsize(0.02)\\\hline
				& 40 & {\color{red}0.212 \scriptsize[5]} & {\color{red}1.954 \scriptsize[1]} & 0.086 \scriptsize(0.07) & 0.162 \scriptsize(0.07) & 0.060 \scriptsize(0.03) & 0.128 \scriptsize(0.07) & 0.046 \scriptsize(0.03) & 0.037 \scriptsize(0.02) & - & -\\
				3 & 60 & {\color{red}0.071 \scriptsize[8]} & {\color{red}0.303 \scriptsize[4]} & 0.037 \scriptsize(0.02) & 0.116 \scriptsize(0.05) & 0.023 \scriptsize(0.02) & 0.083 \scriptsize(0.04) & 0.019 \scriptsize(0.01) & 0.021 \scriptsize(0.02) & 0.039 \scriptsize(0.01) & {\color{red}0.083 \scriptsize[8]}\\
				& 80 & 0.053 \scriptsize(0.04) & {\color{red}0.129 \scriptsize[5]} & 0.022 \scriptsize(0.02) & 0.078 \scriptsize(0.05) & 0.008 \scriptsize(0.01) & 0.044 \scriptsize(0.03) & 0.008 \scriptsize(0.01) & 0.010 \scriptsize(0.01) & 0.017 \scriptsize(0.01) & {\color{red}0.050 \scriptsize[9]}\\
				& 100 & 0.044 \scriptsize(0.03) & {\color{red}0.102 \scriptsize[6]} & 0.023 \scriptsize(0.02) & 0.065 \scriptsize(0.03) & 0.004 \scriptsize(0.01) & 0.042 \scriptsize(0.03) & 0.004 \scriptsize(0.01) & 0.008 \scriptsize(0.01) & 0.008 \scriptsize(0.01) & 0.053 \scriptsize(0.03)\\\hline
				& 40 & {\color{red}0.047 \scriptsize[9]} & {\color{red}0.039 \scriptsize[6]} & 0.016 \scriptsize(0.02) & 0.023 \scriptsize(0.03) & 0.016 \scriptsize(0.02) & 0.010 \scriptsize(0.02) & 0.012 \scriptsize(0.01) & 0.004 \scriptsize(0.01) & - & -\\
				4 & 60 & 0.028 \scriptsize(0.04) & {\color{red}0.035 \scriptsize[8]} & 0.005 \scriptsize(0.01) & 0.015 \scriptsize(0.02) & 0.005 \scriptsize(0.01) & 0.007 \scriptsize(0.02) & 0.005 \scriptsize(0.01) & 0 \scriptsize(0) & 0 & 0\\
				& 80 & 0.008 \scriptsize(0.01) & 0.019 \scriptsize(0.02) & 0.001 \scriptsize(0) & 0.005 \scriptsize(0.01) & 0.001 \scriptsize(0) & 0.004 \scriptsize(0.01) & 0.001 \scriptsize(0) & 0 \scriptsize(0) & 0 \scriptsize(0) & 0.010 \scriptsize(0.01)\\	
				& 100 & 0 \scriptsize(0) & 0.012 \scriptsize(0.01) & 0 \scriptsize(0) & 0.002 \scriptsize(0) & 0 \scriptsize(0) & 0.001 \scriptsize(0) & 0 \scriptsize(0) & 0 \scriptsize(0) & 0 \scriptsize(0) & 0.003 \scriptsize(0.01)\\\hline
			\end{tabu}
		}
		\caption{$\varepsilon$-Indicator averaged over 10 runs for different central parts of the Pareto front, budgets and number of objectives. Lower values are better.
		}
		\label{tab:epsilon}
	\end{table}
	
	\begin{figure}[!ht]
		\centering
		\makebox[\textwidth][c]{\includegraphics[width=1.3\textwidth]{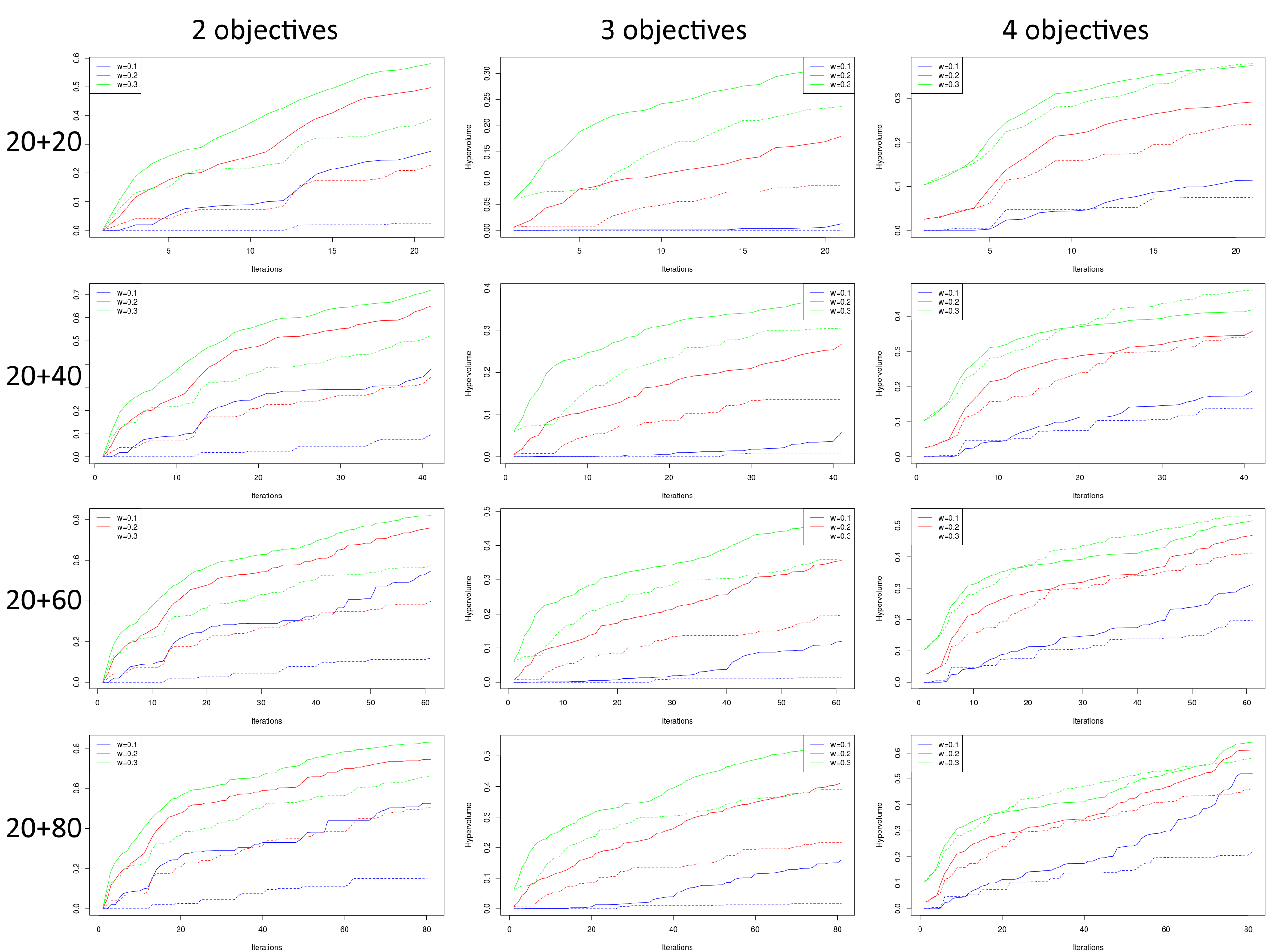}}
		\caption{Mean hypervolume indicator for 2, 3 or 4 objectives (as columns) and total budgets of 40, 60, 80, 100 (as rows). The blue, red and green colors correspond to the improvement regions $\mathcal{I}_{0.1}$, $\mathcal{I}_{0.2}$ and $\mathcal{I}_{0.3}$, respectively. Dashed lines correspond to the standard EHI, continuous lines to the C-EHI algorithm.
}
		\label{fig:evolution_hypervolume}
	\end{figure}
	
	The average performance measures reported in Tables \ref{tab:hypervolume} to \ref{tab:epsilon} confirm the behavior of the C-EHI algorithm already illustrated in Figure \ref{fig:comparaison_approches} for a typical run: mEI set to improve on the estimated center efficiently drives the algorithm towards the (unknown) central part of the real Pareto front. 
	Table \ref{tab:hypervolume} summarizes test results expressed in terms of hypervolume improvements.
	In the most central part of the front ($w=0.1$) C-EHI significantly surpasses the standard EHI. It is also remarkable that despite early GPs inaccuracies, the algorithm does not drift towards off-centered locations of the front.
	EHI outperforms C-EHI only with 4 objectives and $w=0.3$, since in this case $\mathcal I_w$ is not a restrictive central part in such dimension.
		
	The IGD (Table \ref{tab:igd}) shows similar results. Notice that for at least one run, the classical EHI does not reach the $\mathcal{I}_{0.1}$ area in the two and three objective cases, even if 100 evaluations are allowed. 
	In the 4 dimensional case, at least 80 iterations are needed. Again, the results show smaller distances between points in $\PY\cap\mathcal I_{w}$ and $\widehat\PY$ with C-EHI for 2 objectives, and when the restriction area is small. 
	For 4 objectives and $w=0.3$, EHI outperforms C-EHI, but in this case $\mathcal I_{0.3}$ is a quite large part of $Y$. Many solutions in $\PY\cap\mathcal I_{0.3}$ are thus far away from the area where C-EHI converges.
	
	Test results expressed in terms of the $\varepsilon$-Indicator, which is a distance to the Pareto front, are provided in Table \ref{tab:epsilon}. 
	In narrow central areas, C-EHI performs very well, meaning that the best point of $\widehat\PY\cap\mathcal{I}_{w}$ is close to $\PY\cap\mathcal I_{w}$.
	When considering the whole Pareto front, closeness to optimal solutions is improved using C-EHI with 2 and 3 objectives. 
	The $\varepsilon$-Indicator with the whole front is similar to that with the restrictions to $\mathcal I_{w}$, meaning that the closest points to $\PY$ have occured in the central part. It is not necessarily the case for EHI.
	Many 0's occur in the last row of Table \ref{tab:epsilon} where $m=4$. The reason is that, with many objectives, the true Pareto front does not only contain points coming from NSGA-II but also from EHI or C-EHI optimizations.	
	
	Other indicators such as attainment times (average/median/worse number of iterations over the 10 experiments to reach some central objective values) confirm the results reported above, but are not given here for reasons of brevity.
	
	\section{Conclusions}
	In this work, we have developed new concepts and have adapted existing Bayesian multi-objective optimization methods to enhance convergence to equilibrated	solutions of a multi-objective optimization problem at severely restricted number of calls to the objective functions. A general definition of the Pareto front center, valid for non-convex, discontinuous, convoluted fronts has been given and some of its properties analyzed. We have proposed the C-EHI optimization algorithm which first estimates the Pareto center, then maximizes the mEI criterion and finally chooses a targeted central part of the Pareto front in accordance with the remaining budget. 
The C-EHI algorithm has shown better convergence to the center of the Pareto front than other state-of-the-art approaches.
A possible continuation to this work is to study the effect of further increasing the number of objectives as the topology of Pareto fronts in high dimensional spaces remains largely unknown and point targeting becomes more necessary. 
	
	\section*{Acknowledgements}
	This research was performed within the framework of a CIFRE grant (convention \#2016/0690) established between the ANRT and the Groupe PSA for the doctoral work of David Gaudrie.
	
	\noindent The authors would like to thank Philippe Solal for discussions about the center of the Pareto front and Eric Touboul for his help with the geometric proofs of the center invariance to linear scalings.
	
	\appendix
	\section{Appendix: Nadir point estimation using Gaussian Processes}
	\label{sec:appendix1}
	In the field of EMOA's, estimation procedures for extreme points, thus components of $\N$, have been proposed \cite{DebNadir,BechikhNadir}. 
	In the Gaussian Processes framework, we look for $\x$'s that are likely to be extreme design points (Definition~\ref{defextremepoints}).
	Estimating the Nadir point through surrogates is a difficult task. 
	When $m>2$, the Nadir components come from extreme points that are not necessarily optimal in a single objective (cf. Definition~\ref{defnadir}).
	A straightforward estimation of the Nadir involves the knowledge of the whole Pareto front, as each component $j$ of the Nadir point is dependent on the $j$-th objective function, but also on all other functions through the component-wise non-domination property of $\N$.
	However, the C-EHI algorithm only targets central solutions.
	With this algorithm, the GPs may not be accurate at non central locations of $\PY$.
	Using simulated values of the GPs instead of the kriging prediction should nonetheless reduce the impact of a potential inaccuracy as the latter is implicitly considered.
	Applying a step of mono-objective $f_j(\cdot)$ minimization (e.g. using EGO) might diminish this difficulty (at least for the $\I$ estimation), at the expense of $m$ costly evaluations of the computer code.
	
	\vskip\baselineskip
	We now explain the proposed estimation approach. Extreme points have the property of being both \emph{large} in the $j$-th objective and not dominated (ND).
	We are thus interested in $\x's$ with a high probability $\Pr(Y_j(\x)>a_j ~,~ \Y(\x) \text{ ND})$, for $j=1,\dotsc,m$. A typical choice for $a_j$ is the $j$-th component of the Nadir of the current Pareto front approximation, $\widehat{\nu}^{j}_j$. Non-Domination refers to the current Pareto front approximation $\widehat\PY$. 
	These events are not independent since $\Y(\x)$ contains $Y_j(\x)$. However, by conditioning on $Y_j(\x)>\widehat{\nu}^{j}_j$, $\Pr(Y_j(\x)>\widehat{\nu}^{j}_j,\Y(\x) \text{ ND})={\Pr(\Y(\x) \text{ ND}\vert Y_j(\x)>\widehat{\nu}^{j}_j)}\times{\Pr(Y_j(\x)>\widehat{\nu}^{j}_j)}$. The first part can be further simplified: to be non-dominated by $\widehat\PY$, a vector $\mathbf z\in\R^m$ with $z_j>\underset{\y\in\widehat{\PY}}{\max}~y_j$ has to be non-dominated by $\widehat\PY$ with regard to objectives ${1,\dotsc,j-1,j+1,\dotsc,m}$. Hence, $\Pr(\Y(\x) \text{ ND}\vert Y_j(\x)>\widehat{\nu}^{j}_j)=\Pr(\Y(\x) \text{ ND}_{\backslash\{j\}})$ where $\text{ ND}_{\backslash\{j\}}$ stands for non-domination omitting the objective $j$. Finally, the most promising candidates for generating extreme points of the Pareto front are those with large probability $\Pr(\Y(\x) \text{ ND}_{\backslash\{j\}})\times{\Pr(Y_j(\x)>\widehat{\nu}^{j}_j)}$.
	
	Besides these candidates, a second scenario will lead to new extreme points. If $\mathbf z\in\R^m\preceq\widehat{\pmb\nu}^{j}$ is obtained through simulations, $\widehat{\pmb\nu}^{j}$ will no longer belong to the simulated Pareto front. Consequently, the $j$-th component of the Nadir point of the simulated front will also be modified in that case. When $m=2$, the new $\widehat{\nu}^{j}_j$ will be $z_j$, but this does not necessarily hold in higher dimensions.
	
	In short, two events will lead to new extreme points: dominating the $j$-th current extreme point, $\{\Y(\x)\preceq\widehat{\pmb\nu}^{j}\}$, or being both larger than it in $j$-th objective and ND with respect to the approximation front in the remaining objectives, $\{Y_j(\x)>\widehat{\nu}^{j}_j,\Y(\x) \text{ ND}_{\backslash\{j\}}\}$. The areas corresponding to these events are sketched with a 2D example in Figure~\ref{fig:recherche_nadir}.
	Being disjoint, the probability of the union of these events equals the sum. In the end, for estimating the $j$ extreme points and by extension $\N$, the most promising candidates are those maximizing
	\begin{equation}\Pr(\Y(\x) \text{ ND}_{\backslash\{j\}})\times{\Pr(Y_j(\x)>\widehat{\nu}^{j}_j)}+\Pr(\Y(\x)\preceq\widehat{\pmb\nu}^{j})\text{,}\label{eq:candidats_nadir}\end{equation}
	for $j=1,\dotsc,m$. $\Pr(\Y(\x) \text{ ND}_{\backslash\{j\}})$ is the probability of being non-dominated with respect to a $m-1$ dimensional front (which is smaller than the restriction of $\widehat\PY$ to $\{1,\dotsc,m\}\backslash\{j\}$) and is the more computationally demanding term for a given $\x$. 
	The other terms are univariate and product of univariate Gaussian CDF's, respectively.
	
	\begin{figure}[!ht]
		\centering
		\includegraphics[width=0.6\textwidth]{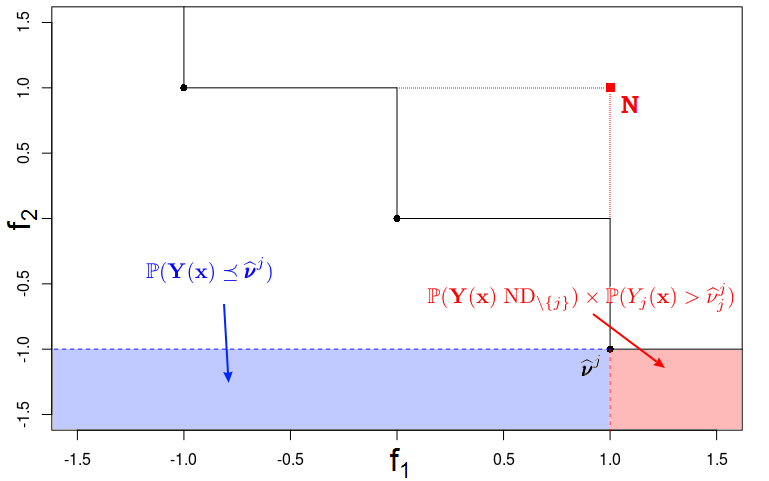}
		\caption{Areas leading to a new first component of the Nadir ($j=1$). A point in the red zone (larger than the first extreme point in the first objective) or in the blue zone (dominating the current extreme point) becomes the new (first) extreme point, and therefore induces a modification of $\N$. }
		\label{fig:recherche_nadir}
	\end{figure}
	
	In the particular case of two objectives, the union of these events reduces to dominating $\widehat{\pmb\nu^{j}}$ in all objectives but $j$, that is to say, in the other objective $\bar j$.
	This is equivalent to looking for candidates with lower $f_{\bar j}(\cdot)$, which has already been investigated when looking for candidates for estimating $I_{\bar j}$. Unfortunately, in a general $m$-dimensional case no simplification occurs.
	The set of candidates that are likely to dominate $\widehat{\pmb\nu}^{j}$ in all objectives but $j$ is included but not equal to the set of candidates likely to maximize (\ref{eq:candidats_nadir}), whose probabilities are respectively $\Pr(\Y(\x)\preceq_{\backslash\{j\}}\widehat{\pmb\nu^{j}})$ and $\Pr(\Y(\x) \text{ ND}_{\backslash\{j\}})\times{\Pr(Y_j(\x)>\widehat{\nu}^{j}_j)}+\Pr(\Y(\x)\preceq\widehat{\pmb\nu}^{j})$, as the latter encompasses more cases for producing new extreme points when $m>2$.
	It is indeed possible to construct $\mathbf z\in\R^m$ such that $z_j>\nu^{j}_j$, $\mathbf z$ ND$_{\backslash\{j\}}$ and $\mathbf z\npreceq_{\backslash\{j\}}\widehat{\pmb\nu^{j}}$.
	Such a $\mathbf z$ will become the $j$-th extreme point without dominating the previous $j$-th extreme point in objectives $\{1,\dotsc,m\}\backslash\{j\}$.
	
	\label{annexes:estimation_nadir}
	
	\section{Appendix: Targeting non central parts of the Pareto front}
\label{sec:noncentraltarget}
In the main body of this paper, we have assumed that the end-user has not expressed any preference and have therefore targeted the empirical center of the Pareto front as a default setting.
In Section \ref{sec:center}, this center $\widehat{\mathbf C}$ was built as the point of the estimated Ideal-Nadir line, $\widehat{\mathcal L}$, the closest to the empirical front $\widehat{\PY}$.
	
Practitioners may nonetheless have preferences regarding the objective space. 
When expressed through a reference point $\RR$ given as an aspiration level, these preferences can be incorporated in our algorithm very simply, by using mEI together with an adequate $\widehat\RR$. 
The adapted reference point $\widehat\RR$ is the point of the segments $\widehat{\mathcal L'}$ joining the estimated Ideal, the reference point $\RR$ and the Nadir, which is the closest to the front $\widehat{\PY}$.
This mechanism accommodates both situations when $\RR$ can and cannot be reached (i.e., $\RR$ is on both sides of the true front $\PY$) and it is illustrated in Figure \ref{fig:update}.
	
	\begin{figure}[h!]
		\centering
		\includegraphics[width=0.8\textwidth]{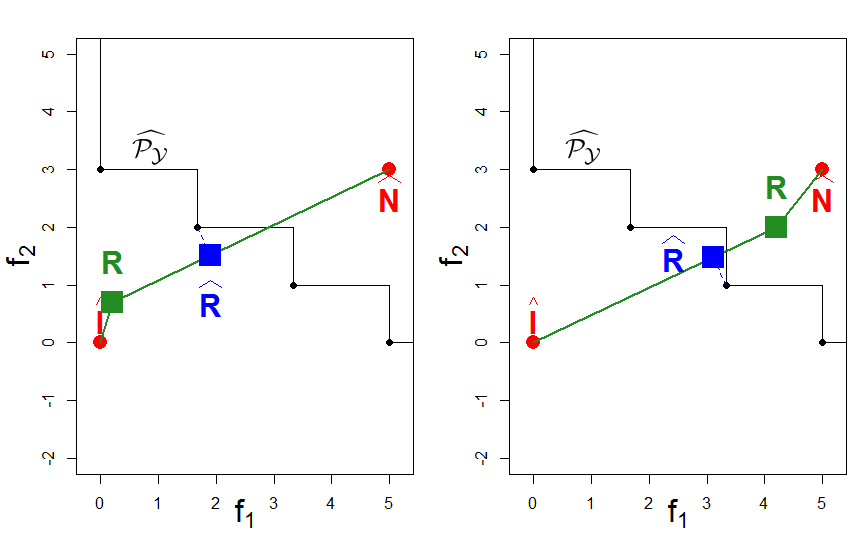}
		\caption{To stay non-dominated and to adapt to $\widehat\PY$, the user-supplied $\RR$ is updated to an $\widehat\RR$. Left: case where $\RR$ is clearly too optimistic, and $\widehat\RR$ is better suited to the current Pareto front $\widehat\PY$. Right: the user-provided target has been attained and a more ambitious $\widehat\RR$ is used instead.}
		\label{fig:update}
	\end{figure}
	
The Algorithm \ref{algo} is readily transformed into a method that aims at $\RR$ just by changing the update of $\widehat{\mathbf{C}}$ into that of $\widehat\RR$. 
Other parts of Algorithm \ref{algo} remain unchanged.

Figure \ref{fig:rehi} shows one optimization run in which a non-central target $\RR$ has been provided.
		\begin{figure}[h!]
			\centering
			\includegraphics[width=0.8\textwidth]{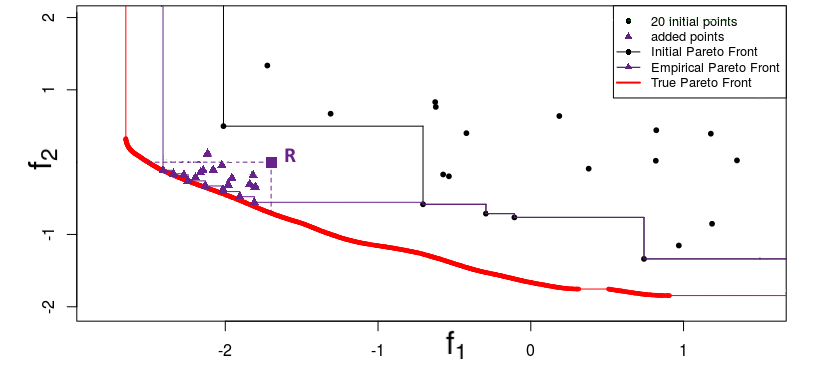}
			\caption{Example of an optimization run where an off-centered target $\RR$ is provided. The Pareto front is found within the user-defined improvement region.}
			\label{fig:rehi}
		\end{figure}

	\vskip\baselineskip
	\bibliographystyle{plain}
	\bibliography{biblio}	

\begin{thebibliography}{10}

\bibitem{WHI2}
Anne Auger, Johannes Bader, Dimo Brockhoff, and Eckart Zitzler.
\newblock Articulating user preferences in many-objective problems by sampling
  the weighted hypervolume.
\newblock In {\em Proceedings of the 11th Annual conference on Genetic and
  evolutionary computation}, pages 555--562. ACM, 2009.

\bibitem{WHI3}
Anne Auger, Johannes Bader, Dimo Brockhoff, and Eckart Zitzler.
\newblock Investigating and exploiting the bias of the weighted hypervolume to
  articulate user preferences.
\newblock In {\em Proceedings of the 11th Annual conference on Genetic and
  evolutionary computation}, pages 563--570. ACM, 2009.

\bibitem{auger2009theory}
Anne Auger, Johannes Bader, Dimo Brockhoff, and Eckart Zitzler.
\newblock Theory of the hypervolume indicator: optimal $\mu$-distributions and
  the choice of the reference point.
\newblock In {\em Proceedings of the tenth ACM SIGEVO workshop on Foundations
  of genetic algorithms}, pages 87--102. ACM, 2009.

\bibitem{auger2012hypervolume}
Anne Auger, Johannes Bader, Dimo Brockhoff, and Eckart Zitzler.
\newblock Hypervolume-based multiobjective optimization: Theoretical
  foundations and practical implications.
\newblock {\em Theoretical Computer Science}, 425:75--103, 2012.

\bibitem{auger2005performance}
Anne Auger and Nikolaus Hansen.
\newblock Performance evaluation of an advanced local search evolutionary
  algorithm.
\newblock In {\em 2005 IEEE congress on evolutionary computation}, volume~2,
  pages 1777--1784. IEEE, 2005.

\bibitem{Pref1}
Slim Bechikh, Marouane Kessentini, Lamjed~Ben Said, and Khaled Gh{\'e}dira.
\newblock Preference incorporation in evolutionary multiobjective optimization:
  a survey of the state-of-the-art.
\newblock In {\em Advances in Computers}, volume~98, pages 141--207. Elsevier,
  2015.

\bibitem{BechikhNadir}
Slim Bechikh, Lamjed~Ben Said, and Khaled Ghedira.
\newblock Estimating {Nadir} point in multi-objective optimization using mobile
  reference points.
\newblock In {\em Evolutionary computation (CEC), 2010 IEEE congress on}, pages
  1--9. IEEE, 2010.

\bibitem{bect2017bayesian}
Julien Bect, Ling Li, and Emmanuel Vazquez.
\newblock Bayesian subset simulation.
\newblock {\em SIAM/ASA Journal on Uncertainty Quantification}, 5(1):762--786,
  2017.

\bibitem{beume2009complexity}
Nicola Beume, Carlos~M Fonseca, Manuel Lopez-Ibanez, Luis Paquete, and Jan
  Vahrenhold.
\newblock On the complexity of computing the hypervolume indicator.
\newblock {\em IEEE Transactions on Evolutionary Computation},
  13(5):1075--1082, 2009.

\bibitem{GPareto}
Mickael Binois and Victor Picheny.
\newblock {GP}areto: An {R} package for {Gaussian}-process based
  multi-objective optimization and analysis.

\bibitem{binois2019kalai}
Micka{\"e}l Binois, Victor Picheny, Patrick Taillandier, and Abderrahmane
  Habbal.
\newblock The {K}alai-{S}morodinski solution for many-objective bayesian
  optimization.
\newblock {\em arXiv preprint arXiv:1902.06565}, 2019.

\bibitem{TheseBinois}
Mickaël Binois.
\newblock {\em Uncertainty quantification on {Pareto} fronts and
  high-dimensional strategies in {Bayesian} optimization, with applications in
  multi-objective automotive design}.
\newblock PhD thesis, École Nationale Supérieure des Mines de Saint-Etienne,
  2015.

\bibitem{bowman1976relationship}
V.~Joseph Bowman.
\newblock On the relationship of the {T}chebycheff norm and the efficient
  frontier of multiple-criteria objectives.
\newblock In {\em Multiple criteria decision making}, pages 76--86. Springer,
  1976.

\bibitem{branke2004finding}
J{\"u}rgen Branke, Kalyanmoy Deb, Henning Dierolf, and Matthias Osswald.
\newblock Finding knees in multi-objective optimization.
\newblock In {\em International conference on parallel problem solving from
  nature}, pages 722--731. Springer, 2004.

\bibitem{branke2008multiobjective}
J{\"u}rgen Branke, Kalyanmoy Deb, Kaisa Miettinen, and Roman Slowi{\'n}ski.
\newblock {\em Multiobjective optimization: Interactive and evolutionary
  approaches}, volume 5252.
\newblock Springer Science \& Business Media, 2008.

\bibitem{WHI4}
Dimo Brockhoff, Johannes Bader, Lothar Thiele, and Eckart Zitzler.
\newblock Directed multiobjective optimization based on the weighted
  hypervolume indicator.
\newblock {\em Journal of Multi-Criteria Decision Analysis}, 20(5-6):291--317,
  2013.

\bibitem{coco-biobj-TR2016}
Dimo Brockhoff, Tea Tusar, Dejan Tusar, Tobias Wagner, Nikolaus Hansen, and
  Anne Auger.
\newblock Biobjective performance assessment with the {COCO} platform.
\newblock {\em CoRR}, abs/1605.01746, 2016.

\bibitem{buchanan03}
John Buchanan and Lorraine Gardiner.
\newblock A comparison of two reference point methods in multiple objective
  mathematical programming.
\newblock {\em European Journal of Operational Research}, 149(1):17--34, 2003.

\bibitem{chan2013klee}
Timothy~M Chan.
\newblock Klee's measure problem made easy.
\newblock In {\em 2013 IEEE 54th Annual Symposium on Foundations of Computer
  Science}, pages 410--419. IEEE, 2013.

\bibitem{chevalier2014corrected}
Clément Chevalier, David Ginsbourger, and Xavier Emery.
\newblock Corrected kriging update formulae for batch-sequential data
  assimilation.
\newblock In {\em Mathematics of Planet Earth}, pages 119--122. Springer, 2014.

\bibitem{IGD}
Carlos A~Coello Coello and Nareli~Cruz Cort{\'e}s.
\newblock Solving multiobjective optimization problems using an artificial
  immune system.
\newblock {\em Genetic Programming and Evolvable Machines}, 6(2):163--190,
  2005.

\bibitem{couckuyt2014fast}
Ivo Couckuyt, Dirk Deschrijver, and Tom Dhaene.
\newblock Fast calculation of multiobjective probability of improvement and
  expected improvement criteria for {P}areto optimization.
\newblock {\em Journal of Global Optimization}, 60(3):575--594, 2014.

\bibitem{cressie2015statistics}
Noel Cressie.
\newblock {\em Statistics for spatial data}.
\newblock John Wiley \& Sons, 1993.

\bibitem{LivreDeb}
Kalyanmoy Deb.
\newblock {\em Multi-objective optimization using evolutionary algorithms},
  volume~16.
\newblock John Wiley \& Sons, 2001.

\bibitem{DebNadir}
Kalyanmoy Deb, Kaisa Miettinen, and Shamik Chaudhuri.
\newblock Toward an estimation of {Nadir} objective vector using a hybrid of
  evolutionary and local search approaches.
\newblock {\em IEEE Transactions on Evolutionary Computation}, 14(6):821--841,
  2010.

\bibitem{NSGAII}
Kalyanmoy Deb, Amrit Pratap, Sameer Agarwal, and Tamt Meyarivan.
\newblock A fast and elitist multiobjective genetic algorithm: {NSGA-II}.
\newblock {\em IEEE transactions on evolutionary computation}, 6(2):182--197,
  2002.

\bibitem{RNSGA}
Kalyanmoy Deb and J~Sundar.
\newblock Reference point based multi-objective optimization using evolutionary
  algorithms.
\newblock In {\em Proceedings of the 8th annual conference on Genetic and
  evolutionary computation}, pages 635--642. ACM, 2006.

\bibitem{EHI2}
Michael Emmerich, Nicola Beume, and Boris Naujoks.
\newblock An {EMO} algorithm using the hypervolume measure as selection
  criterion.
\newblock In {\em International Conference on Evolutionary Multi-Criterion
  Optimization}, pages 62--76. Springer, 2005.

\bibitem{EHI}
Michael Emmerich, André Deutz, and Jan~Willem Klinkenberg.
\newblock Hypervolume-based expected improvement: Monotonicity properties and
  exact computation.
\newblock In {\em Evolutionary Computation (CEC), 2011 IEEE Congress on}, pages
  2147--2154. IEEE, 2011.

\bibitem{EHI3}
Michael Emmerich, Kyriakos Giannakoglou, and Boris Naujoks.
\newblock Single-and multiobjective evolutionary optimization assisted by
  {Gaussian} random field metamodels.
\newblock {\em IEEE Transactions on Evolutionary Computation}, 10(4):421--439,
  2006.

\bibitem{emmerich2016multicriteria}
Michael Emmerich, Kaifeng Yang, Andr{\'e} Deutz, Hao Wang, and Carlos~M
  Fonseca.
\newblock A multicriteria generalization of {B}ayesian global optimization.
\newblock In {\em Advances in Stochastic and Deterministic Global
  Optimization}, pages 229--242. Springer, 2016.

\bibitem{TheseFeliot}
Paul Feliot.
\newblock {\em Une approche {Bayesienne} pour l'optimisation multi-objectif
  sous contraintes}.
\newblock PhD thesis, Universite Paris-Saclay, 2017.

\bibitem{fonseca1998multiobjective}
Carlos~M Fonseca and Peter~J Fleming.
\newblock Multiobjective optimization and multiple constraint handling with
  evolutionary algorithms. {I}. {A} unified formulation.
\newblock {\em IEEE Transactions on Systems, Man, and Cybernetics-Part A:
  Systems and Humans}, 28(1):26--37, 1998.

\bibitem{book_mcdm}
Tomas Gal, Theodor Stewart, and Thomas Hanne.
\newblock {\em Multicriteria decision making: advances in {MCDM} models,
  algorithms, theory, and applications}, volume~21.
\newblock Springer Science \& Business Media, 1999.

\bibitem{KrigingBeliever}
David Ginsbourger, Rodolphe Le~Riche, and Laurent Carraro.
\newblock Kriging is well-suited to parallelize optimization.
\newblock In {\em Computational Intelligence in Expensive Optimization
  Problems}, pages 131--162. Springer, 2010.

\bibitem{halton}
John~H Halton.
\newblock On the efficiency of certain quasi-random sequences of points in
  evaluating multi-dimensional integrals.
\newblock {\em Numerische Mathematik}, 2(1):84--90, 1960.

\bibitem{Paint}
Markus Hartikainen, Kaisa Miettinen, and Margaret~M Wiecek.
\newblock {PAINT}: {Pareto} front interpolation for nonlinear multiobjective
  optimization.
\newblock {\em Computational optimization and applications}, 52(3):845--867,
  2012.

\bibitem{REMOA}
Hisao Ishibuchi, Yasuhiro Hitotsuyanagi, Noritaka Tsukamoto, and Yusuke Nojima.
\newblock Many-objective test problems to visually examine the behavior of
  multiobjective evolution in a decision space.
\newblock In {\em International Conference on Parallel Problem Solving from
  Nature}, pages 91--100. Springer, 2010.

\bibitem{ishibuchi2018specify}
Hisao Ishibuchi, Ryo Imada, Yu~Setoguchi, and Yusuke Nojima.
\newblock How to specify a reference point in hypervolume calculation for fair
  performance comparison.
\newblock {\em Evolutionary computation}, 26(3):411--440, 2018.

\bibitem{jaszkiewicz2018improved}
Andrzej Jaszkiewicz.
\newblock Improved quick hypervolume algorithm.
\newblock {\em Computers \& Operations Research}, 90:72--83, 2018.

\bibitem{EIEMO}
Shinkyu Jeong and Shigeru Obayashi.
\newblock Efficient {G}lobal {O}ptimization ({EGO}) for multi-objective problem
  and data mining.
\newblock In {\em Evolutionary Computation, 2005. The 2005 IEEE Congress on},
  volume~3, pages 2138--2145. IEEE, 2005.

\bibitem{jones2001}
Donald~R Jones.
\newblock A taxonomy of global optimization methods based on response surfaces.
\newblock {\em Journal of global optimization}, 21(4):345--383, 2001.

\bibitem{jones2008large}
Donald~R Jones.
\newblock Large-scale multi-disciplinary mass optimization in the auto
  industry.
\newblock In {\em MOPTA 2008 Conference (20 August 2008)}, 2008.

\bibitem{EGO}
Donald~R Jones, Matthias Schonlau, and William~J Welch.
\newblock Efficient {G}lobal {O}ptimization of expensive black-box functions.
\newblock {\em Journal of Global optimization}, 13(4):455--492, 1998.

\bibitem{KS}
Ehud Kalai and Meir Smorodinsky.
\newblock Other solutions to {Nash}'s bargaining problem.
\newblock {\em Econometrica: Journal of the Econometric Society}, pages
  513--518, 1975.

\bibitem{Keane}
Andy~J Keane.
\newblock Statistical improvement criteria for use in multiobjective design
  optimization.
\newblock {\em AIAA journal}, 44(4):879--891, 2006.

\bibitem{Parego}
Joshua Knowles.
\newblock {ParEGO}: A hybrid algorithm with on-line landscape approximation for
  expensive multiobjective optimization problems.
\newblock {\em IEEE Transactions on Evolutionary Computation}, 10(1):50--66,
  2006.

\bibitem{knowles2002metrics}
Joshua Knowles and David Corne.
\newblock On metrics for comparing nondominated sets.
\newblock In {\em Evolutionary Computation, 2002. CEC'02. Proceedings of the
  2002 Congress on}, volume~1, pages 711--716. IEEE, 2002.

\bibitem{lacour2017box}
Renaud Lacour, Kathrin Klamroth, and Carlos~M Fonseca.
\newblock A box decomposition algorithm to compute the hypervolume indicator.
\newblock {\em Computers \& Operations Research}, 79:347--360, 2017.

\bibitem{Pref2}
Longmei Li, Iryna Yevseyeva, Vitor Basto-Fernandes, Heike Trautmann, Ning Jing,
  and Michael Emmerich.
\newblock An ontology of preference-based multiobjective evolutionary
  algorithms.
\newblock {\em arXiv preprint arXiv:1609.08082}, 2016.

\bibitem{li2018modified}
Zheng Li, Xinyu Wang, Shilun Ruan, Zhaojun Li, Changyu Shen, and Yan Zeng.
\newblock A modified hypervolume based expected improvement for multi-objective
  efficient global optimization method.
\newblock {\em Structural and {M}ultidisciplinary {O}ptimization},
  58(5):1961--1979, 2018.

\bibitem{liang2017pareto}
Chen Liang and Sankaran Mahadevan.
\newblock Pareto surface construction for multi-objective optimization under
  uncertainty.
\newblock {\em Structural and {M}ultidisciplinary {O}ptimization},
  55(5):1865--1882, 2017.

\bibitem{WSEI}
Wudong Liu, Qingfu Zhang, Edward Tsang, Cao Liu, and Botond Virginas.
\newblock On the performance of metamodel assisted {MOEA/D}.
\newblock In {\em International Symposium on Intelligence Computation and
  Applications}, pages 547--557. Springer, 2007.

\bibitem{marler2004survey}
R~Timothy Marler and Jasbir~S Arora.
\newblock Survey of multi-objective optimization methods for engineering.
\newblock {\em Structural and {M}ultidisciplinary {O}ptimization},
  26(6):369--395, 2004.

\bibitem{marler2010weighted}
R~Timothy Marler and Jasbir~S Arora.
\newblock The weighted sum method for multi-objective optimization: new
  insights.
\newblock {\em Structural and {M}ultidisciplinary {O}ptimization},
  41(6):853--862, 2010.

\bibitem{miettinen_book}
Kaisa Miettinen.
\newblock {\em Nonlinear multiobjective optimization}, volume~12.
\newblock Springer Science \& Business Media, 1998.

\bibitem{ExpectedImprovement}
Jonas Mockus.
\newblock On {Bayesian} methods for seeking the extremum.
\newblock In {\em Optimization Techniques IFIP Technical Conference}, pages
  400--404. Springer, 1975.

\bibitem{molchanov2005theory}
Ilya Molchanov.
\newblock {\em Theory of Random Sets}.
\newblock Probability and Its Applications. Springer London, 2005.

\bibitem{doe}
Max~D Morris and Toby~J Mitchell.
\newblock Exploratory designs for computational experiments.
\newblock {\em Journal of statistical planning and inference}, 43(3):381--402,
  1995.

\bibitem{TheseParr}
James Parr.
\newblock {\em Improvement criteria for constraint handling and multiobjective
  optimization}.
\newblock PhD thesis, University of Southampton, 2013.

\bibitem{SUR}
Victor Picheny.
\newblock Multiobjective optimization using {Gaussian} process emulators via
  stepwise uncertainty reduction.
\newblock {\em Statistics and Computing}, 25(6):1265--1280, 2015.

\bibitem{SMS}
Wolfgang Ponweiser, Tobias Wagner, Dirk Biermann, and Markus Vincze.
\newblock Multiobjective optimization on a limited budget of evaluations using
  model-assisted {S}-metric selection.
\newblock In {\em International Conference on Parallel Problem Solving from
  Nature}, pages 784--794. Springer, 2008.

\bibitem{GPML}
Carl~Edward Rasmussen and Christopher~KI Williams.
\newblock {\em Gaussian Processes for Machine Learning}.
\newblock The MIT Press, 2006.

\bibitem{roustant2012dicekriging}
Olivier Roustant, David Ginsbourger, and Yves Deville.
\newblock Dice{K}riging, {DiceOptim}: Two {R} packages for the analysis of
  computer experiments by kriging-based metamodeling and optimization.
\newblock 2012.

\bibitem{russo2014quick}
Luis~MS Russo and Alexandre~P Francisco.
\newblock Quick hypervolume.
\newblock {\em IEEE Transactions on Evolutionary Computation}, 18(4):481--502,
  2014.

\bibitem{sacks1989design}
Jerome Sacks, William~J Welch, Toby~J Mitchell, and Henry~P Wynn.
\newblock Design and analysis of computer experiments.
\newblock {\em Statistical science}, pages 409--423, 1989.

\bibitem{santner2013design}
Thomas~J Santner, Brian~J Williams, and William~I Notz.
\newblock {\em The design and analysis of computer experiments}.
\newblock Springer Science \& Business Media, 2013.

\bibitem{sawaragi1985theory}
Yoshikazu Sawaragi, Hirotaka Nakayama, and Tetsuzo Tanino.
\newblock {\em Theory of multiobjective optimization}, volume 176.
\newblock Elsevier, 1985.

\bibitem{sobol}
Il'ya~Meerovich Sobol'.
\newblock On the distribution of points in a cube and the approximate
  evaluation of integrals.
\newblock {\em Zhurnal Vychislitel'noi Matematiki i Matematicheskoi Fiziki},
  7(4):784--802, 1967.

\bibitem{stein1987large}
Michael~L Stein.
\newblock Large sample properties of simulations using latin hypercube
  sampling.
\newblock {\em Technometrics}, 29(2):143--151, 1987.

\bibitem{stein2012interpolation}
Michael~L Stein.
\newblock {\em Interpolation of spatial data: some theory for kriging}.
\newblock Springer Science \& Business Media, 1999.

\bibitem{TheseSvensson}
Joshua Svenson.
\newblock {\em Computer experiments: Multiobjective optimization and
  sensitivity analysis}.
\newblock PhD thesis, The Ohio State University, 2011.

\bibitem{EMI}
Joshua Svenson and Thomas~J Santner.
\newblock Multiobjective optimization of expensive black-box functions via
  expected maximin improvement.
\newblock {\em The Ohio State University, Columbus, Ohio}, 32, 2010.

\bibitem{triantaphyllou2000multi}
Evangelos Triantaphyllou.
\newblock Multi-criteria decision making methods.
\newblock In {\em Multi-criteria decision making methods: A comparative study},
  pages 5--21. Springer, 2000.

\bibitem{ICMOO}
Tobias Wagner, Michael Emmerich, Andr{\'e} Deutz, and Wolfgang Ponweiser.
\newblock On expected-improvement criteria for model-based multi-objective
  optimization.
\newblock In {\em International Conference on Parallel Problem Solving from
  Nature}, pages 718--727. Springer, 2010.

\bibitem{while2012fast}
Lyndon While, Lucas Bradstreet, and Luigi Barone.
\newblock A fast way of calculating exact hypervolumes.
\newblock {\em IEEE Transactions on Evolutionary Computation}, 16(1):86--95,
  2012.

\bibitem{wierzbicki1980use}
Andrzej Wierzbicki.
\newblock The use of reference objectives in multiobjective optimization.
\newblock In {\em Multiple criteria decision making theory and application},
  pages 468--486. Springer, 1980.

\bibitem{book_mcdm_refpoint}
Andrzej Wierzbicki.
\newblock Reference point approaches. published in multicriteria decision
  making: Advances in {MCDM} models, algorithms, theory, and applications. t.
  gal, tj stewart and t. hanne, 1999.

\bibitem{TEHI2}
Kaifeng Yang, Andre Deutz, Zhiwei Yang, Thomas Back, and Michael Emmerich.
\newblock Truncated expected hypervolume improvement: Exact computation and
  application.
\newblock In {\em Evolutionary Computation (CEC), 2016 IEEE Congress on}, pages
  4350--4357. IEEE, 2016.

\bibitem{yang2019multi}
Kaifeng Yang, Michael Emmerich, Andr{\'e} Deutz, and Thomas B{\"a}ck.
\newblock Multi-objective {B}ayesian global optimization using expected
  hypervolume improvement gradient.
\newblock {\em Swarm and evolutionary computation}, 44:945--956, 2019.

\bibitem{yang2017computing}
Kaifeng Yang, Michael Emmerich, Andr{\'e} Deutz, and Carlos~M Fonseca.
\newblock Computing 3-{D} expected hypervolume improvement and related
  integrals in asymptotically optimal time.
\newblock In {\em International Conference on Evolutionary Multi-Criterion
  Optimization}, pages 685--700. Springer, 2017.

\bibitem{yang2015expected}
Kaifeng Yang, Daniel Gaida, Thomas B{\"a}ck, and Michael Emmerich.
\newblock Expected hypervolume improvement algorithm for {PID} controller
  tuning and the multiobjective dynamical control of a biogas plant.
\newblock In {\em Evolutionary Computation (CEC), 2015 IEEE Congress on}, pages
  1934--1942. IEEE, 2015.

\bibitem{TEHI}
Kaifeng Yang, Longmei Li, André Deutz, Thomas Back, and Michael Emmerich.
\newblock Preference-based multiobjective optimization using truncated expected
  hypervolume improvement.
\newblock In {\em Natural Computation, Fuzzy Systems and Knowledge Discovery
  (ICNC-FSKD), 2016 12th International Conference on}, pages 276--281. IEEE,
  2016.

\bibitem{zeleny1976theory}
Milan Zeleny.
\newblock The theory of the displaced ideal.
\newblock In {\em Multiple criteria decision making Kyoto 1975}, pages
  153--206. Springer, 1976.

\bibitem{zhan2017expected}
Dawei Zhan, Yuansheng Cheng, and Jun Liu.
\newblock Expected improvement matrix-based infill criteria for expensive
  multiobjective optimization.
\newblock {\em IEEE Transactions on Evolutionary Computation}, 21(6):956--975,
  2017.

\bibitem{zhang2019multi}
J~Zhang and AA~Taflanidis.
\newblock Multi-objective optimization for design under uncertainty problems
  through surrogate modeling in augmented input space.
\newblock {\em Structural and {M}ultidisciplinary {O}ptimization},
  59(2):351--372, 2019.

\bibitem{zhang2007moea}
Qingfu Zhang and Hui Li.
\newblock {MOEA/D}: A multiobjective evolutionary algorithm based on
  decomposition.
\newblock {\em IEEE Transactions on evolutionary computation}, 11(6):712--731,
  2007.

\bibitem{TAEI}
Qingfu Zhang, Wudong Liu, Edward Tsang, and Botond Virginas.
\newblock Expensive multiobjective optimization by {MOEA/D} with {Gaussian}
  process model.
\newblock {\em IEEE Transactions on Evolutionary Computation}, 14(3):456--474,
  2010.

\bibitem{TheseZitzler}
Eckart Zitzler.
\newblock Evolutionary algorithms for multiobjective optimization: Methods and
  applications.
\newblock 1999.

\bibitem{WHI}
Eckart Zitzler, Dimo Brockhoff, and Lothar Thiele.
\newblock The hypervolume indicator revisited: On the design of
  {Pareto}-compliant indicators via weighted integration.
\newblock In {\em International Conference on Evolutionary Multi-Criterion
  Optimization}, pages 862--876. Springer, 2007.

\bibitem{zdt2000a}
Eckart Zitzler, Kalyanmoy Deb, and Lothar Thiele.
\newblock {Comparison of Multiobjective Evolutionary Algorithms: Empirical
  Results}.
\newblock {\em Evolutionary Computation}, 8(2):173--195, 2000.

\bibitem{zitzler2003performance}
Eckart Zitzler, Lothar Thiele, Marco Laumanns, Carlos~M Fonseca, and
  Viviane~Grunert Da~Fonseca.
\newblock Performance assessment of multiobjective optimizers: An analysis and
  review.
\newblock {\em IEEE Transactions on evolutionary computation}, 7(2):117--132,
  2003.

\end{thebibliography}
	
\end{document}